%% file: main.tex
\documentclass[11pt]{article}
\usepackage{amsthm}
\usepackage{natbib}
\usepackage{graphicx} 
\usepackage{array} 
\usepackage{mathtools}
\usepackage{amsmath, amssymb, amsfonts, verbatim,bbm}
\usepackage{hyphenat,epsfig,subcaption,multirow}
\usepackage{nicefrac}
\usepackage{paralist}
\usepackage[utf8]{inputenc}
\usepackage[font=small, labelfont=bf]{caption}
\usepackage{wrapfig}
\usepackage[dvipsnames]{xcolor}
\usepackage{dsfont}

\DeclareFontFamily{U}{mathx}{\hyphenchar\font45}
\DeclareFontShape{U}{mathx}{m}{n}{
      <5> <6> <7> <8> <9> <10>
      <10.95> <12> <14.4> <17.28> <20.74> <24.88>
      mathx10
      }{}
\DeclareSymbolFont{mathx}{U}{mathx}{m}{n}
\DeclareMathSymbol{\bigtimes}{1}{mathx}{"91}
\usepackage{tcolorbox}
\tcbuselibrary{skins,breakable}
\tcbset{enhanced jigsaw}

\usepackage[normalem]{ulem}
\usepackage[compact]{titlesec}

\definecolor{DarkRed}{rgb}{0.1,0.1,0.8}
\definecolor{DarkBlue}{rgb}{0.1,0.1,0.5}

\usepackage{nameref}
\definecolor{ForestGreen}{rgb}{0.1333,0.5451,0.1333}
\definecolor{DarkRed}{rgb}{0.8,0,0.4}
\definecolor{Red}{rgb}{0.8,0,0.4}
\usepackage[linktocpage=true,
    pagebackref=true,colorlinks,
    urlcolor=DarkRed,
    linkcolor=DarkRed, citecolor=ForestGreen,
    bookmarks,bookmarksopen,bookmarksnumbered]{hyperref}
\usepackage[noabbrev,nameinlink]{cleveref}
\creflabelformat{equation}{#2\textup{#1}#3}
\crefname{property}{property}{Property}
\creflabelformat{property}{(#1)#2#3}
\crefname{equation}{eq}{Eq}
\creflabelformat{equation}{#1#2#3}

\usepackage{bm}
\usepackage{url}
\usepackage{xspace}
\usepackage[mathscr]{euscript}

\usepackage{tikz}
\usetikzlibrary{arrows}
\usetikzlibrary{arrows.meta}
\usetikzlibrary{shapes}
\usetikzlibrary{backgrounds}
\usetikzlibrary{positioning}
\usetikzlibrary{decorations.markings}
\usetikzlibrary{patterns}
\usetikzlibrary{calc}
\usetikzlibrary{fit}
\usepackage{enumitem}
\usetikzlibrary{decorations.pathmorphing}

\usepackage{mdframed}

\usepackage[noend]{algpseudocode}
\makeatletter
\def\BState{\State\hskip-\ALG@thistlm}
\makeatother

\usepackage[margin=1in]{geometry}

\usepackage{soul}

\usepackage{array}
\usepackage{booktabs}
\usepackage{makecell}
\usepackage{tabularx}
\usepackage{multirow}
\usepackage{ragged2e}
\usepackage{colortbl}
\newcolumntype{P}[1]{>{\RaggedRight\hspace{0pt}}p{#1}}
\newcolumntype{L}{>{\arraybackslash}m{3cm}}
\newcolumntype{q}{>{\arraybackslash}m{4.5cm}}

\definecolor{aliceblue}{rgb}{0.94, 0.97, 1.0}
\definecolor{blizzardblue}{rgb}{0.67, 0.9, 0.93}

\newtheorem{lemma}{Lemma}[section]

\newtheorem{theorem}[lemma]{Theorem}

\newtheorem{definition}[lemma]{Definition}

\newtheorem*{claim*}{Claim}
\newtheorem*{proposition*}{Proposition}
\newtheorem*{lemma*}{Lemma}
\newtheorem*{problem*}{Problem}

\crefname{lemma}{Lemma}{Lemmas}
\crefname{note}{Note}{Notes}
\crefname{claim}{Claim}{Claims}

\newtheorem{mdresult}{Result}

\newtheorem{remark}[lemma]{Remark}

\newtheorem{example}{Example}

\newtheoremstyle{restate}{}{}{\itshape}{}{\bfseries}{~(restated).}{.5em}{\thmnote{#3}}
\theoremstyle{restate}

\theoremstyle{definition}
\newtheorem{mdalg}{Algorithm}

\allowdisplaybreaks

\renewcommand{\qed}{\nobreak \ifvmode \relax \else
      \ifdim\lastskip<1.5em \hskip-\lastskip
      \hskip1.5em plus0em minus0.5em \fi \nobreak
      \vrule height0.75em width0.5em depth0.25em\fi}

\input{draft/macros}

\title{\textbf{PAC Learning with Improvements}  
}
{\author{
{\bf Idan Attias$^{1,2}$ \qquad
Avrim Blum$^2$ \qquad
Keziah Naggita$^2$}\\[1ex]
 {\bf Donya Saless$^2$ \qquad
Dravyansh Sharma$^{2,3}$\qquad
 Matthew Walter$^2$}\footnote{Alphabetical order}
 \\[3ex]
$^1$University of Illinois at Chicago \\[1ex]
$^2$Toyota Technological Institute at Chicago\\[1ex]
$^3$Northwestern University\\[1ex]
{\text{\{idan, avrim, knaggita, donya, dravy, mwalter\}@ttic.edu}}
}}

\date{}

\begin{document}

\maketitle

\begin{abstract}
One of the most basic lower bounds in machine learning is that in nearly any nontrivial setting, it takes {\em at least} $1/\epsilon$ samples to learn to error $\epsilon$ (and more, if the classifier being learned is complex).  However, suppose that data points are agents who have the ability to improve by a small amount if doing so will allow them to receive a (desired) positive classification.  In that case, we may actually be able to achieve {\em zero} error by just being ``close enough''.  
For example, imagine a hiring test used to measure an agent's skill at some job such that for some threshold $\theta$, agents who score above $\theta$ will be successful and those who score below $\theta$ will not (i.e., learning a threshold on the line). 
Suppose also that by putting in effort, agents can improve their skill level by some small amount $r$.  In that case, if we learn an approximation $\hat{\theta}$ of $\theta$ such that $\theta \leq \hat{\theta} \leq \theta + r$ and use it for hiring, we can actually achieve error zero, in the sense that (a) any agent classified as positive is truly qualified, and (b) any agent who truly is qualified can be classified as positive by putting in effort.  Thus, the ability for agents to improve has the potential to allow for a goal one could not hope to achieve in standard models, namely zero error.

In this paper, we explore this phenomenon more broadly, giving general results and examining under what conditions the ability of agents to improve can allow for a reduction in the sample complexity of learning, or alternatively, can make learning harder.  We also examine both theoretically and empirically what kinds of improvement-aware algorithms can take into account agents who have the ability to improve to a limited extent when it is in their interest to do so.

\end{abstract}

\input{arxiv_content_theory_no_comment}

\input{arxiv_experiments_no_comment}
\section{Discussion}

We propose a novel model for learning with strategic agents where the agents are allowed to improve. Surprisingly, we are able to achieve zero error (with high probability) by designing appropriate risk-averse learners for several well-studied concept classes, including a fairly general discrete graph-based model. We show that the VC dimension of the concept class is not the correct combinatorial dimension to capture learnability in the context of improvements. We further show that the intersection-closed property is sufficient, and in a certain sense necessary for proper learning with respect to arbitrary improvement sets. We leave open the questions of fully characterizing proper PAC learning with improvements and characterizing improper PAC learnability with improvements in terms of the concept class and the improvement sets available to the agents.

\section*{Impact Statement}
Our work centers on individuals' capability to strategically improve in response to decision-making algorithms, and the effect this has on the kinds of guarantees that one can hope to prove for machine learning algorithms. Our findings can potentially help inform decision-makers to design better algorithms.  However, our main contributions are primarily theoretical in nature, and we anticipate no direct social impact of our work.  We nonetheless believe there is a valuable opportunity to leverage our findings in ways that support both decision-makers and the individuals subject to these algorithms. To realize this potential, future research must critically assess the ethical and practical consequences of deploying improvement-aware algorithms, particularly in high-stakes domains.

\section*{Acknowledgments}
This work was supported in part by the National Science
Foundation under grants CCF-2212968, ECCS-2216899, ECCS-2216970, and ECCS-2217023, by the Simons Foundation under the Simons Collaboration on the Theory of Algorithmic Fairness, and by the Office of Naval Research MURI Grant N000142412742.

\clearpage
\bibliographystyle{alpha}
\bibliography{main}

\clearpage
\appendix

\input{arxiv_appendix_theory_no_comment}

\input{arxiv_exp_appendix_no_comment}
\end{document}

%% file: draft/macros.tex
\newcommand{\R}{\ensuremath{\mathbb{R}}}

\DeclareMathOperator*{\Prob}{\ensuremath{\mathbb{P}}}
\renewcommand{\Pr}{\Prob}

\newenvironment{tbox}{\begin{tcolorbox}[
		enlarge top by=5pt,
		enlarge bottom by=5pt,
		 breakable,
		 boxsep=0pt,
                  left=4pt,
                  right=4pt,
                  top=10pt,
                  arc=0pt,
                  boxrule=1pt,toprule=1pt,
                  colback=white
                  ]
	}
{\end{tcolorbox}}

\newcommand{\II}{\ensuremath{\mathbb{I}}}

\newcommand{\mireal}[1][]{
  \ifx\relax#1\relax%
    \II(\mione \,; \mitwo)%
  \else%
    \II(\mione \,; \mitwo\mid #1)%
  \fi
}

\newcommand{\norm}[1]{\ensuremath{\left\lVert #1 \right\rVert}}

\newif\ifhidecomments
\hidecommentstrue

\ifhidecomments
    \newcommand{\avrim}[1]{}
    \newcommand{\idan}[1]{}
    \newcommand{\donya}[1]{}
    \newcommand{\dravy}[1]{}
    \newcommand{\kn}[1]{}
    \newcommand{\matt}[1]{}
\else
    \newcommand{\avrim}[1]{\noindent{\textcolor{blue}{\{{\bf Avrim} \em #1\}}}}
    \newcommand{\idan}[1]{\noindent{\textcolor{violet}{\{{\bf Idan} \em #1\}}}}
    \newcommand{\donya}[1]{\noindent{\textcolor{magenta}{\{{\bf Donya:} \em #1\}}}}
    \newcommand{\dravy}[1]{\noindent{\textcolor{olive}{\{{\bf Dravy:} \em #1\}}}}
    \newcommand{\kn}[1]{\noindent{\textcolor{orange}{\{{\bf kn:} \em #1\}}}}
    \newcommand{\matt}[1]{\noindent{\textcolor{teal}{\{{\bf Matt:} \em #1\}}}}
\fi

\newcommand{\lrset}[1]{\mathopen{}\left\{#1\right\}}

\newcommand{\clos}{\mathrm{CLOS}}

\newcommand{\ir}{\mathrm{IR}}
\newcommand{\bs}{\mathrm{BS}}


%% file: arxiv_content_theory_no_comment.tex
\section{Introduction}

There has been growing interest in recent years in machine learning settings where a deployed classifier will influence the behavior of the entities it is aiming to classify.  For example, a classifier that maps loan applicants to credit scores and then uses a particular cutoff $\hat{\theta}$ to determine whether an applicant should receive a loan will induce those below the cutoff value to take actions to improve their score. This setting is called strategic classification \cite{hardt2016strategic} or measure management \cite{bloomfield2016counts} when the actions taken do not truly improve the agent's quality, and performative prediction \cite{perdomo2021performativeprediction} more generally.  
In this work, our focus is on the case that the improvements are real e.g., paying off high-interest credit card debt, taking a money management class, etc.\ for genuinely improving one's loan application.  
{That is, the agent  responds to the classifier in order to potentially improve 
their classification~\cite{jkleinhowdo,causalstratmiller2020}, changing its true features in the process. The classifier must take this ``strategic improvement'' response into account.}

{Unlike previous works on strategic improvement that focus extensively on efficiently incentivizing and maximizing agent improvement (e.g., \cite{causalstratmiller2020,jkleinhowdo,incentiveawarenika,pmlr-v119-shavit20a}, among others), we aim to understand how an agent's capacity for improvement impacts learnability, sample complexity, and algorithm design for accurate classification.}
One high-level take-away from our theoretical analysis and empirical results is that the ability of agents to improve favors algorithms that are more ‘‘conservative'' in their decisions.  This is both due to the reduced concern over false-negative errors (since agents in those regions may still be able to improve to be classified as positive) and the increased concern over false-positive errors (which may cause individuals to ``improve'' incorrectly).

To illustrate the potential reduction in 
sample complexity that result from agents' ability to improve, one of the most basic lower bounds in machine learning is that in nearly any nontrivial setting, it takes {\em at least} $1/\epsilon$ samples to learn to error $\epsilon$ (and more, if the classifier being learned is complex).  However, if agents have the ability to improve by a small amount, we may actually be able to achieve {\em zero} error by just being ``close enough''.   
{To the best of our knowledge, this  has not been previously observed in the strategic improvement literature.} 
Returning to the loan example above, suppose that by putting in effort, agents can improve their credit score by some small amount $r$, and suppose we are in the realizable case that there is some true threshold $\theta$ such that agents with credit score above $\theta$ will be good customers and those who score below $\theta$ will not. In that case, if we learn an approximation $\hat{\theta}$ of $\theta$ such that $\theta \leq \hat{\theta} \leq \theta + r$ and use it as a cutoff to determine who should receive a loan, we can actually achieve zero error in  
that (a) any agent classified as positive is truly qualified, and (b) any agent who truly is qualified can get classified as positive by putting in effort.  Thus, the ability for agents to improve can potentially allow for a goal one could not hope to achieve otherwise.

{We also observe fundamental differences in the inherent learnability of concept classes, compared to both standard PAC learning where the agents cannot respond to the classifier, as well as the strategic PAC setting where the agent tries to deceive the classifier to obtain a more favorable classification. Somewhat surprisingly, learning with improvements can sometimes be easier than the standard PAC setting, and it can sometimes be harder than strategic classification. We show that proper learnability with improvements in the realizable setting is closely linked to the concept class being intersection-closed.}

Concretely, our contributions are as follows:

\begin{itemize}
    \item In Section~\ref{sec:separating-PAC-and-improvements}, we show a separation between the standard PAC model and our model of PAC learning with improvements. Specifically, we show that
    a finite VC dimension is neither necessary nor sufficient for PAC learnability with improvements. We further show a similar separation from the more recently studied  PAC model for strategic classification \citep{hardt2016strategic,strategicPAC}.
    
    \item In Section~\ref{sec:geometric-concepts}, we study learnability of geometric concepts in $\mathbb{R}^d$. We show that any intersection-closed concept class is learnable under our model, and show that the generalization error can be smaller than the standard PAC setting for interesting cases including thresholds and high-dimensional rectangles. We also show that the intersection-closed property is essentially necessary for proper learnability in our setting.

    \item In Section~\ref{sec:graph-model}, we study a graph model in which each node represents an agent and the improvement set of an agent is the set of its neighbors in the graph. We establish near-tight bounds on the number of labeled points the learner needs to see to learn a hypothesis which achieves zero error with high probability, given the ability to learn the labels of uniformly random nodes. We further show that it is possible to learn a ``fairer'' hypothesis that also enables improvement whenever it leads to a better classification for an agent. We also study a \emph{teaching} setting where the teacher aims to find the smallest set of labels needed to ensure that a risk-averse student achieves zero-error, and show that providing the labels for the dominating set of the positive subgraph (induced by the true positive nodes) is sufficient.

    \item In Section~\ref{sec:experiments}, 
    we conduct experiments on three real-world and one fully synthetic binary classification tabular datasets to investigate how the error rate of a model function (\(h\)) decreases when test-set agents that it initially classified as negative improve. Our results indicate that while risk-averse models may start with higher error rates, their errors rapidly drop as the negatively classified test agents improve and the improvement budget (\(r\)) increases. 
    A stricter penalty for false positives typically leads to more accurate positive classifications, resulting in greater gains from agent improvements. In most 
    cases, test errors decline sharply, sometimes reaching zero (e.g., in Figure~\ref{fig:synthetic_0.5}).
    \footnote{Our code is publicly available \href{https://github.com/ripl/PLI/tree/main}{here}.}
    
\end{itemize}

\paragraph{Related Work.}
Learning in the presence of strategic (``gaming''), utility-maximizing agents has gained increasing attention in recent years (\cite{disparateeffects,smithasocialcost,braverman_et_al,strategicperceptron,incentiveawarenika}, among others). Early research framed this problem as a Stackelberg competition \cite{hardt2016strategic,stacklerbergames}, where negatively classified agents manipulate their features to obtain more favorable outcomes if the benefits outweigh the costs. Kleinberg and Raghavan~\cite{jkleinhowdo} extend this model by considering agents who can both manipulate and genuinely improve their features, proposing a mechanism that incentivizes authentic improvement. {This model has been studied under a causal lens, where the learner may not a priori know which features correspond to manipulation or improvement. Strategic learning from observable data requires solving a causal inference problem in this setting~\cite{causalstratmiller2020}, and the ability to test different decision rules can be helpful~\cite{pmlr-v119-shavit20a}.}
Ahmadi et al.~\cite{ahmadi2022classificationstrategicagentsgame} consider a similar setting and propose classification models that balance maximizing true positives with minimizing false positives. 

We extend this line of work by 
analyzing the inherent learnability of classes, the sample complexity of learning, and the ability to achieve zero-error classification, when agents can truly improve. {Inherent learnability of concepts has been studied in the strategic manipulation setting \cite{strategicPAC,cohen2024learnability,lechner2022learning}, but not in the strategic improvement setting. In Section \ref{sec:comparison-strategic}, we show how our improvement setting differs from  strategic manipulation with respect to learnability.
The sample complexity of learning in the presence of purely improving agents has  been studied by Haghtalab et al.~\cite{incentiveawarenika}, but from a social welfare perspective where the goal is to maximize the true positives after improvement. In contrast, our primary focus is classification error, which is more sensitive to false positives. In Section \ref{sec:enabling-improvement}, we show that these two objectives need not be in conflict and may be simultaneously optimized. Finally, the ability of a learner to achieve zero-error for non-trivial concept classes and distributions has not been previously observed in any strategic or non-strategic setting.}

Our work also relates to research in reliable machine learning \cite{rivest1988learning,yaniv10a}, where a learner may abstain from classification to avoid mistakes, balancing coverage (the proportion of classified points) against error. In contrast, we strive for a zero false positive rate and minimal false negative rate, aligning with learning under one-sided error \cite{natarajan1987learning,kalaiReliable}.
We include a more detailed discussion of the related work in Appendix~\ref{app:related-work}.

\section{Formal Setting: PAC Learning with Improvements}\label{sec:setting}
 Let $\mathcal{X}$ denote the instance space consisting of agents with the ability to {\it improve}, as defined below. We restrict our attention to the case of binary classification, i.e., the label space is $\{0,1\}$. Without loss of generality, we refer to label $0$ as the {\it negative} class and label $1$ as the {\it positive} class. Let $\Delta:\mathcal{X}\rightarrow 2^\mathcal{X}$ denote the {\it improvement function} that maps each point (agent) $x\in \mathcal{X}$ to a set of points $\Delta(x)$ (the {\it improvement set}) to which $x$ can potentially {\it improve} itself in order to be classified positively. For example, if $\mathcal{X}$ is a metric space, we can define $\Delta(x)$ as the  $\ell_p$-ball centered at $x$. 
 Let $\mathcal{H}\subseteq \{0,1\}^\mathcal{X}$ denote the concept space, that is, the set of candidate classifiers.
 We will focus on the {\it realizable} setting, i.e.\ we assume the existence of an unknown (to the learner) target concept $f^*:\mathcal{X}\rightarrow\{0,1\}$ that correctly labels all points in $\mathcal{X}$ and satisfies $f^*\in \mathcal{H}$. 

The intuition behind the model is as follows. The learner first publishes a classifier $h:\mathcal{X}\rightarrow\{0,1\}$ 
(potentially based on some data sample labeled according to $f^*$). 
We will refer to this function $h$ as the learner’s hypothesis.
Each agent then reacts to 
$h$~\cite{zrnic2022leadsfollowsstrategicclassification, hardt2016strategic}---if it was classified negatively by $h$, the agent attempts to find a point in its improvement set that is positively classified by $h$ and moves to it. Note that the agents do not know the true function $f^*$ and as a result cannot react with respect to the ground truth, only based on $h$.

We formalize this as the {\it reaction set} with respect to $h$,

\begin{align*}
    \Delta_{h}(x) =\!
    \begin{cases}
        \{x\} \text{\ \ \ if } h(x) = 1,\\
        \{x\} \text{\ \ \ if } \{x'\in \Delta(x) \mid h(x')=1\}= \emptyset, \\
        \{x'\in \Delta(x):  h(x')=1\} \hfill \text{ otherwise.}
    \end{cases}%
    \end{align*}
In other words, if $h$ classifies $x$ as positive, the agent $x$ stays in place and does not attempt to improve. If $h$ classifies $x$  as negative, there are two types of reactions. Either, there is no point in its improvement set that can improve the agent's classification according to $h$ and the agent again stays put. Otherwise, the agent reacts and moves to be predicted positive by $h$. This corresponds to utility-maximizing agents that have a utility of $1$ for being classified as positive, a utility of $0$ for being classified as negative, and that incur a cost for moving, where $\Delta(x)$ corresponds to the points that $x$ can move to at a cost less than $1$. 

We say that a test point $x$ has been misclassified if there exists a point in the reaction set of $x$ where $h$ disagrees with $f^*$, formally, 
\begin{align}\label{def:loss-function}
    \textsc{Loss}(x; h, f^*)=\max_{x'\in \Delta_{h}(x)} \mathbb{I}\left[h\left(x'\right)\neq f^*\left(x'\right)\right]. 
\end{align}

\begin{remark}\label{rmk:misleading-improvement}
  The formulation of the loss function in \eqref{def:loss-function} allows for scenarios where an input \( x \) initially satisfies \( f^*(x) = 1\) and \( h(x) = 0 \), but under \( \Delta_h(x) \), it may transition to a setting where \( f^*(x') = 0 \) and \( h(x') = 1 \) for some \( x' \in \Delta_h(x) \). Think of an example where there are two features such that improving one often comes at the expense of the other. For instance, consider the trade-off between strength and endurance in athletics. Let $f^*(x)$ represent a person's endurance (e.g., marathon running capability), and $h(x)$ represent their strength (e.g., sprinting power). Focusing on increasing $h(x)$ through strength training enhances power, but this often comes at the expense of endurance, thus reducing $f^*(x)$. This reflects the natural conflict between optimizing for one feature while sacrificing the other.
\end{remark}

In words, this corresponds to an assumption that agents will improve to a point in their reaction set while breaking ties adversarially, or equivalently, that they will break ties in favor of points $x'$ for which $f^*(x')=0$. 
This assumption is natural if we want our positive results to be robust to unknown tie-breaking mechanisms, and would also hold if improving to points $x' \in \Delta_h(x)$ whose true label according to $f^*$ is negative is less effort than improving such points whose true label is positive.
Note that this loss function favors classifiers that label uncertain points as negative rather than positive.
For example, if $\{x \mid h(x)=1\} \subseteq \{x \mid f^*(x)=1\}$ then $h$ may still have zero loss if all points $x$ in the difference have at least one point $x' \in \Delta(x)$ for which $h(x')=1$.  The fact that true positives might need to put in effort to improve in order to be classified as positive (or that some negative points are not able to improve themselves to be classified as positive by $h$ even if they would have been able to do so with respect to $f^*$) does not count as an error in our setting. 

See Section~\ref{subsec:thresholds} for a concrete example.

Analogous to standard PAC learning, we assume the learner has access to a finite set of samples $S\in \mathcal{X}^m$ drawn randomly according to some fixed but unknown distribution $\mathcal{D}$ over $\mathcal{X}$, and labeled by $f^*$. The learner's population loss is given by $\textsc{Loss}_{\mathcal{D}}(h,f^*)=\mathbb{P}_{x\sim \mathcal{D}}\left[\textsc{Loss}(x; h, f^*)\right]$. This is formalized in the following.

\begin{definition}[PAC Learning with improvements]\label{def:pac-learning-with-improvements}
Algorithm $\mathcal{A}$ {\em PAC-learns with improvements} a concept class $\mathcal{H}$ with respect to improvement function $\Delta$ and data distribution $\mathcal{D}$ using sample size $M\coloneqq M(\epsilon,\delta,\Delta,\mathcal{H},\mathcal{D})$\footnote{We say the sample complexity of $\mathcal{A}$ is the smallest such $M$.}, if for any $f^* \in \mathcal{H}$, any $\epsilon > 0$ and $\delta>0$, the following holds. Algorithm \( \mathcal{A} \), with access to a sample $S\overset{\text{i.i.d.}}{\sim}\mathcal{D}^M$ labeled according to $f^*$, produces with probability at least \( 1 - \delta \) a hypothesis 
$h$ 
with \( \textsc{Loss}_{\mathcal{D}}(h,f^*)\le \epsilon \). We further say that $\mathcal{A}$ learns  $\mathcal{H}$ w.r.t.\ $\Delta$ and $\mathcal{D}$ with zero-error with sample size $M$ if for any $\delta>0$, given $S\overset{\text{i.i.d.}}{\sim}\mathcal{D}^M$ labeled by $f^*$, it returns 
$h$
with \( \textsc{Loss}_{\mathcal{D}}(h,f^*)=0\) with probability at least $1-\delta$. We will also consider distribution-independent learning, where the guarantee should hold for all distributions $\mathcal{D}$ and proper learning where we require $h \in \mathcal{H}$.

\end{definition}
\noindent Note that in our learning with improvements setting zero-error can be achieved by learning (with high probability) from a finite sample in several interesting cases, which is impossible to achieve in the standard PAC model.

\section{Separating PAC Learning with Improvements from the Standard and Strategic PAC Models}\label{sec:separating-PAC-and-improvements}

In this section, we prove that learning with improvements diverges from the 
behavior of the standard PAC model for binary classification, and also from the more recently studied PAC learning model for strategic classification \cite{hardt2016strategic,strategicPAC}.

\subsection{Comparison with the standard PAC model}

In the standard PAC model, the learnability of a concept class is equivalent to the class having a finite VC dimension. However, in our setting, where agents can improve, this condition is neither necessary nor sufficient for learnability. Concretely, we demonstrate that a class with an infinite VC dimension can still be learnable with improvements. We also provide examples of hypothesis classes with finite VC dimensions and corresponding improvement sets that cannot be learned in our framework. 
\begin{theorem}
    Finite VC dimension is neither necessary nor sufficient for PAC learnability with improvements. 
\end{theorem}

\begin{proof}
    The proof is in Examples \ref{ex:example_notnecessary} and \ref{ex:example_notsufficient} below.
\end{proof}

\begin{example}[Finite VC dimension is not necessary for learnability with improvements]
    Consider any class $\mathcal{H}$ of infinite VC-dimension, and define $\Delta(x)=\mathcal{X}$ for all examples $x\in \mathcal{X}$.  We can learn this class $\mathcal{H}$ with respect to this improvement function $\Delta$ with sample complexity $M(\epsilon,\delta)=\frac{1}{\epsilon}\ln(\frac{1}{\delta})$ for any data distribution $\mathcal{D}$ as follows.  First, draw a sample $S$ of size $M(\epsilon,\delta)$. Next, if all examples in $S$ are negative, then output the ``all-negative'' classifier; otherwise, select any positive example $x^* \in S$ and output the classifier $h(x)=\mathbb{I}[x=x^*]$. Note that in the latter case, the hypothesis $h$ has error zero, because all agents will improve to $x^*$.  Therefore, if $\Pr_{x\sim \mathcal{D}}[f^*(x)=1] > \epsilon$, then $h$ will have zero error with probability at least $1-\delta$, whereas if $\Pr_{x\sim \mathcal{D}}[f^*(x)=1] \leq \epsilon$, then $h$ will have error at most $\epsilon$ with probability $1$.
    
    \label{ex:example_notnecessary}
\end{example}

\begin{example}[Finite VC dimension is not sufficient for learnability with improvements]
    Let the instance space $\mathcal{X}$ be $[0,1]$, let ${\mathcal H}=\{h_{abcd} : h_{abcd}(x)=1 \mbox{ iff }x\in [a,b)\cup(c,d]\}$ (i.e., $\mathcal{H}$ is the class of unions of two intervals, where to make the example easier, we define the intervals to be half-open), and let $\mathcal{D}$ be the uniform distribution over $[0,1]$. We define $\Delta$ as follows.  For $x\in [0,1/4) \cup (3/4,1]$ let $\Delta(x)=[0,1]$; for $x\in [1/4,3/4]$, let $\Delta(x)=\{x\}$. 

    We claim that no algorithm with finite training data can guarantee an expected error of less than $1/4$, even though the class is easily PAC-learnable without improvements.

    Consider a target function defined as the union of two intervals $[1/4, b) \cup (b, 3/4]$ where the number $b$ was randomly chosen in $[1/4, 3/4]$.  With probability $1$, the learner will not see the point $b$ in its training data, so it learns nothing from its training data about the location of $b$.  Finally, if the learner outputs a classifier whose positive region has probability mass $\leq 1/4$, then its error rate is at least $1/4$ because the positive examples cannot move so at least half of their probability mass will get misclassified.  On the other hand, if the learner outputs a classifier whose positive region has probability mass greater than $1/4$, then it has at least a $50 \%$ chance of including a negative point in its positive region (it will surely include a negative point if it is not contained in $[1/4, 3/4]$ and has at least a 50\% chance of doing so otherwise, since $b$ was uniformly chosen from $[1/4,3/4]$).  If the classifier has a negative point in its positive region, then it will have an error rate at least $50\%$, because all the negatives in $[0, 1/4)$ and $(3/4, 1]$ will move to a false positive (here we use that agents break ties adversarially).  So, either way, its expected error is $25\%$.
\label{ex:example_notsufficient}
\end{example}

Union of two intervals is arguably the simplest class that is not intersection-closed (Definition \ref{def:intersection-closed}). Indeed, we show in Section~\ref{sec:intersection-closed} that such an example could not be possible for intersection-closed classes. 

As another example of the separation of our model from the standard PAC setting, we show it is possible that for some target function $f^*$ the learner using a hypothesis space $\tilde{\mathcal{H}}$ achieves no error in the standard PAC setting, but it is impossible to avoid a constant error rate when learning the same target with improvements using the same hypothesis space $\tilde{\mathcal{H}}$.

\begin{theorem}
    Consider a hypothesis class $\tilde{\mathcal{H}}$ and target function $f^* \not\in \tilde{\mathcal{H}}$. Let $d(f^*,\tilde{\mathcal{H}})$ denote the error of the best hypothesis in $\tilde{\mathcal{H}}$ w.r.t.\ the target $f^*$.  It is possible to have $d(f^*,\tilde{\mathcal{H}})=0$ in the standard PAC setting (there exists $h\in \tilde{\mathcal{H}}$ that achieves error $0$) but $d(f^*,\tilde{\mathcal{H}}) = 1/2$ for PAC learning with improvements (i.e., all  $h\in \tilde{\mathcal{H}}$ have error at least $1/2$).
\end{theorem}
\begin{proof}
    See Example~\ref{ex:standard_improve} in Appendix~\ref{app:example-error-gap}. 
\end{proof}

\subsection{Comparison with the  PAC model for strategic classification}\label{sec:comparison-strategic}

We first observe that the  strategic classification loss can be obtained by a subtle modification to our loss function (\ref{def:loss-function}),

\begin{align*}\label{def:loss-strategic}
    \textsc{Loss}^{\textsc{Str}}(x; h, f^*)=\max_{x'\in \Delta_{h}(x)} \mathbb{I}\left[h\left(x'\right)\neq f^*(x)\right]. 
\end{align*}

\noindent Intuitively,  for a negative point with $f^*(x)=0$, $\Delta_h(x)$ here denotes the set of points that the agent $x$ can ``pretend'' to be in order to potentially deceive the classifier $h$ into incorrectly classifying the agent positive. Since the movement within $\Delta_h(x)$ is viewed as a manipulation by the agent $x$, the prediction on the strategically perturbed point is compared with the original label of $x$, i.e.\ $f^*(x)$.

Prior work has shown that learnability in the strategic classification setting is captured by the {\it strategic VC dimension} (SVC) introduced by Sundaram et al.~\citep{strategicPAC}. We state below the   definition of SVC, adapted to our improvement-set based setting above which is a special case of the general cost functions studied under the strategic classification setting.

\begin{definition}[Strategic VC dimension, \cite{strategicPAC}]\label{def:strategic-vc}
    Define the $n$-th shattering coefficient of a strategic classification problem as
\[\sigma_n(\mathcal{H},\Delta)=\max_{(x_1,\dots,x_n)\in\mathcal{X}^n}|\{(h(x_1'),\dots,h(x_n')):h\in\mathcal{H},x_i'\in\Delta_h(x_i)\}|.\]
\noindent Then SVC$(\mathcal{H},\Delta)=\sup\{n\ge 0:\sigma_n(\mathcal{H},\Delta)=2^n\}$.
\end{definition}

\noindent A natural question to ask is whether learning with improvements is ``easier'' than strategic classification. That is, if a concept space $\mathcal{H}$ is learnable w.r.t.\ $\Delta$ and $\mathcal{D}$ in the strategic classification setting, then is it also learnable with improvements? Interestingly, we answer this question in the negative. More precisely, we show that finite SVC (which is known from prior work to be a sufficient condition for  strategic PAC learning) is actually not a  sufficient condition for PAC learnability with improvements.

\begin{theorem}\label{thm:strategic-easier}
    Finite strategic VC dimension \cite{strategicPAC} does not necessarily imply PAC learnability with improvements. 
\end{theorem}

\begin{proof}
Let the instance space $\mathcal{X}$ be $[0,1]$, let ${\mathcal H}=\{h_{abcd} : h_{abcd}(x)=1 \mbox{ iff }x\in [a,b)\cup(c,d]\}$, and let $\mathcal{D}$ be the uniform distribution over $[0,1]$. We define $\Delta$ as follows.  For $x\in [0,3/4)$ let $\Delta(x)=\mathcal{B}(x,1/4)=(x-1/4,x+1/4)\cap [0,1]$; for $x\in [3/4,1]$, let $\Delta(x)=\{x\}$. 

We claim that no algorithm with finite training data can guarantee an expected error of less than $1/16$ for the above  when learning with improvements, even though the class is  PAC-learnable in the strategic classification setting. To see the latter, note that SVC$(\mathcal{H},\Delta)\le 4$. Indeed, consider the points $(0,1/4,1/2,3/4,1)\in\mathcal{X}^5$. Notice the (strategic) labeling $(1,0,1,0,1)$ cannot be achieved for any $h\in\mathcal{H}$, which establishes the claim.

Now consider a target function defined as the union of two intervals $[1/2, b) \cup (b, 1]$ where the number $b$ was randomly chosen in $[3/4, 1]$.  The learner will not see the point $b$ given a finite training set, so it learns nothing about the location of $b$ (almost surely).  Now, we consider two cases. Either, the learner outputs a classifier whose positive region has probability mass at most $1/16$ over the interval $[3/4, 7/8]$. Then its error rate is at least $1/16$ because the positive examples in $[3/4,7/8]$ cannot move so at least half of their probability mass will get misclassified.  On the other hand, if the learner outputs a classifier whose positive region has probability mass greater than $1/16$ on the interval $[3/4,7/8]$, then it has at least a $50 \%$ chance of including the negative point $b$ in its positive region (over the random choice of the target function).  If the classifier has a negative point in $[3/4,7/8]$ that is incorrectly predicted to be positive, then it will have an error rate at least $1/16$, because all the positives in $[5/8, 3/4)$ will move to a false positive (here we use that agents break ties adversarially, see also Remark \ref{rmk:misleading-improvement}).  So, either way, its expected error is  $\ge 1/16$.
\end{proof}

\noindent On the other hand, it is not too hard to come up with examples where it is easier to learn in the improvements setting when compared to the strategic setting. Example~\ref{ex:improveasier} shows that it is possible to learn perfectly with improvements (with zero error) in a setting where avoiding a large constant error is unavoidable in the strategic classification setting.

\begin{example}[Learnability with improvements may be easier than strategic classification]
    Define $\Delta(x)=\mathcal{X}$ for all examples $x\in \mathcal{X}$. Suppose the ``all-negative'' classifier $h_-(x)=0$,  the ``all-positive'' classifier $h_+(x)=1$, and all ``singleton-positive'' classifiers $h_{x^*}(x)=\mathbb{I}[x=x^*]$ lie in the concept space $\mathcal{H}$. Select any $f^*\in\mathcal{H}$ and any data distribution $\mathcal{D}$ over $\mathcal{X}$ such that $\mathbb{P}_{x\sim \mathcal{D}}[f^*(x)=0]=\mathbb{P}_{x\sim \mathcal{D}}[f^*(x)=1]=\frac{1}{2}$. Now with $O(\log \frac{1}{\delta})$ examples, the learner sees a positive example, say $x^+$, in its training set with probability $\ge\delta$. Outputting $h_{x^+}$ achieves zero-error in the learning with improvements setting, as all negative points can improve to $x^+$. In contrast, a learner in the strategic classification setting must suffer an error of at least $1/2$ here. Indeed, either the learner outputs $h_-$ and suffers an error of $1/2$ on the positive points. Or, the learner selects an $h$ that labels at least one point as positive and incurs an error $1/2$ on the negative points, all of which  successfully deceive the learner.
\label{ex:improveasier}
\end{example}

\noindent Furthermore, let's consider an improvement function $\Delta$ that takes into account $f^*$, such that $f^*(x')=1$ for $x'\in\Delta(x)$ for any $x\in\mathcal{X}$. That is the improvement function is in a certain sense consistent with $f^*$, guaranteeing positive classification after any move. In this setting, any classifier $h$ will have lower error in the improvements setting compared to strategic classification. This is because a negative point that moves and becomes positive is an error in  strategic learning but the point would have genuinely improved in this case.
Contrasting this with Remark \ref{rmk:misleading-improvement}, when $\Delta$ does not satisfy the above property, we note that the reason it is possible to do worse in the improvement setting (e.g.\ Theorem \ref{thm:strategic-easier}) is because some true positive examples can potentially become negative when moving in response to a false positive for the learner's hypothesis $h$.

\section{PAC Learning of Geometric Concepts}\label{sec:geometric-concepts}
In this section, we first demonstrate the gain of the learner when agents can improve for the natural class of thresholds on the real line, where agents can move by a distance of at most $r$. We then study intersection-closed classes. In particular, we derive sample complexity bounds for the class of axis-aligned hyperrectangles, where the improvement sets are the $\ell_\infty$ balls. We further establish negative results for proper learners in the absence of the intersection-closed property. Lastly, we study the class of homogeneous halfspaces under the uniform distribution over the unit ball, where agents can improve by adjusting their angle. Complete proofs for this section are located in Appendix~\ref{app:proofs-geometric}.\looseness-1

We will use $(a)_+$ to denote $\max\{a,0\}$.

\subsection{Warm-up: Zero  Error for Learning Thresholds}\label{subsec:thresholds}
Let \(\mathcal{H} = \{ h_t : t \in \mathbb{R} \}\) be the class of one-sided threshold functions, where $h_t(x) = \mathbb{I}\{x \geq t\}$.

The improvement set of $x$ is simply the closed ball centered at $x$ with radius $r$, i.e., $\Delta(x) =\{ x' \mid \lvert x - x'\rvert \leq r\} $.
Suppose the data distribution $\mathcal{D}$ is uniform over $[0,1]$, and labels are generated according to a target threshold $h_{t^*} \in \mathcal{H}$ for some $t^* \in [0,1]$. Let $S = \{(x_i, y_i)\}_{i=1}^m$ be the set of training samples, where $x_i \overset{\text{i.i.d.}}{\sim} \mathcal{D}$ and $y_i = h_{t^*}(x_i)$. 

There are several options for choosing a threshold that achieves zero empirical error on $S$, as shown by the shaded area in Figure~\ref{fig:threshold}. Due to the asymmetry of the loss function (Eqn.~\ref{def:loss-function}), we choose the rightmost threshold consistent with $S$. This is the most ``conservative'' option, as any $x$ that improves up to this threshold is guaranteed to be positive with respect to the unknown ground-truth $h_{t^*}$.  This is a property that would not necessarily hold for lower thresholds. We define this threshold with respect to $S$ as follows,

\[
t_{S^+} = 
\begin{cases} 
\min(S^+), & \text{if } S^+ \neq \emptyset, \\
1, & \text{if } S^+ = \emptyset,
\end{cases}
\]
where \(S^+ = \{x_i \in S : y_i = 1\}\) is the set of positive examples in \(S\).
The hypothesis \(h_{S^+}\) is defined as $h_{S^+} = \mathbb{I}\{x \geq t_{S^+}\}$.

Notice that using classifier $h_{S^+}$ will induce agents (at test time) $x\in [t_{S^+}-r,t_{S^+})$ 
to improve to be classified as positive by $h_{S^+}$, which will be a correct classification since $t_{S^+} \geq t^*$. 

\begin{figure}[!t]
    \centering
    \begin{tikzpicture}
        \node (m1) {\Large \bf \strut --};
        \node[right=1mm of m1] (m2) {\Large \bf \strut--};
        \node[right=1mm of m2] (m3) {\Large \bf \strut--};
        \node[right=1mm of m3] (m4) {\Large \bf \strut--};

        \node[rectangle, fill=purple!10, thick, minimum width=0.75cm, minimum height=0.5cm, right=-1.0mm of m4] (r1) {};
        \node[rectangle, fill=blue!10, thick, minimum width=0.75cm, minimum height=0.5cm, right=0mm of r1] (r2) {};
        \node[rectangle, fill=green!10, thick, minimum width=0.75cm, minimum height=0.5cm, right=0mm of r2] (r3) {};

        \coordinate (m1) at ($(r1)!0.5!(r2)$);
        \draw[thick] (m1) -- ++(0,-0.26) -- ++(0,0.52);
        \coordinate (m2) at ($(r2)!0.5!(r3)$);
        \draw[dashed] (m2) -- ++(0,-0.26) -- ++(0,0.52);
        \coordinate (m3) at (r3.east);
        \draw[thick] (m3) -- ++(0,-0.26) -- ++(0,0.52);

        \draw[<->, thick] ([yshift=3mm]r1.west) -- ([yshift=3mm]r1.east) node[midway, above] {$r$}; %
        \draw[<->, thick] ([yshift=3mm]r3.west) -- ([yshift=3mm]r3.east) node[midway, above] {$r$};

        \coordinate (ht) at ($(r1.south)!0.5!(r2.south)$);
        \node[below=0.0mm of ht] {$h_{t^*}$};

        \coordinate (hs) at ($(r2.south)!0.5!(r3.south)$);
        \node[below=0.0mm of m3] at (m3 |- ht) {$h_{S^+}$};

        \node[right=-1.0mm of r3] (p1) {\Large \bf \strut +};
        \node[right=1mm of p1] (p2) {\Large \bf \strut+};
        \node[right=1mm of p2] (p3) {\Large \bf \strut+};
        \node[right=1mm of p3] (p4) {\Large \bf \strut+};
    \end{tikzpicture}
    \caption{Learning thresholds with improvements.
\label{fig:threshold}}
\end{figure}
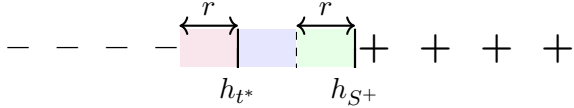
\begin{theorem}[Thresholds, uniform distribution]
    Let the improvement set $\Delta$ be the closed ball with radius $r$, 
    $\Delta(x) =\{ x'\mid \lvert x - x' \rvert \leq r\} $.
    Let $\mathcal{D}$ be the uniform distribution on $[0,1]$. For any \(\epsilon,\delta \in (0,1/2)\), with probability $1-\delta$,
\begin{align*}
   \textsc{Loss}_{\mathcal{D}}(h_{S^+},h_{t^*})
   \leq (\epsilon -r)_+,
\end{align*}
with sample complexity $M = O\left(\frac{1}{\epsilon} \log \frac{1}{\delta}\right).$
\label{thm:thresholds-uniform}
\end{theorem}
\begin{proof}
    See Appendix~\ref{subsec:threshold-uniform-D}.
\end{proof}

\noindent Note that the population error is improved from $\epsilon$ (in the standard PAC setting) to $\epsilon-r$ for the same sample size, and we can achieve zero error as long as we set $\epsilon\leq r$.\\
In Appendix~\ref{subsec:threshold-arbitrary-D}, we prove a similar result for arbitrary distribution $\mathcal{D}$, where instead of getting $\epsilon-r$ population error, the reduction in the error  
is replaced by the following distribution-dependent quantity
    \begin{align}\label{eq:IR-threshold}
        p(h_{S^+};h_{t^*},\mathcal{D},r)= \mathbb{P}_{x\sim \mathcal{D}}\left[x\in[t_{S^+}-r,t_{S^+}]\right].
    \end{align}

\noindent Note that the class of thresholds is closed under intersection: 
$
\bigcap_{i=1}^n h_{t_i} = h_{\max\{t_1, t_2, \dots, t_n\}}.
$
In the following, we extend the analysis to such hypothesis classes, more generally.
\subsection{Intersection Closed Classes}\label{sec:intersection-closed}

The learnability of intersection-closed hypothesis classes in the standard PAC model has been extensively studied \cite{Helmbold1990,auer1997learning,auer1998line,auer2007new,darnstadt2015optimal}. In this section, we study  the learnability with improvements of these classes. 
We start with the following definitions. 

\begin{definition}[Closure operator of a set]
    For any set $S\subseteq \mathcal{X}$ and any hypothesis class $\mathcal{H}\subseteq 2^\mathcal{X}$, the \emph{closure of $S$ with respect to $\mathcal{H}$}, denoted by $\clos_\mathcal{H}(S):2^\mathcal{X}\rightarrow 2^\mathcal{X}$, is defined as the intersection of all hypotheses in $\mathcal{H}$ that contain $S$, that is, $\clos_{\mathcal{H}}(S)=\underset {h\in {\mathcal{H}}, S\subseteq h}{\bigcap} h$. In words, the closure of $S$ is the smallest hypotheses in ${\mathcal{H}}$ which contains $S$. If $\lrset{h:{\mathcal{H}}: S\subseteq h}=\emptyset$, then $\clos_{\mathcal{H}}(S)=\mathcal{X}$.
    
\end{definition}

\begin{definition}[Intersection-closed classes]

    A hypothesis class $\mathcal{H}\subset 2^\mathcal{X}$ is \emph{intersection-closed} if for all finite $S\subseteq \mathcal{X}$, $\clos_{\mathcal{H}}(S)\in \mathcal{H}$. In words, the intersection of all hypotheses in $\mathcal{H}$ containing an arbitrary subset of the domain belongs to $\mathcal{H}$. For finite hypothesis classes, an equivalent definition states that for any $h_1,h_2\in \mathcal{H}$, the intersection $h_1\cap h_2$ is in $\mathcal{H}$ as well \cite{natarajan1987learning}.
    \label{def:intersection-closed}
\end{definition}

\noindent Many natural hypothesis classes are intersection-closed, for example, axis-parallel $d$-dimensional hyperrectangles, intersections of halfspaces, $k$-CNF boolean functions, and subspaces of a linear space.

The \textit{Closure algorithm} is a learning algorithm that generates a hypothesis by taking the closure of the positive examples in a given dataset, and negative examples do not influence the generated hypothesis. The hypothesis returned by this algorithm is always the smallest hypothesis containing all of the positive examples seen so far in the training set.
\begin{definition}[Closure algorithm \citep{natarajan1987learning,Helmbold1990}]
    Let $f^*\in {\mathcal{H}}$ and let $S=\{(x_1,y_1=f^*(x_1)), $ $\ldots,  (x_m, y_m=f^*(x_m))\}$ be a set of labeled examples, where  each  \(x_i \in \mathcal{X}\). The hypothesis \(h^c_S\) produced by the closure algorithm is defined as:
    \[
    h^c_S(x) =
    \begin{cases}
    1, & \text{if } x \in \clos_{\mathcal{H}}\left(\{x_i \in S : y_i = 1\}\right), \\
    0, & \text{otherwise}.
    \end{cases}
    \]
    Here, $\clos_{\mathcal{H}}\left(\{x_i \in S : y_i = 1\}\right)$ denotes the closure of the set of positive examples in $S$ with respect to ${\mathcal{H}}$.\label{def:closure-algorithm}
\end{definition}

\noindent The closure algorithm learns intersection-closed classes with VC dimension $d$ with an optimal sample complexity of $\Theta\left(\frac{1}{\epsilon}(d+\log \frac{1}{\delta})\right)$ \citep{auer2004new,darnstadt2015optimal}.

We apply the closure algorithm for learning with improvements. In order to quantify the improvement gain of the returned hypothesis, we define the \emph{improvement region} of $h$ as the set of points that can improve from a negative classification to a (correct) positive classification by 
$h$.

\begin{definition}[Improvement Region] The improvement region of hypothesis $h\subseteq f^*$, w.r.t. $f^*$ and $\Delta$ is
\begin{align}\label{def:IR}
\begin{split}
    \ir(h;f^*,\Delta)
    \coloneq
    \lrset{x:h(x)=0, \exists x'\in\Delta(x):h(x')=f^*(x')=1}.
\end{split}
\end{align}

The gain from improvements is the probability mass of the improvement region under $\mathcal{D}$: \\ $\mathbb{P}_{x \sim \mathcal{D}}\left[x \in \ir(h;f^*,\Delta) \right].$

\end{definition}
\noindent Note that for the class of thresholds, the closure algorithm returns exactly the hypothesis $h_{S^+}$, 
and the probability mass of the improvement region is $p(h_{S^+};h_{t^*},\mathcal{D},r)$ (cf.\ Eqn.~\ref{eq:IR-threshold}).

\paragraph{Axis-Aligned  Hyperrectangles in $[0,1]^d$.}
An axis-aligned hyperrectangle classifier assigns a value of 1 to a point if and only if the point lies within a specific rectangle. 
Formally, let $a=(a_1,\ldots,a_d),b=(b_1,\ldots,b_d)\in[0,1]^d$ where $a_i\leq b_i$ for $i\in\{1,\ldots,d\}\coloneq[d]$. 
A hyperrectangle 
$R_{(a,b)}=\prod_{i\in [d]}[a_i,b_i]$
classifies a point $x=(x_1,\ldots,x_d)$ as: 
$ R_{(a,b)}(x)=\mathbb{I}\{x_i
\in [a_i,b_i],\; \forall i \in [d]\}$.

In the following, we show that the closure algorithm learns with improvements the hypothesis class $\mathcal{H}_{\text{rec}}=\{R_{(a,b)}: a,b\in[0,1]^d\}.$

\begin{theorem}[Axis-aligned Hyperrectangles]\label{thm:rectangles}
    Let the improvement set $\Delta$ be the closed $\ell_\infty$ ball with radius $r$, $\Delta(x) =\{ x'\mid \norm{x - x'}_\infty \leq r\} $.
    Let $R_S^c$ be the rectangle returned by the closure algorithm given $S\overset{\text{i.i.d.}}{\sim} \mathcal{D}^m$, and $R^*$ be the target rectangle. For
    any distribution $\mathcal{D}$ 
    ,
    for any \(\epsilon,\delta \in (0,1/2)\), with probability $1-\delta$,
\begin{align*}
   \textsc{Loss}_{\mathcal{D}}(R_S^c,R^*)
   \leq \left(\epsilon - \mathbb{P}_{x \sim \mathcal{D}}\left[x \in \ir(R_S^c;R^*,\Delta) \right]\right)_+,
\end{align*}
with sample complexity $M = O\left(\frac{1}{\epsilon}\left(d+ \log \frac{1}{\delta}\right)\right)$.

When $\mathcal{D}$ is the uniform distribution on $[0,1]^2$, we can get the following expression.
Denote by $l_1$ and $l_2$ the width and height (respectively) of the rectangle $R_S^c$. Then, 
\begin{align*}
\mathbb{P}_{x \sim \mathcal{D}}\left[x \in \ir(R_S^c;R^*,\Delta) \right]
=
2r(l_1+l_2)+4r^2.
\end{align*}
\end{theorem}
\begin{proof}
    See Appendix~\ref{app:rectangles}.
\end{proof}

\noindent Note that, as opposed to the simple case of thresholds,  the improvement region for hyperrectangles depends on the geometry of the target hypothesis.

\paragraph{Arbitrary Intersection-closed Classes.} We will now show that any intersection-closed concept class with a finite VC dimension is PAC learnable with improvements {\it w.r.t.\ any} improvement function $\Delta$.
\begin{theorem}
    Let $\mathcal{H}$ be an intersection-closed concept class on instance space $\mathcal{X}$. There is a learner that PAC-learns with improvements $\mathcal{H}$ with respect to any improvement function $\Delta$ and any data distribution $\mathcal{D}$ given a sample of size $O\left(\frac{1}{\epsilon}(d_{\text{VC}}(\mathcal{H})+\log\frac{1}{\delta})\right)$, where $d_{\text{VC}}(\mathcal{H})$ denotes the VC-dimension of $\mathcal{H}$.

    \label{thm:intersection-closed}
\end{theorem}

\begin{proof}
    Let $S\sim\mathcal{D}^m$ and $h^c_S$ denote the classifier learned by the closure algorithm (Definition~\ref{def:closure-algorithm}). For some $m=O(\frac{1}{\epsilon}(d_{\text{VC}}(\mathcal{H})+\log\frac{1}{\delta}))$, we know from prior work \cite{auer2004new,darnstadt2015optimal} that $h^c_S$ satisfies, with probability at least $1-\delta$,
    \[\mathbb{P}_{x\sim \mathcal{D}}[h^c_S(x)\ne f^*(x)]\le \epsilon,\]
    for any target concept $f^*\in \mathcal{H}$.

    Now for any $x\in\mathcal{X}$ if $h^c_S(x)=1$, the improvement loss $\textsc{Loss}(x; h^c_S, f^*)=\mathbb{I}[h^c_S(x)\ne f^*(x)]=0$ since $\Delta_{h^c_S}(x)=\{x\}$ and $f^*(x)=1$ since $h^c_S$ is obtained using the closure algorithm. If $h^c_S(x)=0$ and if $\Delta_{h^c_S}(x)\ne\{x\}$, for any point $x'\in\Delta_{h^c_S}(x)$, we have $h^c_S(x')=1=f^*(x')$ and therefore $\textsc{Loss}(x; h^c_S, f^*)=0$. So the only points for which $h^c_S$ can make a mistake are points where $h^c_S(x)=0$ and $\Delta_{h^c_S}(x)=\{x\}$, i.e.\ the points do not move in reaction to ${h^c_S}$. This implies $h^c_S$ must disagree with $f^*$ on these points also in the PAC setting. But the probability mass of these points is at most $\epsilon$ as noted above.
\end{proof}

\noindent We can also establish the following negative result which indicates the hardness of proper learning in the absence of the intersection-closed property.

\begin{theorem}
    Let $\mathcal{H}$ be any concept class on a finite instance space $\mathcal{X}$ 
    such that at least one point $x'\in\mathcal{X}$ is classified negative by all $h\in \mathcal{H}$ (i.e.\ $\{x\mid h(x)=0 \text{ for all } h\in \mathcal{H}\}\ne\emptyset$), and suppose $\mathcal{H}\mid_{\mathcal{X}\setminus\{x'\}}$ is not intersection-closed on $\mathcal{X}\setminus\{x'\}$. Then there exists a data distribution $\mathcal{D}$ and an improvement function $\Delta$ such that no proper learner can PAC-learn with improvements $\mathcal{H}$ w.r.t.\ $\Delta$ and $\mathcal{D}$.
    \label{thm:hardness-intersection-closed}
\end{theorem}

\begin{proof}
    See Appendix~\ref{app:hardness-intersection-closed}.
\end{proof}

\noindent Note that we have an additional requirement that all classifiers in the concept space agree on some negative point 
(intuitively, agents who should never achieve positive classification). Consider the following simple example where this condition does not hold, the concept class is not intersection-closed, and learnability is possible in our setting.

    Suppose $\mathcal{X}=\{x_1,x_2\}$ and $\mathcal{H}=\{h_1,h_2\}$ with $h_1(x_1)=1,h_1(x_2)=0$ and $h_2(x)=1-h_1(x)$ for either $x\in\mathcal{X}$. Clearly $h_1\cap h_2\notin \mathcal{H}$ and  $\mathcal{H}$ is not intersection-closed, yet knowledge of a single label tells us the target concept.

\subsection{Halfspaces on the Unit Ball}

We now consider the problem of learning homogeneous halfspaces with respect to the uniform distribution on the unit ball (or any spherically-symmetric distribution), when agents have the ability to improve by an angle of $r$.
\begin{theorem}\label{thm:halfspaces}
 Consider the class of $d$-dimensional halfspaces passing through the origin, i.e., $\mathcal{H} = \{x \mapsto \operatorname{sign}(w^Tx):w \in \R^d\}$. Suppose \( \mathcal{X} \) is the surface of the origin-centered unit sphere in \( \mathbb{R}^d \) for \( d > 2 \), and \( \mathcal{D} \) is the uniform distribution on \( \mathcal{X} \). 
For each point \( x \in \mathcal{X} \), define its neighborhood \( \Delta(x) = \{ x' \mid \arccos(\langle x, x' \rangle) \leq r \} \). For any \(\delta \in (0,1/2)\), and training sample $S \overset{\text{i.i.d.}}{\sim} \mathcal{D}^m$ of size $\tilde{O}\left(\frac{d + \log\frac{1}{\delta}}{r}\right)$,  with probability $1-\delta$,
$\textsc{Loss}_{\mathcal{D}}({\rm POS}_{\mathcal{H}}(S),f^*)=0,$ where ${\rm POS}_{\mathcal{H}}(S)$ is the  intersection of the positive regions of all $h \in \mathcal{H}$ consistent with the training set $S$.

\end{theorem}

\begin{proof}
The algorithm will use ${\rm POS}(\mathcal{H}_S)$ as its classifier; that is, a point $x'$ is classified as positive if {\em every} $h\in \mathcal{H}$ consistent with $S$ labels $x'$ as positive.  Note that this classifier will have zero loss under improvement function $\Delta$ if every $h\in \mathcal{H}$ consistent with $S$ has angle at most $r$ with $f^*$; this is because any positive example $x$ can then move into the positive agreement region simply by moving an angular distance $r$ in the direction of the normal vector to $f^*$.  So, all that remains is to show that after $m$ examples, with probability at least $1-\delta$, every $h\in \mathcal{H}$ consistent with $S$ has angle at most $r$ with $f^*$.

Consider some $h \in \mathcal{H}$ given by \( h(x) = \operatorname{sign}(w^Tx) \).
The probability mass lying in the disagreement region of $h$ and $f^*$ is:
\begin{align}
\rho_{\mathcal{D}} (h,f^*) = \mathbb{P}_{x \sim \mathcal{D}}[h(x) \neq f^*(x)] = \frac{\arccos(\langle w, w^* \rangle)}{\pi}.
\label{unitballactive}
\end{align}
Therefore, to have the property that every $h\in \mathcal{H}$ consistent with $S$ has angle at most $r$ with $f^*$, we just need that every $h\in \mathcal{H}$ of error greater than $\frac{r}{\pi}$ should make at least one mistake on $S$.  By standard realizable VC bounds, it suffices to have
\[
m =  \mathcal{\tilde{O}}\left(\frac{d+\log\frac{1}{\delta}}{r}\right)
\]
number of i.i.d. samples for this to hold with probability at least $1-\delta$.
\end{proof}

\begin{remark}
  It is worth noting that while the aforementioned result relies on the classifier ${\rm POS}(\mathcal{H}_S)$, which is a fairly complex function, a similar guarantee can be achieved using a linear classifier (though non-homogeneous, so it is still not ``proper''). Specifically, by obtaining a sufficiently large sample, one can construct a homogeneous linear classifier whose angle with respect to the target is at most \( \frac{r}{2} \). We can then shift that classifier by $r/2$ (so it is no longer homogeneous) to ensure its positive region is contained inside the positive region for $f^*$.
\end{remark}

\section{Zero-error Learning in the Graph Model}\label{sec:graph-model}

In this section, we will consider a general discrete model for studying classification of agents with the ability to improve. The agents are located on the nodes of an undirected graph, and the edges determine the improvement function, i.e.\ the agents can move to neighboring nodes in order to potentially improve their classification. Note that the graph nodes correspond to an arbitrary discrete instance space $\mathcal{X}$. Remarkably, zero error may be attained even in this general setting. All proofs in this Section are deferred to Appendix~\ref{app:graph}.\looseness-1

Formally, let \( G = (V, E) \) denote an undirected graph. The vertex set \(V = \{x_1, x_2, \dots, x_n\}\) represents a fixed collection of \(n\) points corresponding to a finite instance space $\mathcal{X}$. 
The edge set \( E \subseteq V \times V \) captures the adjacency information relevant for defining the improvement function. More precisely, for a given vertex \( x \in V \), the improvement set of \( x \) is given by its neighborhood in the graph, i.e.\
$\Delta(x) = \{x' \in V \ | \ (x, x') \in E\}$\footnote{Our results readily extend to $\Delta(x) = \{x' \in V \ | \ d_G(x, x') \le r\}$, where $d_G$ denotes the shortest path metric on $G$, by applying our arguments to $G^r$, the $r^\textrm{th}$ power of $G$ (see appendix).}. 
Let \( f^*: V \to \{0, +1\} \) represent the target labeling (or partition) of the vertices in the graph \( G \). Assume that 
the hypothesis space \( \mathcal{H} \) is the set of all possible labelings of the graph, which is finite. 
\subsection{Near-tight Sample Complexity for Zero-Error}

Our first result is to show that we can obtain zero-error in the learning with improvements setting, when the data distribution $\mathcal{D}$ is given by a uniform distribution over $V$, and obtain near-tight bounds on the sample complexity. Our learner in this case is the ‘‘conservative'' classifier $h\in\mathcal{H}$ that classifies exactly the positive points seen in the sample as positive, and the remaining points as negative. Even though we allow $f^*$ to be an arbitrary labeling in $\mathcal{H}$, we do not need to see all the labels to learn an $h$ that achieves zero error w.r.t.\ $f^*$. Intuitively, this is because for any positively labeled node $x$ it is sufficient to see the label of $x$ or that of one of its neighbors since the agents can move to a neighbor predicted positive by $h$. We further show that no algorithm can achieve a better sample complexity, up to some logarithmic factors.

\begin{theorem}
Let \( G = (V, E) \) be an undirected graph with \( n = |V| \) vertices, and let \( f^*: V \to \{0, +1\} \) denote the ground truth labeling function. 
Let \( d_{\min}^+ \) denote the minimum degree of the vertices in \( G^+ \), the induced subgraph of $G$ on the vertices  $x\in V$ with $f^*(x) = 1$. Assume that the data distribution \(\mathcal{D}\) is uniform on \( V \). For any $\delta>0$, 
and training sample $S \overset{\text{i.i.d.}}{\sim} \mathcal{D}^m$ of size  $m =  O \left(\frac{n (\log n + \log \frac{1}{\delta})}{d_{\min}^+ + 1}\right)$, there exists a learner that achieves zero generalization error, i.e.\ learns a hypothesis $h$ such that $\textsc{Loss}_{\mathcal{D}}(h,f^*)=0$, with probability at least $1-\delta$ over the draw of $S$. Moreover, there exists a graph $G$ for which any learner that achieves zero generalization error must see at least $\Omega\left(\frac{n}{d_{\min}^+ + 1} \log \frac{n}{d_{\min}^+ + 1}\right)$ labeled points in the training sample, with high constant probability.
\label{DSSamplecomplexity}
\end{theorem}

\begin{proof} Given a sample $S$ labeled by $f^*$, let $S^+=\{x \in S \mid f^*(x) = +1\}$ denote the set of positive points in $S$. To achieve the claimed upper bound on the sample complexity of learning with zero-error, the learner outputs $h_S$ with $h_S(x)=\mathbb{I}\{x\in S^+\}$, that is the classifier which positively classifies exactly the points in $S^+$. We will now show that with a sample of size $m=O \left(\frac{n (\log n + \log \frac{1}{\delta})}{d_{\min}^+ + 1}\right)$, the proposed $h_S$ achieves zero generalization error with probability at least $1-\delta$.

Let \( V^+ = \{x \in V \mid f^*(x) = +1\} \) denote the set of vertices in $G^+$. We say that $x\in V^+$ is {\it covered} by the sample $S$ if $x\in S$ or there exists $x'\in S$ such that $x'\in V^+$ and $(x,x')\in E$. Note that if every $x\in V^+$ is covered by $S$, then  $\textsc{Loss}_{\mathcal{D}}(h_S,f^*)=0$ (formally established in Theorem~\ref{thm:dominating-set-teaching}). It is therefore sufficiently to bound determine the sample size needed to guarantee that every positive vertex is covered with high probability. 

Let \( x \in V^+ \) be a vertex in the positive subgraph \( G^+ \). For \( x \) to be covered, it must either be included in the sample $S$, or have at least one of its neighbors \( x' \in \Delta(x) \) included in $S$ . The probability of sampling \( x \) directly in one draw is \( \frac{1}{n} \). The probability of sampling any of its neighbors is proportional to its degree in \( G^+ \). Thus, the total probability of covering \( x \) in one draw is:
\[
p_{\text{cover}}(x) = \frac{1}{n} \cdot \left(1 + d(x)\right),
\]
where \( d(x) \) is the degree of \( x \) in \( G^+ \). Since \( d(x) \geq d_{\min}^+ \), we have:
\[
p_{\text{cover}}(x) \geq \frac{d_{\min}^+ + 1}{n}.
\]

To ensure the desired coverage holds with probability at least \( 1 - \delta \), we analyze the failure probability for a single vertex. The probability that a given vertex \( x \in V^+ \) is not covered after \( m \) samples is
\[
\mathbb{P}[x \text{ is not covered}] \leq \left(1 - \frac{d_{\min}^+ + 1}{n}\right)^m.
\]

\noindent To ensure that this holds for all \( |V^+| \leq n \) vertices, we apply the union bound
\begin{align*}
    \mathbb{P}[\exists x \in V^+ \text{ not covered}] \leq n \left(1 - \frac{d_{\min}^+ + 1}{n}\right)^m.
\end{align*}

\noindent Therefore, the sample size \( m \) required to ensure that the probability of the above bad event is at most $\delta$ is given by 
\begin{align*}
    m = O\left(\frac{n (\log n + \log \frac{1}{\delta})}{d_{\min}^+ + 1}\right).
\end{align*}

To establish the lower bound, consider a graph \( G \) on \( n \) vertices 
consisting of \( k \) disjoint cliques, each of the same size \( \frac{n}{k} = d_{\min}^+ + 1 \), see Figure~\ref{posclique}. Now note that if our sample $S$ does not contain any node from any one of the cliques (say $C$), then zero-error is not possible. This is because, one of two cases occur. If the learner's hypothesis $h$ predicts any point in $C$ as positive then we can select an $f^*$ that predicts $C$ entirely as negative while being consistent with $S$, causing the population loss to be at least $\frac{1}{k}$. On the other hand, if the learned $h$ predicts all points in $C$ as negative, then we can set $f^*$ to label $C$ entirely as positive, again incurring a population loss of at least $\frac{1}{k}$.

\begin{figure}[ht]
\vskip 0.2in
\begin{center}
\centerline{\includegraphics[width=0.4\columnwidth]{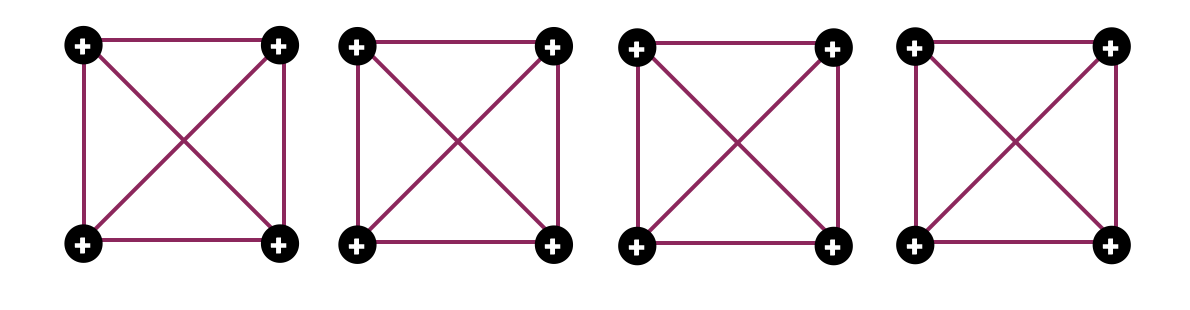}}
\caption{The graph $G$ used to establish our lower bound on the zero-error sample complexity. The graph consists of $k$ components, each of size $\frac{n}{k}$.}
\label{posclique}
\end{center}
\vskip -0.2in
\end{figure}

Our goal therefore is to determine a lower bound on the number of points required to ensure that every clique has at least one of its vertices included in the training sample $S$, which ensures that for every positive vertex, either the vertex itself or one of its neighbors is included. Using the standard coupon collector analysis, the number of trials needed to collect $k = \frac{n}{d_{\min}^+ + 1}$ coupons is $\Omega(k\log k)$ with high constant probability.
\end{proof}

\noindent We note that the value of \( d_{\min}^+ \) (and, therefore our bound on the sample complexity) is generally not known to the learner in advance, and it depends on the graph structure as well as the (unknown) target labeling  $f^*$. It is an interesting open question to design a learner that can determine whether a sample of sufficient size has been collected to guarantee zero-error. For the special case of  the complete graph,  our sample complexity bound becomes $m=\tilde{O}\left(\frac{n}{n^+}\right)$, where $n^+$ is the number of nodes labeled positive by $f^*$.

\subsection{Enabling Improvement Whenever It Helps}\label{sec:enabling-improvement}

Note that our loss function $\textsc{Loss}(x; h, f^*)$ 
penalizes the learner for mistakes w.r.t.\ the target $f^*$ after the agents have potentially reacted to $h$. However, we say nothing about whether a negative point $x$ that truly has the ability to improve and get positively classified (i.e.\ $f^*(x')=1$ for some $x'\in\Delta(x)$) will also be able to do so under our published classifier $h$. We will 
now consider an alternative measure of the performance of $h$ (conceptually captures recall in the improvement setting) which measures the probability mass of the points for which we fail to enable improvement even though it is possible under $f^*$. Formally, we define

\begin{eqnarray}\label{def:loss-function-enabling}
    \textsc{Loss}^{\textsc{e}}_{\mathcal{D}}(h, f^*) 
    =
    \mathbb{P}_{x \sim\mathcal{D}}\left[f^*(x)=0 \land  \mathbb{I}[\Delta_{f^*}(x) = \{x\}] = \mathbb{I}[\Delta_h(x) =\{x\}] \right]. \nonumber
\end{eqnarray}

\noindent That is, we wish to ensure that an agent $x$ with $f^*(x)=0$ and an option to improve to a truly positive point in its reaction set w.r.t.\ $f^*$, will also see some option to improve and get positively classified according to $h$.
In Theorem~\ref{thm:enable-improvement-and-zero-loss} in the Appendix, we obtain near-tight sample complexity bounds for learning a concept $h$ that simultaneously guarantees that $\textsc{Loss}^{\textsc{e}}_{\mathcal{D}}(h, f^*)=0$ and $\textsc{Loss}_{\mathcal{D}}(h, f^*)=0$ when the data distribution  is uniform over $V$.

\subsection{Teaching a Risk-Averse Student}
The theory of teaching \cite{GOLDMAN199520} studies the size of the smallest set of labeled examples needed to guarantee that a unique function in the concept space is consistent with the set (for labels according to any target concept in the space). If the teacher (that knows $f^*$) provides this labeled set, then a student that can do consistent learning (find a concept consistent with training data) will learn the target concept $f^*$.
In our learning with improvements over the graph setting, it is natural to consider a simple variant where the student outputs the most risk-averse concept that only labels positive points seen in the labeled set received from the teacher as positive. Here we will consider the question of the minimum number of labeled examples the teacher needs to provide to the risk-averse student to achieve zero-error in our setting.

Let $G^+$ denote the induced subgraph of $G$ on \( V^+ = \{x \in V \mid f^*(x) = +1\} \), the nodes labeled positive by the target concept $f^*$. We show that it is sufficient for the teacher to present the labels of a \emph{dominating set} of $G^+$ (Definition~\ref{defDS}) for the risk-averse student to learn a zero-error classifier $h$. This observation also motivates and helps establish our learning result (Theorem~\ref{DSSamplecomplexity}). 
\begin{theorem}\label{thm:dominating-set-teaching}
    Let \( G = (V, E) \) be an undirected graph, and  \( f^*\) 
    be the target labeling. Let \( G^+ \) denote the induced subgraph on the vertices  $x\in V$ with $f^*(x) = 1$, and $S^+$ denote the dominating set of $G^+$. Then $\textsc{Loss}(x; h_{S^+},f^*)=0$ for any $x\in V$, where $h_{S^+}(x)=\mathbb{I}\{x\in S^+\}$.\looseness-1 
\end{theorem}

\begin{proof}
    See Appendix~\ref{subsec:domset}.
\end{proof}

%% file: arxiv_experiments_no_comment.tex
\section{Evaluation}\label{sec:experiments}

Below and in Appendix~\ref{app:sec_eval}, we present the setup and results of our evaluation of improvement-aware algorithms, focusing on practical risk-aversion strategies, specifically loss-based and threshold-based methods. These strategies align with our theory, which predicts optimal performance for risk-averse classifiers when agents improve under a limited budget $r$. We also investigate whether, and under what conditions, model error can be driven to zero when agents can improve within an $\ell_{\infty}$ ball of radius $r$. Achieving zero error is a notable property in the learning with improvements setting, even for broad concept classes (cf. Section~\ref{sec:graph-model}, where the instance space is discrete and concepts may be arbitrary functions).  Our code is available \href{https://github.com/ripl/PLI/tree/main}{here}.

\paragraph{Datasets.}  We use three real-world tabular datasets: the Adult UCI \cite{adult2}, the OULAD and the Law school datasets \cite{le2022survey}, and a synthetic \(8\)-dimensional binary classification dataset with class separability \(4\) and minimal outliers, generated using Scikit-learn's \texttt{make\_classification} function \cite{make_classification_scikit_learn}. In each case we train a zero-error model \(f^\star\) on the entire dataset, which we treat as the true labeling function for our experiments.

Let \(\mathcal{S}_{T} = \{(x, y) \mid x \in \mathbb{R}^d, \, y \in \{0,1\}\}\)  represent the dataset (e.g., Adult), where \(x\) is the feature vector and \(y = f^\star(x)\) is the label. For all experiments, we split \(\mathcal{S}_{T}\) into training \(\mathcal{S}_\textrm{train}\) (\(70\%\)) and testing \(\mathcal{S}_\textrm{test}\) (\(30\%\)) subsets. Further dataset details, including improvement features and class distributions, are provided in Appendix~\ref{app:sec_datasets}.

\paragraph{Classifiers.} For each full dataset \(\mathcal{S}_{T}\), we trained a zero-error model \(f^\star\) using decision trees. We trained the decision-maker model  \(h: \mathbb{R}^d \to \{0,1\}\), taking the form of a two-layer neural network, on \(\mathcal{S}_\textrm{train}\) with tuned hyperparameters. To assess the loss function's impact on error drop rate when agents improve, we trained both a standard model with binary cross-entropy (\(\mathcal{L}_{\textrm{BCE}}\)) loss and a risk-averse model with weighted-BCE (\(\mathcal{L}_{\textrm{wBCE}}\)) loss (Equation~\ref{eq:losses}).
\begin{equation}\label{eq:losses}
    \begin{aligned}
    \mathcal{L}_{\textrm{wBCE}} &= - \frac{1}{n} \sum_{i=1}^n \Bigg[w_{\textrm{FP}} (1 - y_{i})\log(1 - \hat{y}_{i}) + w_{\textrm{FN}} y_{i}\log(\hat{y}_{i}) \Bigg]
    \end{aligned}
\end{equation}
where, \(n = \lvert\mathcal{S}_{\textrm{train}}\rvert\), \(y \in \{0, 1\}\) is the true label, \(\hat{y} \in (0,1]\) is the model prediction, and \(w_{\textrm{FP}}\) and \(w_{\textrm{FN}}\) are the false positive and false negative weights, respectively. Setting $w_{\textrm{FP}}=w_{\textrm{FN}}=1$ in $\mathcal{L}_{\textrm{wBCE}}$ recovers $\mathcal{L}_{\textrm{BCE}}$. 

Beyond the \(\mathcal{L}_{\textrm{wBCE}}\) loss function, which penalizes false positives more heavily, we explore another form of risk-averse classification by applying a higher threshold of \(0.9\) (instead of the usual \(0.5\)) to the sigmoid output of the final layer. For further details, refer to Appendix~\ref{app:sec_classifiers}.

\paragraph{Improvement.} Given a trained model function \(h\), a data sample \(x \in \mathcal{S}_{\textrm{test}}\)  with an undesirable model outcome \( h(x) = 0 \), and a subset of improvable features along with a predefined improvement budget \( r \), we use Projected Gradient Descent~\cite{AlexAdversarial} to compute the minimal change within the budget \( r \) required to transform \(x\) into a positive outcome \( h(x') = 1 \). Specifically, we aim to find:
\begin{equation} \label{eq:improve_pgd}
x' = \textrm{Proj}_{\Delta(x)}  \left( x_{(t)} + \alpha \cdot  \textrm{sign}(\nabla_{x_{(t)}} \mathcal{L}(h(x_{(t)}), h(x))) \right)
\end{equation}
\(\textrm{such that} \ h(x') = 1\). Here, \( \nabla \mathcal{L}(h(x_{(t)}), h(x)) \) represents the gradient of the loss function (BCE or wBCE), \( t \) the current iteration, and \( \alpha \) the step size. 
\( \textrm{Proj}_{\Delta(x)} \) denotes the projection of \(x_{(t)}\) onto the \(\ell_\infty\) ball of radius \(r\) centered at \(x\), \(\Delta(x) = \{ x_{(t)} \in \mathbb{R}^{d} : \|x_{(t)} - x\|_{\infty} \leq r \} \). This ensures that updates remain within the \( r \)-ball constraint.
A successful improvement occurs when a negatively classified sample \( x \) 
transforms within the specified budget \(r\) into \( x' \) such that \( h(x') = 1 \) and \( f^\star(x') = 1 \) 
(see Appendix~\ref{app:sec_improve}).

\begin{figure*}[!b]
    \centering
    \subfloat[Adult \(\big(w_\textrm{FN}=0.001\big)\) \label{fig:adult_0.5}]{
        \begin{minipage}{0.45\linewidth}
            \centering
            \includegraphics[width=\linewidth]{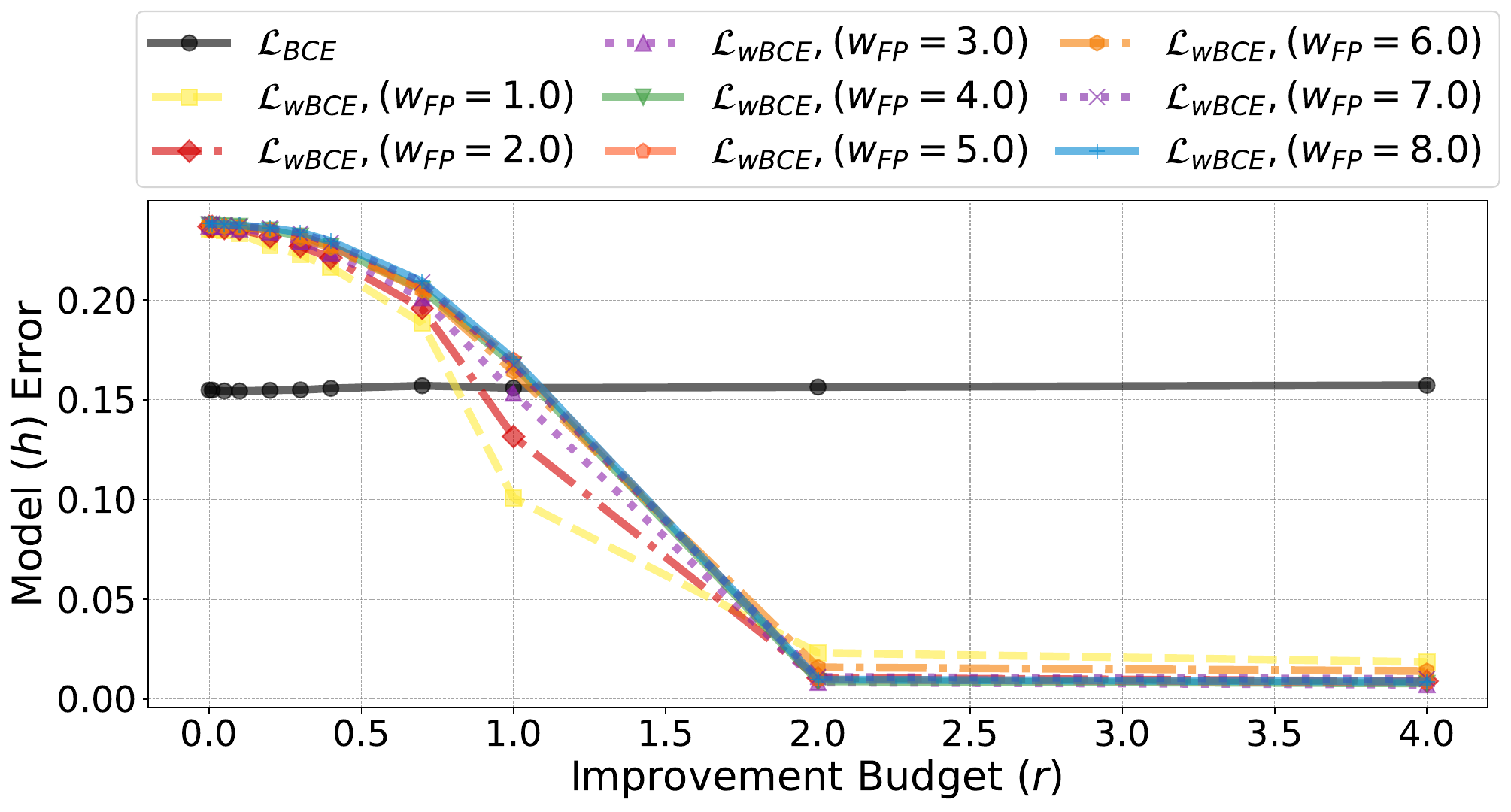}
        \end{minipage}
    } 
    \hfill
    \subfloat[OULAD \(\big(w_\textrm{FN}=1.33\big)\) \label{fig:oulad_0.5}]{
        \begin{minipage}{0.45\linewidth}
            \centering
            \includegraphics[width=\linewidth]{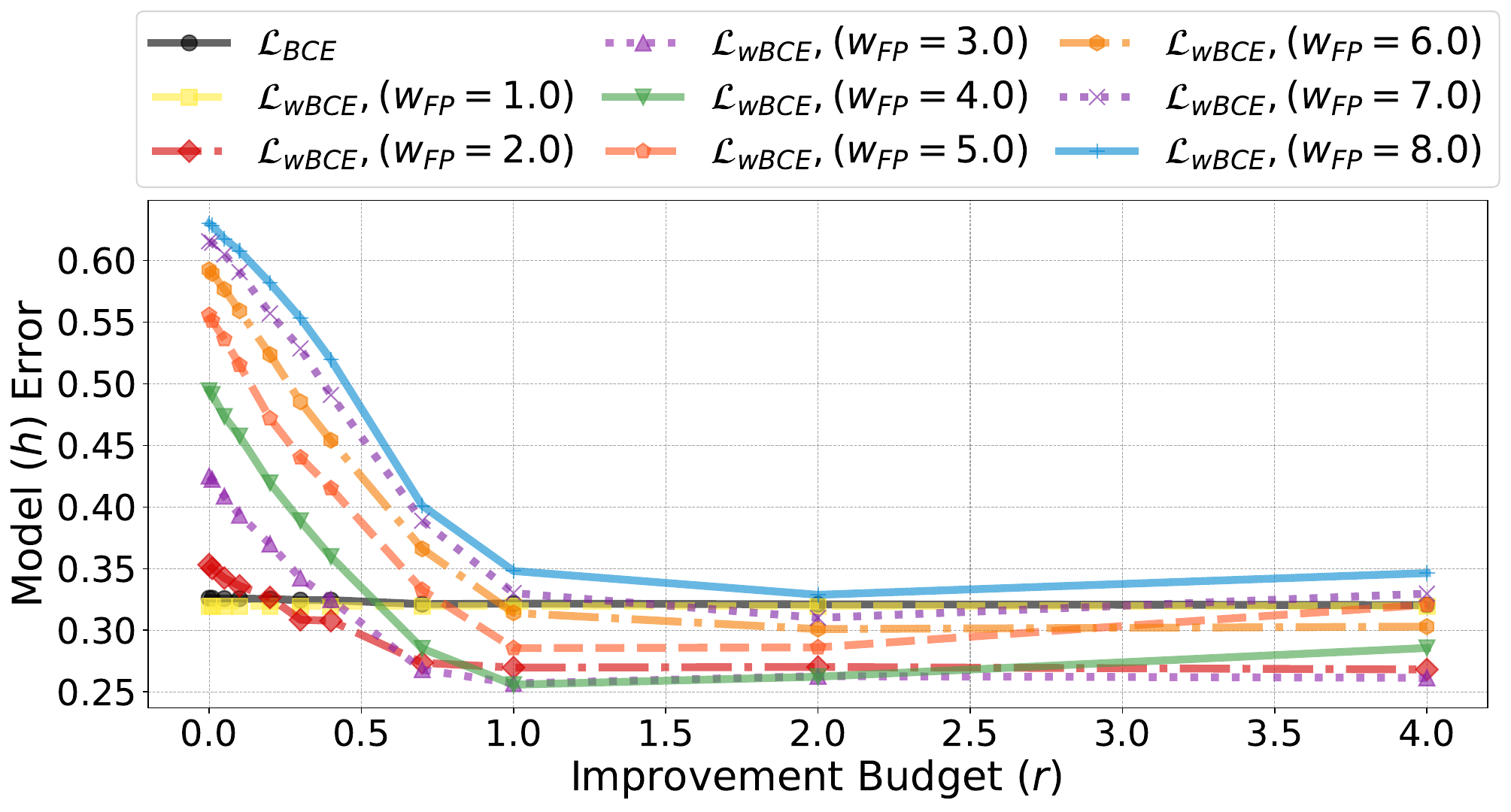}
        \end{minipage}
    }%
    \hfill
    \subfloat[Law school \(\big(\mathcal{L}_\textrm{wBCE} \ \textrm{where} \ w_\textrm{FN}=0.009\big)\) \label{fig:law_0.5}]{
        \begin{minipage}{0.45\linewidth}
            \centering
            \includegraphics[width=\linewidth]{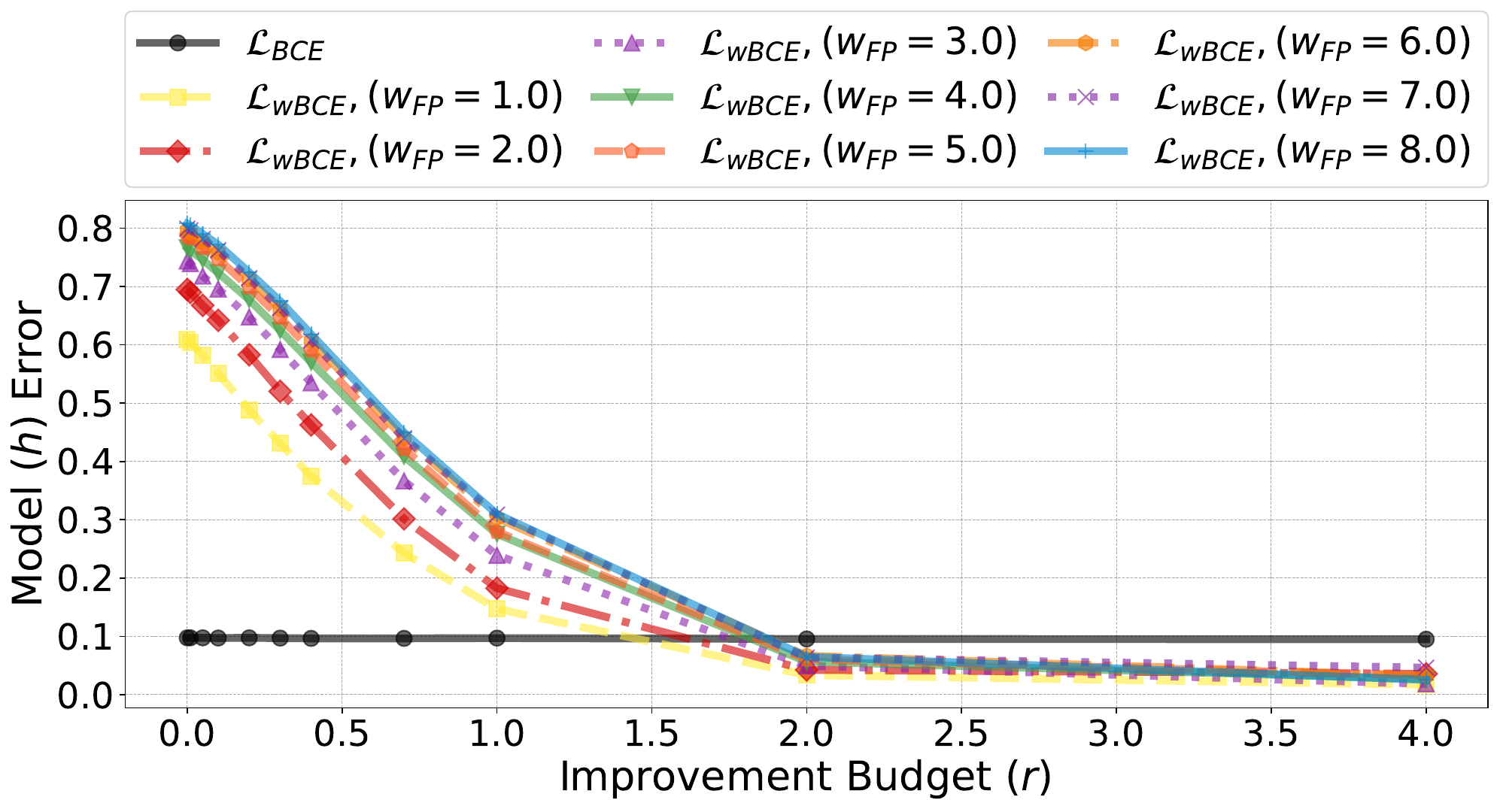}
        \end{minipage}
    }%
    \hfill
    \subfloat[Synthetic \(\big(w_\textrm{FN}=0.009\big)\)\label{fig:synthetic_0.5}]{
        \begin{minipage}{0.45\linewidth}
            \centering
            \includegraphics[width=\linewidth]{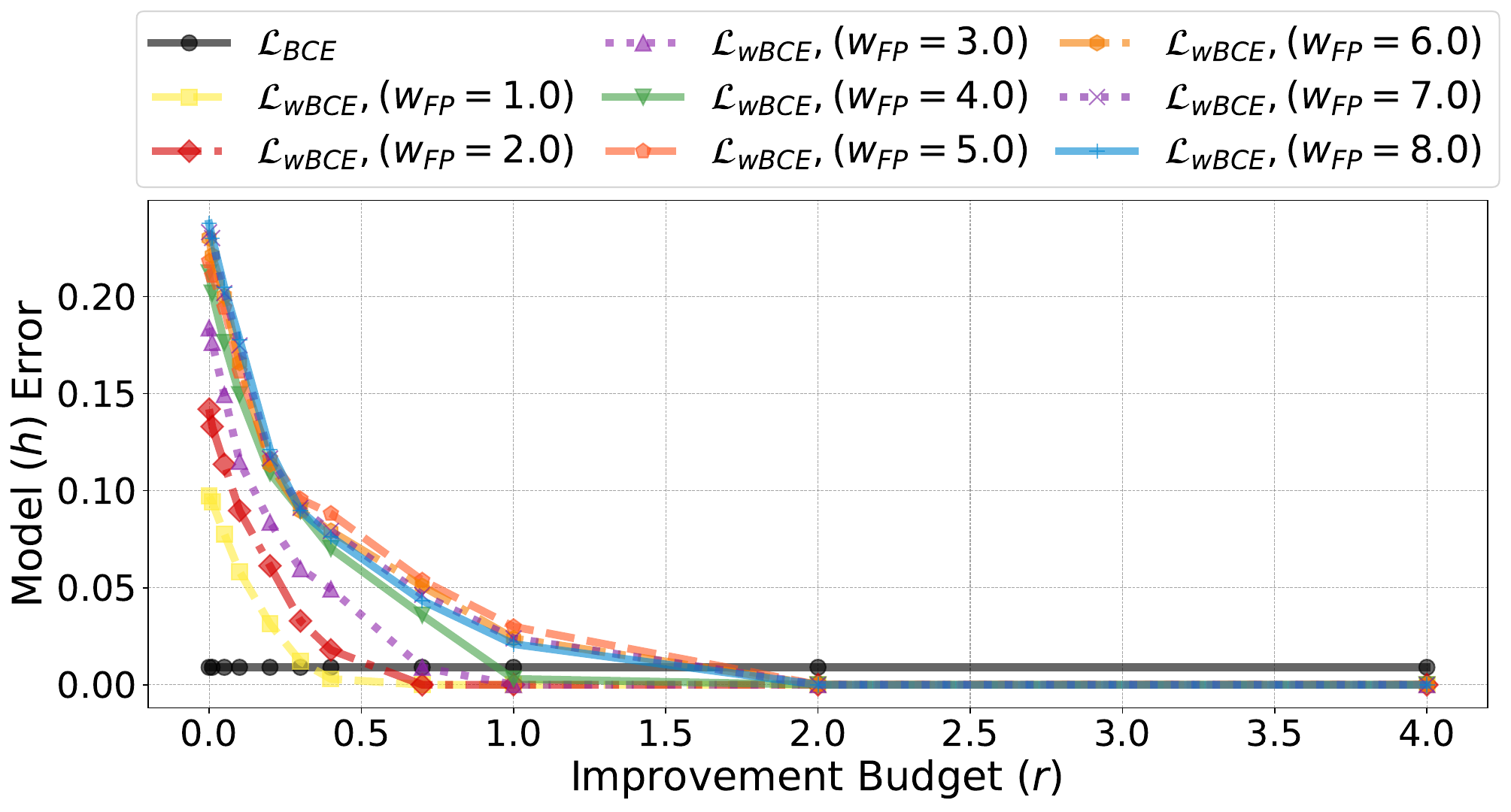}
        \end{minipage}
    }
    \caption{We compare the performance gains when agents improve to the risk-averse \(\big(\mathcal{L}_{\textrm{wBCE}}, \frac{w_\textrm{FP}}{w_\textrm{FN}}>1,  w_{\textrm{FP}}=\{i\}_{i=1}^{8}\big)\) and the standard  (\(\mathcal{L}_\textrm{BCE}, w_{\textrm{FP}}= w_{\textrm{FN}}=1\)) models across four datasets (Adult, OULAD, Law school, and Synthetic) using a fixed classification threshold of \(0.5\). Higher improvement budgets (\(r\)) and greater risk-aversion (high \(\frac{w_\textrm{FP}}{w_\textrm{FN}}\)) accelerate error reduction. See Figure~\ref{fig:app_thresh0.5thresh0.9} (Appendix) for a side-by-side comparison with threshold (\(0.9\)).}
    \label{fig:main_thresh0.5}
\end{figure*}

\subsection{Results}
Here, we highlight the key insights from our evaluations. A more detailed discussion of the results, along with additional empirical evaluation, can be found in Appendix~\ref{app:sec_results}.

Results show that risk-averse (wBCE-trained) models consistently outperform standard (BCE-trained) models in reducing overall error as the improvement budget increases (see Figure~\ref{fig:main_thresh0.5}, and Appendix, Figures~\ref{fig:bce_wbce_threshvar} and \ref{fig:app_thresh0.5thresh0.9}).
While error gains relative to BCE-trained models tend to cancel out, wBCE-trained models retain low false positive rates after agent movement and exhibit a marked decline in false negatives as the improvement budget increases (see Appendix, Figure~\ref{fig:oulad_synthetic_move_erroreval_0.5}). Modest improvement budgets (\(r \leq 2.0\)) lead to substantial error reduction, but returns diminish for (\(r > 2.0\)), especially with a decision threshold of \(0.5\).

Among risk-averse strategies, loss-based risk aversion, where \(\mathcal{L}_{\textrm{wBCE}}\)-trained models use \(\frac{w_\textrm{FP}}{w_\textrm{FN}} > 1\) outperforms threshold-based approaches that classify agents as positive only if predicted probability exceeds \(0.9\) (see Figure~\ref{fig:main_thresh0.5}, and Appendix, Figure~\ref{fig:app_thresh0.5thresh0.9}).

Our results also show that dataset class separability (see Appendix, Figures~\ref{fig:jumbleness} and \ref{fig:orig_knn}) significantly affects the optimal level of risk aversion and the relationship between improvement budget and error reduction .

In summary, risk-averse models initially incur higher errors but achieve rapid error reduction as agents improve and \(r\) increases. 
A stricter false-positive penalty improves the positive agreement region, reducing test error, sometimes to zero (e.g., Appendix, Figure~\ref{fig:app_synthetic_0.5}).

\begin{figure*}[!t]
    \centering
    \subfloat[\(\%\) of agents that transition from TN/FN to TP/FP \label{fig:adult_move_0.5}]{
        \begin{minipage}{0.45\linewidth}
            \centering
            \includegraphics[width=\linewidth]{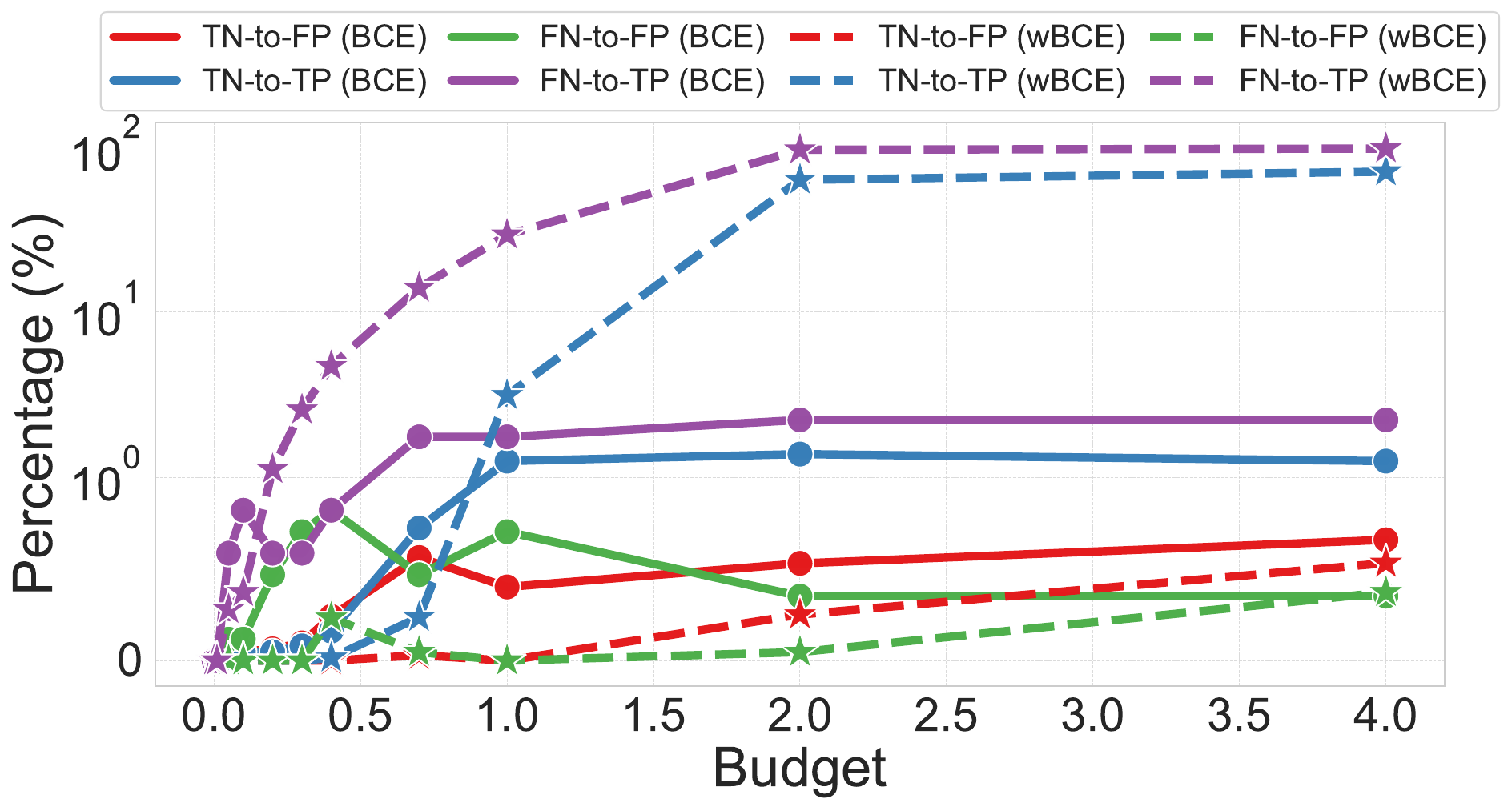}
        \end{minipage}
    } 
    \hfill
    \subfloat[FNR/FPR before and after agents' improvement \label{fig:adult_move_fpr_fnr_0.5}]{
        \begin{minipage}{0.48\linewidth}
            \centering
            \includegraphics[width=\linewidth]{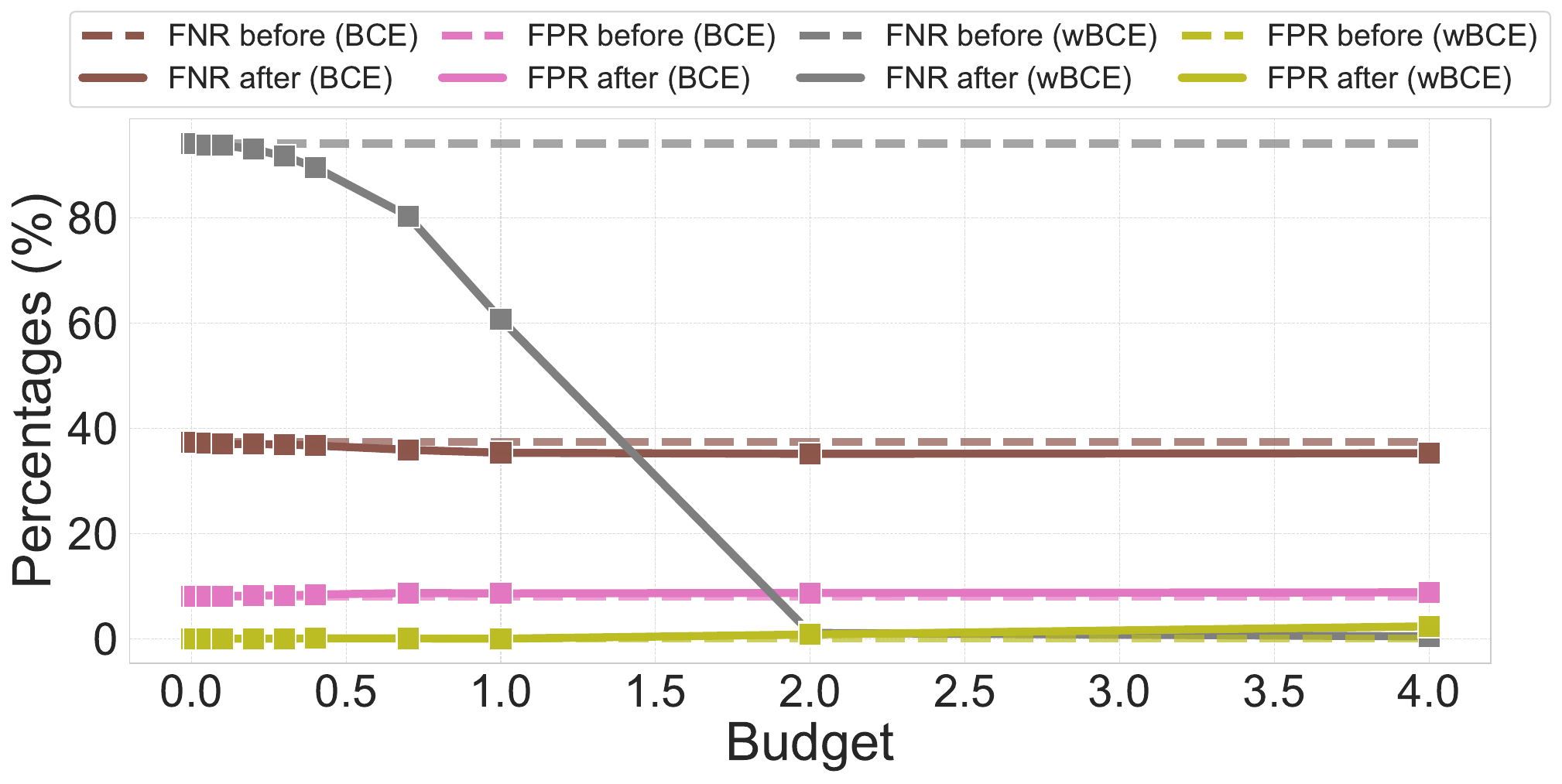}
        \end{minipage}
    }%
    \caption{The percentage of negatively classified agents (true negatives (TN) and false negatives (FN)) that transition to true positives (TP) and false positives (FP) after responding to the classifier (\(h(x)\)) is shown in Figures~\ref{fig:adult_move_0.5}. 
    Figures~\ref{fig:adult_move_fpr_fnr_0.5} shows the FPR and FNR before and after agents move. While the  wBCE-trained model used \(w_{\textrm{FP}}=0.001\) and \(w_{\textrm{FN}} = 4.4\), the BCE-trained model used \(w_{\textrm{FP}} = w_{\textrm{FN}} = 1\), and an agent is classified as positive if the probability of being positive is above \(0.5\).}
    \label{fig:adult_move_erroreval_0.5}
\end{figure*}

%% file: arxiv_appendix_theory_no_comment.tex
\section{Additional Related Work}\label{app:related-work}

\paragraph{Classification of gaming agents.}
Hardt et al.~\cite{hardt2016strategic} formalized the concept of strategic behavior, often referred to as ``gaming,'' where test-set agents who are negatively classified intentionally modify their features---within the bounds of a separable cost function---without altering their target label, to deceive the model into classifying them as positive. They theoretically and empirically showed that their strategy-robust algorithm outperforms the standard SVM algorithm under gaming. However, as the extent of gaming increases, overall model accuracy declines. 
Dong et al.~\cite{reveleadpref} also study a Stackelberg equilibrium where agents strategically respond to classification learners. However, unlike Hardt et al.~\cite{hardt2016strategic}, their model assumes that the learner lacks direct knowledge of the agents' utility functions and instead infers them through observed revealed preferences. Additionally, agents arrive sequentially, and only the true negatives strategically respond to the learner. The learner's objective is to minimize the Stackelberg regret. 
Chen et al.~\cite{Chen2019LearningSL} also study a  learner whose goal is to minimize the Stackelberg regret, where gaming agents arrive sequentially. However, unlike Dong et al.~\cite{reveleadpref} which assumes a convex loss function, they deal with a less smooth agent utility function and learner loss function. They propose the Grinder algorithm, which adaptively partitions the learner's action space based on the agents' responses. 
Performative prediction~\cite{perdomo2021performativeprediction} considers a setting that involves a repeated interaction between the classifier and the agents, and as a result the underlying distribution of the gaming agents may change over time. 

\paragraph{Classification of  agents that can both game and improve.} 
Unlike earlier works in the strategic classification literature, which primarily focus on settings where agents engage in gaming behavior, Kleinberg and Raghavan~\cite{jkleinhowdo} examine a scenario where agents can genuinely improve. In this context, the agent can modify their observable features and true label to achieve a positive model outcome. The authors demonstrate that a learner employing a linear mechanism can encourage rational agents, who optimize their allocation of effort, to prioritize actions that result in meaningful improvement. They show how to achieve this by selecting an evaluation rule that incentivizes a desirable effort profile. 

Ahmadi et al.~\cite{ahmadi2022classificationstrategicagentsgame}, like Kleinberg and Raghavan~\cite{jkleinhowdo}, consider the agents' potentially truthful and actionable responses to the model. However, the primary objective of Ahmadi et al.~\cite{ahmadi2022classificationstrategicagentsgame} is to maximize true positive classifications while minimizing false positives. Notably, for the linear case, they show that the resulting classifier can become non-convex, depending on the agents' initial positions. 

On the other hand, Ahmadi et al.~\cite{ahmadi2022settingfairincentivesmaximize} design reachable sets of target levels such that they can incentivize effort-bounded agents within each group to improve optimally.

\paragraph{Theoretical guarantees of incentive-aware or incentive-compatible classifiers. } 
Zhang and Conitzer~\cite{Zhang_Conitzer_2021} show that the vanilla ERM principle fails under strategic manipulation (gaming), even in simple scenarios that would otherwise be straightforward without gaming. To address this, they propose the concepts of incentive-aware and incentive-compatible ERMs, theoretically analyzing the corresponding classifiers, their sample complexity, and the impact of the VC dimension on the associated hypothesis class. Finally, they extend their analysis to ERM-based learning in environments with transitive strategic manipulation. 

Given adversarial data points wanting to receive an incorrect label, Cullina et al.~\cite{Cullina2018PAClearningIT} theoretically show that the sample complexity of PAC-learning a set of halfspace classifiers does not increase in the presence of adversaries bounded by convex constraint sets and that the adversarial VC dimension can be arbitrarily larger or smaller than the standard VC dimension. 
Sundaram et al.~\cite{strategicPAC} provide theoretical guarantees for an offline, full-information strategic classification framework where data points have distinct preferences over classification outcomes (\(+\) or \(-\)) and incur varying manipulation costs, modeled using seminorm-induced cost functions. They propose a PAC-learning framework for strategic linear classifiers in this setting, providing a detailed analysis of their statistical and computational learnability. Additionally, they extend the concept of the adversarial VC dimension Cullina et al.~\cite{Cullina2018PAClearningIT} to this strategic context. They also show, among other things, that employing randomized linear classifiers can substantially improve accuracy compared to deterministic methods.

\paragraph{Reliable machine learning.} 
The concept of risk aversion in our work is closely related to selective classification or machine learning with a reject option~\cite{yaniv10a,NIPS2017_4a8423d5,Hendrickx2021MachineLW}, where the classifier balances the trade-off between risk and coverage, opting to abstain from making predictions when it is likely to make mistakes. Similarly, risk-averse classification aligns with aspects of learning with one-sided error~\cite{natarajan1987learning,onesidedagree}, particularly positive reliable learners~\cite{kalaiReliable}, which aim to achieve zero false positive errors while minimizing false negatives. Prior work has shown connections between strategic classification and adversarial learning (e.g.\ \cite{strategicPAC}), but it remains an interesting open question if similar connections can be established between learning with improvements and reliable learning in the presence of adversarial attacks~\cite{balcan2022robustly,balcan2023reliable,blum2024regularized}.

\paragraph{Adversarial robustness.} Our results indicate that the properties of the concept class that govern learnability with improvements are different from those established for adversarial robustness. In particular, finite VC dimension is not sufficient for PAC learnability with improvements (Example \ref{ex:example_notsufficient}) as opposed to adversarial robustness, where it is sufficient for (improper) learning~\citep{montasser2019vc}. Furthermore, our risk-averse learning algorithm is very different from techniques proposed in the adversarial literature, including robust Empirical Risk Minimization~\citep{Cullina2018PAClearningIT,attias2022improved}, boosting algorithms whose generalization is based on sample compression schemes \citep{montasser2019vc,montasser2020reducing,attias2022characterization,attias2023adversarially} and approaches developed for neural networks~\citep{AlexAdversarial,Cohen2019CertifiedAR,balcan2023analysis}. Similarly, our algorithm for learning halfspaces on a unit ball differs from classical approaches for learning from malicious noise~\citep{Awasthi2013ThePO,Balcan2020NoiseIC}. However, due to analogies between the frameworks, it is meaningful to ask future research questions along the lines of the directions pursued in the adversarial robustness literature---for example, extensions to unknown or uncertain improvement sets \citep{montasser2021adversarially,lechner2022learning}, and learning with tolerance \citep{ashtiani2023adversarially,raman2023proper,ashtiani2025simplifying}.

\section{Separating PAC Learning with Improvements from the Standard PAC Models}

Here, we further prove that learning with improvements diverges from the behavior of the standard PAC model for binary classification.

\subsection{Separation from Standard PAC Learning Model when $\tilde{\mathcal{H}}\subset{\mathcal{H}}$}\label{app:example-error-gap}

\begin{example}[Error gap when the learner's hypothesis space $\tilde{\mathcal{H}}$ is a strict subset of the concept space ${\mathcal{H}}$ that contains $f^*$]\label{ex:standard_improve}

Let $\mathcal{X}=[-1,1]$ and $\tilde{\mathcal{H}}$ denote the set of concepts including unions of up to $k$ open intervals. The set of possible improvements for any point $x\in \mathcal{X}$ is given by $\mathcal{X} \cap \mathbb{Q}$, where $\mathbb{Q}$ denotes the set of rational numbers. Suppose the data distribution is uniform over $\mathcal{X}$. We set the target concept $f^*$ as follows

\[f^*(x)=
\begin{cases} 
        0, & \text{if } x<0, \text{or } x\in \mathbb{Q}, \\ 
        1, & \text{otherwise},
    \end{cases}
\]

Note that rationals are dense in $[0,1]$ and the set of all rationals have a Lebesgue measure zero.
Thus, on any finite sample $S\in \mathcal{X}^m$, any sampled point $x$ will have a positive label according to $f^*$ iff $x\ge 0$ (with probability 1).
In the standard PAC learning setting, the classifier $\Tilde{f}=\mathbb{I}\{x\in(0,1)\}$  achieves zero error w.r.t.\ the target $f^*$. This is because the misclassification error for points in $\mathbb{Q}$ is  zero.

In our setting where agents have the ability to improve, for an $h\in \tilde{\mathcal{H}}$ which predicts any point $x'$ in $\mathbb{Q}$ as positive,  all negative agents in $[-1,0)$ can move to such a point $x'$ and be falsely classified as positive. This corresponds to a  lower bound of $\frac{1}{2}$ on the error.
Since rationals are dense in the reals, any open interval which $h$ classifies as positive must contain a point in $\mathbb{Q}$.
On the other hand, if $h$ classifies no point as positive, then error rate is again $\frac{1}{2}$ as all the positive points are misclassified.

\end{example}

\section{Missing Proofs from Section \ref{sec:geometric-concepts}\label{app:proofs-geometric}}

We include below missing proofs from Section \ref{sec:geometric-concepts}.

\subsection{Proof of Theorem \ref{thm:thresholds-uniform}: Learning Thresholds with the Uniform Distribution}\label{subsec:threshold-uniform-D}

\begin{proof}
    Let $S\overset{\text{i.i.d.}}{\sim} \mathcal{D}^m$, where $\mathcal{D}$ is the uniform distribution over $[0,1]$. By using a standard calculation of the sample complexity of thresholds, 
    \begin{align}\label{eq:threshold-uniform}
    \begin{split}
    \underset{S\sim \mathcal{D}^m}{\mathbb{P}}\left[\underset{x\sim \mathcal{D}}{\mathbb{P}}\left[h_{t^*}(x)\neq h_{S^+}(x)\right]>\epsilon\right]
    &\leq \prod^m_{i=1}\mathbb{P}\left[x_i \notin [t^*,t^*+\epsilon]\right]     
    \\
    &\leq 
    (1-\epsilon)^m
    \\
    &\leq
    e^{-\epsilon m}
    \\
    &\leq \delta, 
    \end{split}
    \end{align}

    \noindent where the last inequality holds for $m\geq \frac{1}{\epsilon}\log\frac{1}{\delta}$. Since whenever $h_{S^+}$ classifies a point as positive, $h_{t^*}$ also classifies it as positive, 
    any negative point that improves in response to $h_{S^+}$ must move to a true positive point, and the error can only decrease in the improvements setting for the choice of $h_{S^+}$.

    Now, since we allow improvements of distance $r$, the points in the interval $[t_{S^+}-r, t_{S^+}]$ that would have been classified negatively without improvement are able to improve under $h_{S^+}$ (and indeed improve to be positive with respect to $h_{t^*}$) and are thus classified correctly. The points on which $h_{S^+}$ makes mistakes are those in the interval $[t^*, t_{S^+}-r]$. Since $\mathcal{D}$ is uniform, our previous inequality implies that with probability at least $1-\delta$ we have $t_{S^+}\leq t^*+\epsilon$. This implies the that error is at most $\max(\epsilon - r,0)$ with probability $1 - \delta$ as desired.  

\end{proof}

\subsection{Learning Thresholds with An Arbitrary Distribution}\label{subsec:threshold-arbitrary-D}

\begin{theorem}[Thresholds, arbitrary distribution]\label{thm:threshold-arbitrary-D}
    Let the improvement set $\Delta$ be the closed ball with radius $r$, $\Delta(x) =\{ x'\mid |x - x'| \leq r\} $.
    For any distribution $\mathcal{D}$, and any \(\epsilon,\delta \in (0,1/2)\), with probability $1-\delta$,
    \begin{align*}
       \textsc{Loss}_{\mathcal{D}}(h_{S^+},h_{t^*})
       \leq (\epsilon -p(h_{S^+};h_{t^*},\mathcal{D},r))_+,
    \end{align*}
    where 
    \begin{align*}
        p(h_{S^+};h_{t^*},\mathcal{D},r)= \mathbb{P}_{x\sim \mathcal{D}}\left[x\in[t_{S^+}-r,t_{S^+}]\right],
    \end{align*}
    with sample complexity $M = O\left(\frac{1}{\epsilon} \log \frac{1}{\delta}\right).$

    \label{thm:thresholds-general}
\end{theorem}

\begin{proof}

    Let $t_0$ be such that $\mathbb{P}_{x\sim\mathcal{D}}\left[x\in [t^*,t_0]\right]=\epsilon$. By following the same derivation as in Eqn.~\ref{eq:threshold-uniform} and replacing $t^*+\epsilon$ with $t_0$, we get that 
    \begin{align*}
        \underset{S\sim \mathcal{D}^m}{\mathbb{P}}\left[\underset{x\sim \mathcal{D}}{\mathbb{P}}\left[h_{t^*}(x)\neq h_{S^+}(x)\right]\leq\epsilon\right],
    \end{align*}
    with probability $1-\delta$ for $m\geq \frac{1}{\epsilon}\log\frac{1}{\delta}$.
    
    The points in the interval $[t_{S^+}-r, t_{S^+}]$ are able to improve under $h_{S^+}$ and thus classified correctly. The gain to the error of $h_{S^+}$ would be the probability mass of points in $\mathcal{D}$ that fall into this interval, defined as 
    \begin{align*}
        p(h_{S^+};h_{t^*},\mathcal{D},r)= \mathbb{P}_{x\sim \mathcal{D}}\left[x\in[t_{S^+}-r,t_{S^+}]\right].
    \end{align*}
    The points on which $h_{S^+}$ makes mistakes are those in the interval $[t^*, t_{S^+}-r]$. 
    
    We conclude that with probability at least $1-\delta$ we have $t_{S^+}\leq t_0$, and given the improvement of points in $[t^*, t_{S^+}-r]$ we have an error at most $\max(\epsilon -  p(h_{S^+};h_{t^*},\mathcal{D},r),0)$ with probability $1 - \delta$ as desired. 

\end{proof}

\subsection{Proof of Theorem \ref{thm:rectangles}}
\label{app:rectangles}

\begin{proof}

    Let $R_S^c=\clos_{\mathcal{H}_{\text{rec}}}\left(\{x_i \in S : y_i = 1\}\right)$ be the output of the closure algorithm.
    For the standard PAC setting, we have 
    \begin{align*}
        \underset{S\sim \mathcal{D}^m}{\mathbb{P}}\left[\underset{x\sim \mathcal{D}}{\mathbb{P}}\left[R_S^c(x)\neq R^*(x)\right]\leq\epsilon\right],
    \end{align*}
    with probability $1-\delta$ for $m\geq \Omega\left(\frac{1}{\epsilon}\left(d+\log\frac{1}{\delta}\right)\right)$, see  Darnst{\"a}dt.,~\citep{darnstadt2015optimal}. Since whenever $R_S^c$ classifies a point as positive, $R^*$ also classifies it as positive,
    any negative point that improves in response to $R_S^c$ must move to a true positive point and the error can only decrease in the improvements setting for the choice of $R_S^c$.
    
    Now, in order to quantify the gain in error from the improvements, we define the ``outer boundary strip" of $R_S^c$.
    Let the rectangle defined by
    $R_S^c=\prod_{i\in [d]}[a_i,b_i]$.
    The points that are able to improve under $R_S^c$ are exactly fall into the outer boundary strip of size $r$, defined as
    
    \begin{align*}
        \bs(R_S^c,r)
        =
        \prod_{i\in [d]}[a_i+r,b_i+r]
        \setminus
        \prod_{i\in [d]}[a_i,b_i].
    \end{align*}
    
    Note that this is exactly the improvement region of $R_S^c$: $\bs(R_S^c,r)=\ir(R_S^c;R^*,\Delta)$.
    Under general distribution $\mathcal{D}$, the probability mass of the improvement region is
    \begin{align*}
        \mathbb{P}_{x \sim \mathcal{D}}\left[x \in \ir(R_S^c;R^*,\Delta) \right]
        =
        \mathbb{P}_{x\sim\mathcal{D}} \left[x\in \bs(R_S^c,r) \right],
    \end{align*}
    since $R_S^c$ is the smallest rectangle that fits $S$, these points that are able to improve under $R_S^c$  indeed improve to be positive with respect to $R^*$.
    This implies that 
    \begin{align*}
       \textsc{Loss}_{\mathcal{D}}(R_S^c,R^*)
       \leq 
       \max\left(\epsilon - \mathbb{P}_{x \sim \mathcal{D}}\left[x \in \ir(R_S^c;R^*,\Delta) \right]\right),0).
    \end{align*}
    
    Now, for the uniform distribution, we can compute an exact expression of the improvement region. Let $l_i=b_i-a_i$, then 
    \begin{align*}
        \mathbb{P}_{x \sim \mathcal{D}}\left[x \in \ir(R_S^c;R^*,\Delta) \right]
        &=
        \prod_{i\in [d]}(l_i+2r)-\prod_{i\in [d]}l_i.
    \end{align*}
    For $d=2$, we get
    \begin{align*}
      \mathbb{P}_{x \sim \mathcal{D}}\left[x \in \ir(R_S^c;R^*,\Delta) \right]  
      &=
      (l_1+2r)(l_2+2r)-l_1l_2
      \\
      &=
      2r(l_1+l_2)+4r^2.
    \end{align*}
    
\end{proof}

\subsection{Proof of Theorem \ref{thm:hardness-intersection-closed}}
\label{app:hardness-intersection-closed}

\begin{proof}
    For any concept $h\in\mathcal{H}$, let $\mathcal{X}^+_h$ denote the set of points $\{x\in\mathcal{X}\mid h(x)=1\}$ positively classified by $h$. Since $\mathcal{H}':=\mathcal{H}\mid_{\mathcal{X}\setminus\{x'\}}$ is not intersection-closed, there must exist a set $S\subseteq \mathcal{X}\setminus \{x'\}$ such that $\clos_{\mathcal{H}'}(S)\notin \mathcal{H}'$. For the uniformly negative point $x'$, we have its improvement set as $\Delta(x')=\mathcal{X}\setminus S$. For points in $\clos(S)$ we have the improvement set as the empty set. We set the data distribution $\mathcal{D}$ as the uniform distribution over $\clos(S)\cup\{x'\}$. Let $h_1\in\mathcal{H}$ be a minimally consistent classifier w.r.t.\ $S$, i.e.\ if $h'\in\mathcal{H}$ and $\mathcal{X}^+_{h'}\subseteq \mathcal{X}^+_{h_1}$, then $h'=h_1$. By choice of $S$, there is a point $x_1\in \mathcal{X}^+_{h_1}\setminus \clos_{\mathcal{H}'}(S)$. By the definition of closure of $S$, there must exist $h_2\in\mathcal{H}$ consistent with $S$ (assumed minimally consistent WLOG) such that $h_2(x_1)=0$. Also, since $h_1$ was chosen to be minimally consistent, there must exist $x_2\in \mathcal{X}^+_{h_2}$ such that $h_1(x_2)=0$. We will set the target concept $f^*$ to one of $h_1$ or $h_2$.
    
    Now any learner that picks a concept not consistent with $S$ will clearly suffer a constant error on the points in $\clos(S)$ which are incorrectly classified as negative and not allowed to improve. Suppose therefore that the learner selects a hypothesis $h$ consistent with $S$. Let $\tilde{h}$ denote a classifier which is minimally consistent with $S$ and $\mathcal{X}^+_{\tilde{h}}\subseteq \mathcal{X}^+_{h}$ ($\tilde{h}$ could possibly be the same as $h$). If $\tilde{h}=h_1$ (resp.\ $\tilde{h}=h_2$), the  learner suffers a constant error as $x'$ can improve to the false positive $x_1$ (resp.\ $x_2$) when the target concept is $h_2$ (resp.\ $h_1$). Else, there must exist $\tilde{x}\in \mathcal{X}^+_{\tilde{h}}$ such that $h_1(\tilde{x})=0$, since $h_1$ was chosen to be minimally consistent (and likewise for $h_2$). $h(\tilde{x})=1$ in this case, and $x'$ can now falsely ``improve'' to $\tilde{x}$. Since the learner has no way of knowing from the sample whether the target is $h_1$ or $h_2$, it must suffer a constant error for any $h$ it selects from $\mathcal{H}$. 

\end{proof}

\section{Missing Proofs and Additional Definitions from Section \ref{sec:graph-model} \label{app:graph}}

We include below additional definitions and complete proofs for results in Section \ref{sec:graph-model}.

\subsection{Additional Definitions}

\begin{definition}
    The shortest path metric \( d_G : V \times V \to [0, |V|+1) \) is defined as follows: 
    
    \[
    d_G(x, x') =
    \begin{cases}
    \min \left\{ k \ \middle| \ \exists \ (x_0=x, x_1, \dots, x_k=x') \subseteq V, \ 
    (x_{i-1}, x_i) \in E \ \forall i \in [k] \right\}, & \text{if a path exists,} \\
    |V|+1, & \text{if no path exists.}
    \end{cases}
    \]
    
    \noindent Here, \( k \) is the length of the shortest path between \(x\) and \(x'\) in terms of the number of edges. If there is no path connecting \(x\) and \(x'\), the distance is defined as \( |V|+1 \).
    \noindent The shortest path metric satisfies the following properties:
    \begin{itemize}
        \item \textbf{Non-negativity:} \( d_G(x, x') \geq 0 \) for all \( x, x' \in V \), with \( d_G(x, x) = 0 \).
        \item \textbf{Symmetry:} \( d_G(x, x') = d_G(x', x) \) for all \( x, x' \in V \), since \( G \) is undirected.
        \item \textbf{Triangle inequality:} \( d_G(x, x') \leq d_G(x, z) + d_G(z, x') \) for all \( x, x', z \in V \).
    \end{itemize}
    
     \noindent Our results in Section \ref{sec:graph-model} extend to the more general improvement function $\Delta(x) = \{x' \in V \ | \ d_G(x, x') \le r\}$ by applying our arguments to $G^r$, the $r$th power of $G$.
    
    \label{shortestpathdef}
\end{definition}

\begin{definition}[Dominating Set]
    Let \( G = (V, E) \) be an undirected graph, where \( V \) is the set of vertices and \( E \subseteq V \times V \) is the set of edges. A subset of vertices \( S \subseteq V \) is called a dominating set if every vertex in \( V \) is either in \( S \) or adjacent to at least one vertex in \( S \). Formally, \( S \) is a dominating set if:
    \[
    \forall x \in V, \quad x \in S \ \text{or} \ \exists x' \in S \ \text{such that} \ (x, x') \in E.
    \]
\label{defDS}
\end{definition}

\subsection{Enabling Improvement Whenever It Helps}\label{app:improvwhenhelps}

\begin{theorem}
    Let \( G = (V, E) \) be an undirected graph with \( n = |V| \) vertices, and let \( f^*: V \to \{0, +1\} \) denote the ground truth labeling function. Define: \( V^+ = \{x \in V \mid f^*(x) = +1\} \), the set of positive vertices. Define \( V^- = \{x \in V \mid f^*(x) = 0\} \) the set of negative vertices, and \( N = \{x \in V^- \mid \exists x' \in V^+ \text{ such that } (x, x') \in E\} \) denote the set of negative vertices that have a positive neighbor. Let \( d_{\min}^N = \min_{x \in N} |\{x' \in V^+ \mid (x, x') \in E\}| \) denote the minimum number of positive neighbors of vertices in $N$.  Assume that the data distribution \(\mathcal{D}\) is uniform on \( V \). For any $\delta>0$, 
    and training sample $S \overset{\text{i.i.d.}}{\sim} \mathcal{D}^m$ of size  $m =  O \left(\frac{n (\log n + \log \frac{1}{\delta})}{d_{\min}^N}\right)$, there exists a learner that  outputs a hypothesis $h$ such that $\textsc{Loss}^{\textsc{e}}_{\mathcal{D}}(h,f^*)=0$, with probability at least $1-\delta$ over the draw of $S$. Moreover, there exists a graph $G$ for which any learner that always outputs $h$ with $\textsc{Loss}^{\textsc{e}}_{\mathcal{D}}(h,f^*)=0$ for any $\mathcal{D},f^*$ must see at least $\Omega\left(\frac{n}{d_{\min}^N} \log \frac{n}{d_{\min}^N}\right)$ labeled points in the training sample, with high constant probability.
\label{thm:enable-improvement}
\end{theorem}

\begin{proof}
    The proof of the upper bound is technically similar to the proof of Theorem \ref{DSSamplecomplexity}. Essentially, to ensure that $\textsc{Loss}^{\textsc{e}}=0$, we need to {\it cover} all the vertices in $N$ by some vertices in $V^+$. The probability that any fixed node in $N$ is covered by a random sample can be lower-bounded in terms of $d_{\min}^N$ as
    \[
p_{\text{cover}}(x) \geq \frac{d_{\min}^N }{n}.
\]
 Using the same argument as in Theorem \ref{DSSamplecomplexity}, we obtain an upper bound of $m =  O \left(\frac{n (\log n + \log \frac{1}{\delta})}{d_{\min}^N}\right)$ on the sample complexity of the classifier $h$ which outputs exactly the positively labeled points in its sample as positive to guarantee that $\textsc{Loss}^{\textsc{e}}_{\mathcal{D}}(h, f^*) = 0$ with probability at least $1-\delta$.

 To establish the lower bound, consider a graph $G=(V,E)$ with two types of nodes, i.e.\ $V=V_1\cup V_2$, $|V_1|=k$, $|V_2|=n-k$. $f^*$ labels all nodes in $V_1$ as negative. Each node $x_i\in V_1$ has  $d_{\min}^N=\frac{n-k}{k}$ neighboring nodes $\Delta_i$ in $|V_2|$ and the sets of these neighbors are pairwise disjoint. Now, suppose our training sample $S$ does not contain any point in $\Delta_i$ for some $i\in[k]$. If the learned hypothesis $h$ predicts any point  $\Delta_i$ as positive, we have $h(x_i)\ne\{x_i\}$ but if $f^*$ labels all points in $\Delta_i$ negative, $f^*(x_i)=\{x_i\}$ and we incur loss corresponding to $x_i$. Similarly, if $h$ labels all points in $\Delta_i$ as negative then $h(x_i)=\{x_i\}$ but we can label $\Delta_i$ consistent with $S$ such that $f^*(x_i)\ne\{x_i\}$.

\begin{figure}[ht]
\vskip 0.2in
\begin{center}
\centerline{\includegraphics[width=0.45\columnwidth]{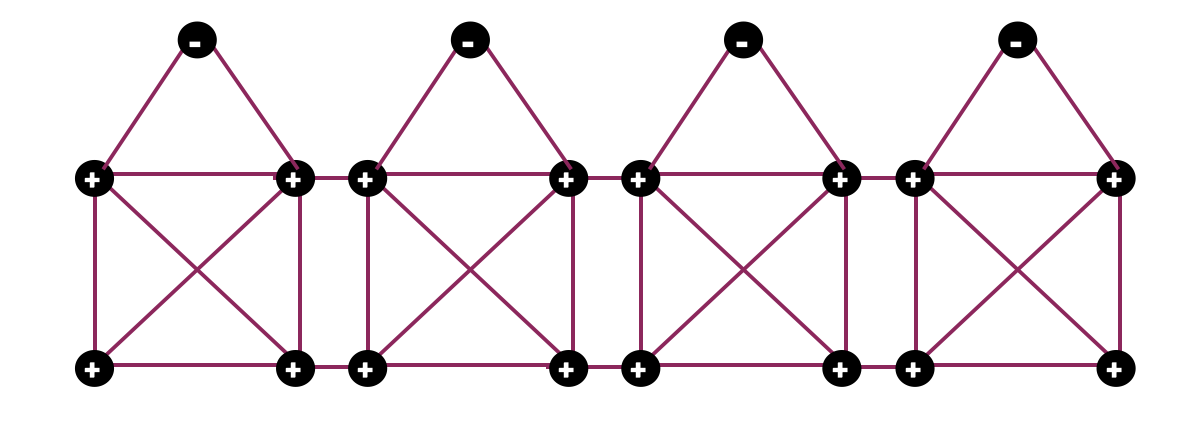}}
\caption{The graph $G$ used to establish our lower bound in Theorem \ref{thm:enable-improvement}.}
\label{posnegcl}
\end{center}
\vskip -0.2in
\end{figure}

 Therefore it is sufficient  to determine a lower bound on the number of points required to ensure that every $\Delta_i$ has at least one of its vertices included in the training sample $S$. Using the standard coupon collector analysis, the number of trials needed to collect $k = \frac{n}{d_{\min}^N+1}$ coupons is $\Omega(k\log k)$ with high constant probability. 
\end{proof}

\begin{theorem}
    Let \( G = (V, E) \) be an undirected graph with \( n = |V| \) vertices, and let \( f^*: V \to \{0, +1\} \) denote the ground truth labeling function. Define \( V^+ = \{x \in V \mid f^*(x) = +1\} \), the set of positive vertices and $d_{\min}^+$ denote the minimum degree of a vertex in $V^+$ in the subgraph of $G$ induced by $V^+$. Define \( V^- = \{x \in V \mid f^*(x) = 0\} \) the set of negative vertices, and \( N = \{x \in V^- \mid \exists x' \in V^+ \text{ such that } (x, x') \in E\} \) denote the set of negative vertices that have a positive neighbor. Let \( d_{\min}^N = \min_{x \in N} |\{x' \in V^+ \mid (x, x') \in E\}| \) denote the minimum number of positive neighbors of vertices in $N$.  Assume that the data distribution \(\mathcal{D}\) is uniform on \( V \). For any $\delta>0$,  
    and training sample $S \overset{\text{i.i.d.}}{\sim} \mathcal{D}^m$ of size  $m =  O \left(\frac{n (\log n + \log \frac{1}{\delta})}{\min\{d_{\min}^N,d_{\min}^+\}}\right)$, there exists a learner that  outputs a hypothesis $h$ such that $\textsc{Loss}_{\mathcal{D}}(h,f^*)=0$ and  $\textsc{Loss}^{\textsc{e}}_{\mathcal{D}}(h,f^*)=0$, with probability at least $1-\delta$ over the draw of $S$. Moreover, there exists a graph $G$ for which any learner that always outputs $h$ with $\textsc{Loss}_{\mathcal{D}}(h,f^*)=\textsc{Loss}^{\textsc{e}}_{\mathcal{D}}(h,f^*)=0$ for any $\mathcal{D},f^*$ must see at least $\Omega\left(\max\{\frac{n}{d_{\min}^+} \log \frac{n}{d_{\min}^+}, \frac{n}{d_{\min}^N} \log{\frac{n}{d_{\min}^N}}\}\right)$ labeled points in the training sample, with high constant probability.
\label{thm:enable-improvement-and-zero-loss}
\end{theorem}
\begin{proof}
    See Theorems \ref{DSSamplecomplexity} and \ref{thm:enable-improvement}.
\end{proof}

\subsection{Proof of Theorem \ref{thm:dominating-set-teaching}: Teaching a Risk-averse Student}\label{subsec:domset}

\begin{proof}
    Since \( S^+ \) is a dominating set of \( G^+ =(V^+,E^+)\), for any \( x \in V^+ \), either \( x \in S^+ \) or there exists \( x' \in S^+ \) such that \( (x, x') \in E^+ \). In the first case, $h_{S^+}(x)=1$ and therefore $\Delta_{h_{S^+}}(x)=\{x\}$. Thus,  $h_{S^+}(x)=1=f^*(x)$ implies that $\textsc{Loss}(x;h_{S^+},f^*)=0$ in this case.
    In the second case, \( x \notin S^+ \), but there exists a neighbor $\tilde{x}\in \Delta(x)$ such that $\tilde{x}\in S^+$ by the definition of $S^+$. Thus, for any point  $x'\in\Delta_{h_{S^+}}(x)\subseteq S^+$, we have that $h_{S^+}(x')=1=f^*(x')$, ensuring that $\textsc{Loss}(x;h_{S^+},f^*)=0$ in this case as well.

    For \( x \in V \setminus V^+ \), if \( x \) has no positive neighbors in $S^+$, \( \Delta_{h_{S^+}}(x) = \{x\} \) because there is no neighboring vertex \( x' \in \Delta(x) \) that would induce a reaction. Thus, $h_{S^+}(x)=0=f^*(x)$ in this case, implying the loss $\textsc{Loss}(x;h_{S^+},f^*)$ on $x$ is zero.  Finally, if \( x \in V \setminus V^+\) has positive neighbors contained in the dominating set, i.e., \( \Delta(x) \cap S^+ \ne \emptyset\), then \( h_{S^+}(x') = +1 \). 
    The reaction set \( \Delta_{h_{S^+}}(x) \) ensures that \( x \) moves to one of these neighbors. Specifically, the reaction set allows \( x \) to improve and move to a neighboring vertex \( x' \in \Delta(x) \cap S^+ \) such that \( f^*(x') = +1 \). Thus, \( h(x') = f^*(x') = +1 \) for any $x'\in\Delta_{h_{S^+}}(x)$ implying $\textsc{Loss}(x;h_{S^+},f^*)=0$.
    
\end{proof}

%% file: arxiv_exp_appendix_no_comment.tex
\clearpage
\section{Evaluation: Supplementary Details}\label{app:sec_eval}

This section includes supplementary details on the datasets and classifiers used, how improvement is done and results of the empirical evaluations.

\subsection{Datasets}\label{app:sec_datasets}

We utilize three real-world datasets: the Adult Income dataset from UCI and the Open University Learning Analytics Dataset (OULAD) and Law School datasets sourced from Le Quy et al.~\cite{le2022surveygit}. 
The preprocessing steps for all the datasets, similar to those described in  Le Quy et al.~\cite{le2022surveygit}, include removing missing data and applying label encoding to categorical variables.
In addition to the real-world datasets, we generate an 8-dimensional synthetic dataset with increased separability (\textit{class\_sep} = \(4\)) using the \textit{make\_classification} function from Scikit-learn. We clean the dataset by removing duplicates and outliers, with Z-scores applied with thresholds (\(0.9\) for class \(0\) and \(0.8\) for class \(1\)). The cleaned synthetic dataset is then balanced using SMOTE \cite{smote2002} to ensure class balance.

Statistical details of the datasets, including test/train sizes and number of features, are in Table~\ref{tab:datasets_info}.  We examine the structural variations within datasets to gain deeper insights into how the characteristics influence the impact of improvements on error drop rates. Figure~\ref{fig:orig_y_hist} highlights the target distribution across training datasets: the Adult dataset has a higher proportion of negative examples, whereas the OULAD and Law School datasets have a higher percentage of positive examples. The synthetic dataset, by contrast, is balanced.
Figure~\ref{fig:jumbleness} and \ref{fig:orig_knn} illustrate dataset separability properties, showing that the synthetic dataset (\(k\)-NN error: \(0.1016\)) and the Law School dataset (\(k\)-NN error: \(0.1010\)) have the highest separability. However, as Figure~\ref{fig:orig_knn} shows, the Adult and synthetic datasets exhibit the lowest false positive (FP) outlier rates.
\renewcommand{\arraystretch}{1.0} 
\setlength{\abovecaptionskip}{3.6pt}
\setlength{\belowcaptionskip}{3.6pt}
\begin{table*}[htbp]
\scriptsize
\centering
\begin{tabular}{p{1.8cm}p{3.7cm}p{0.6cm}p{1.7cm}p{6cm}}
    \toprule
    \textbf{Dataset} & \textbf{Target variable} & \(d\) & \textbf{Train/Test} & \textbf{Improvable features} \\ 
    \midrule
    Adult & \(\{1(>50K), \ 0(\leq50K)\}\) & \(14\) & \(21113/9049\) & \{``hours-per-week, capital-gain, capital-loss, fnlwgt, educational-num, workclass, education, occupation''\} \\ 
    OULAD & \(\{1(\text{pass}), \ 0(\text{fail})\}\) & \(11\) & \(15093/6469\) & \{``code\_module, code\_presentation, imd\_band, highest\_education, num\_of\_prev\_attempts, studied\_credits''\} \\ 
    Law school & \(\{1(\text{pass}), \ 0(\text{fail})\}\) & \(11\) & \(14558/6240\) & \{``decile1b, decile3, lsat, ugpa, zfygpa, zgpa, fulltime, fam\_inc, tier''\} \\ 
    Synthetic & \(\{1(\text{positive}), \ 0(\text{negative})\}\) & \(8\) & \(1561/669\) & \{all features are used\} \\ 
    \bottomrule
\end{tabular}  
\caption{Details of the tabular datasets, both synthetic and real, used in the experiments.}
\label{tab:datasets_info}
\end{table*}
\begin{figure}[!b]
    \centering
    \subfloat[Adult\label{fig:adult_y_hist}]{
        \begin{minipage}{0.22\linewidth}
            \centering
            \includegraphics[width=\linewidth]{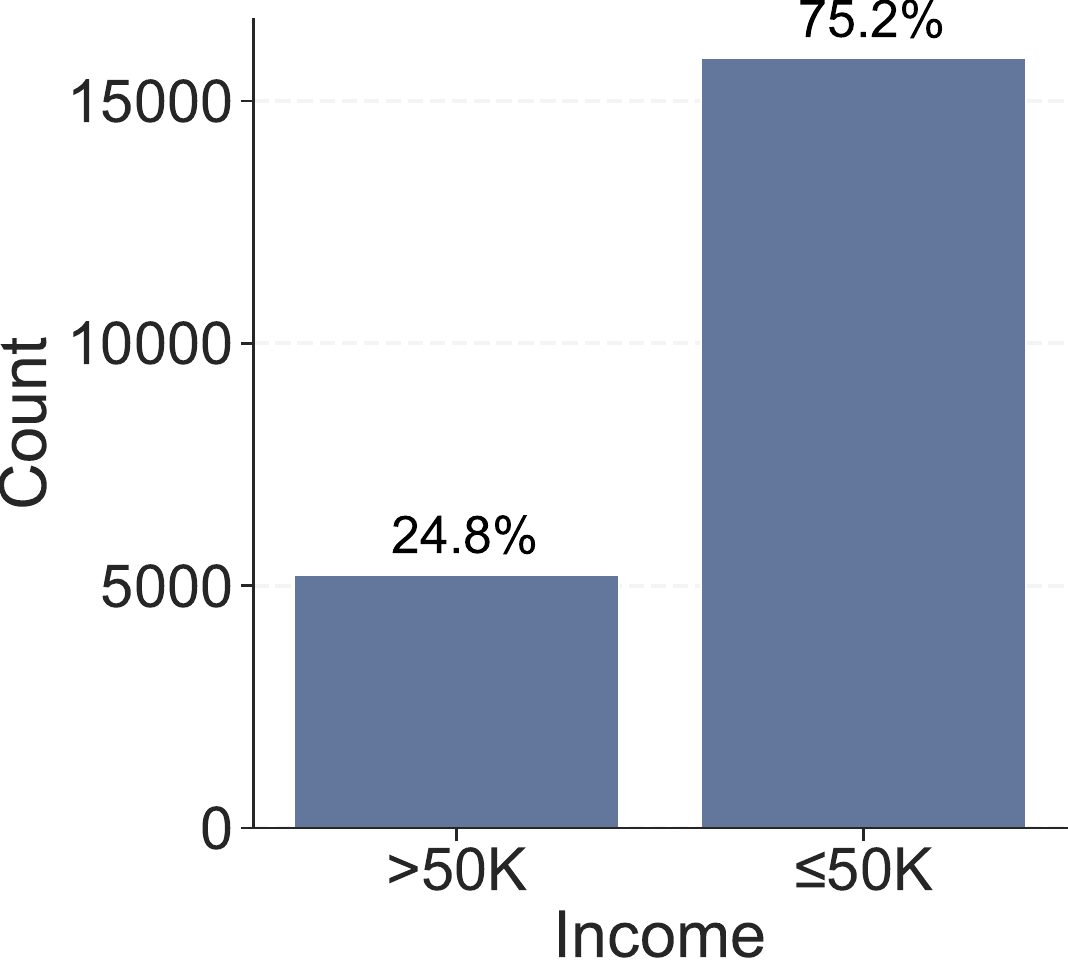}
        \end{minipage}
    }%
    \hfill
    \subfloat[OULAD\label{fig:oulad_y_hist}]{
        \begin{minipage}{0.22\linewidth}
            \centering
            \includegraphics[width=\linewidth]{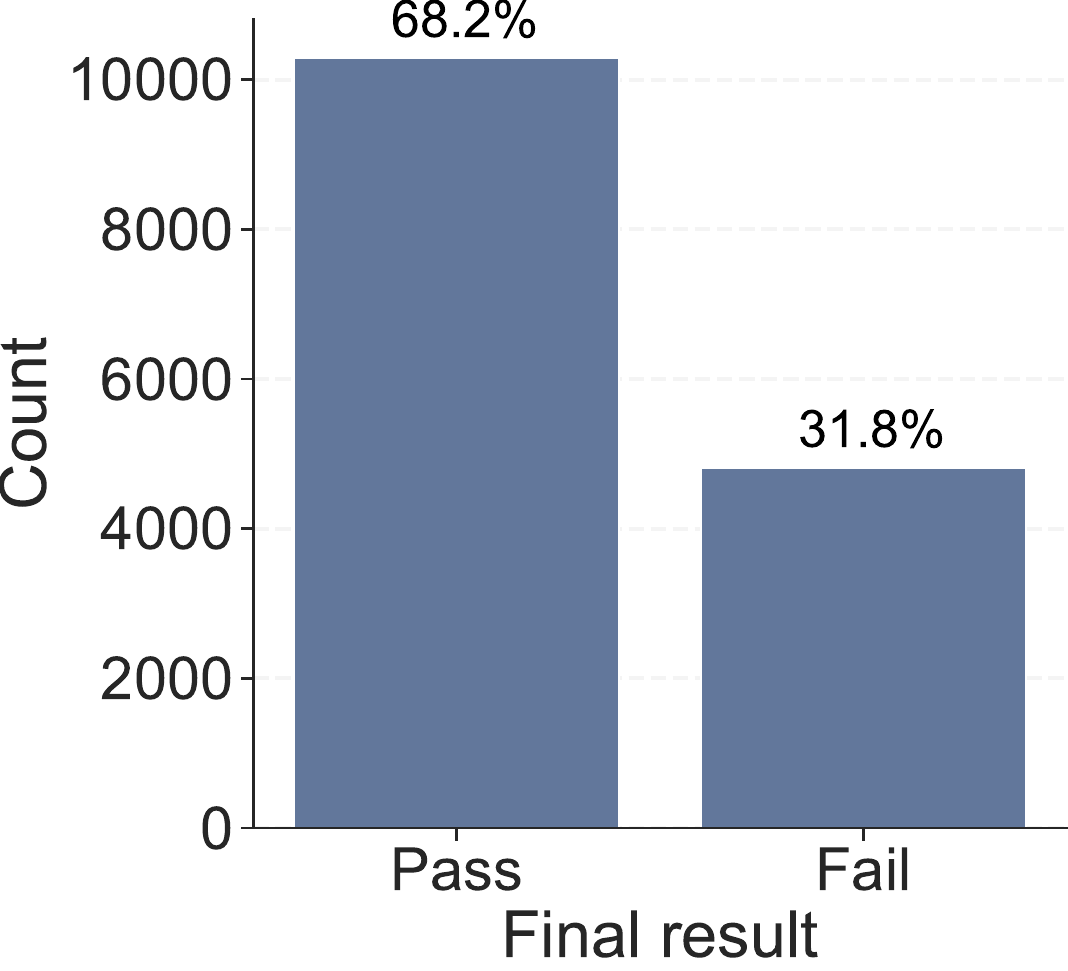}
        \end{minipage}
    }%
    \hfill
    \subfloat[Law School\label{fig:law_y_hist}]{
        \begin{minipage}{0.22\linewidth}
            \centering
            \includegraphics[width=\linewidth]{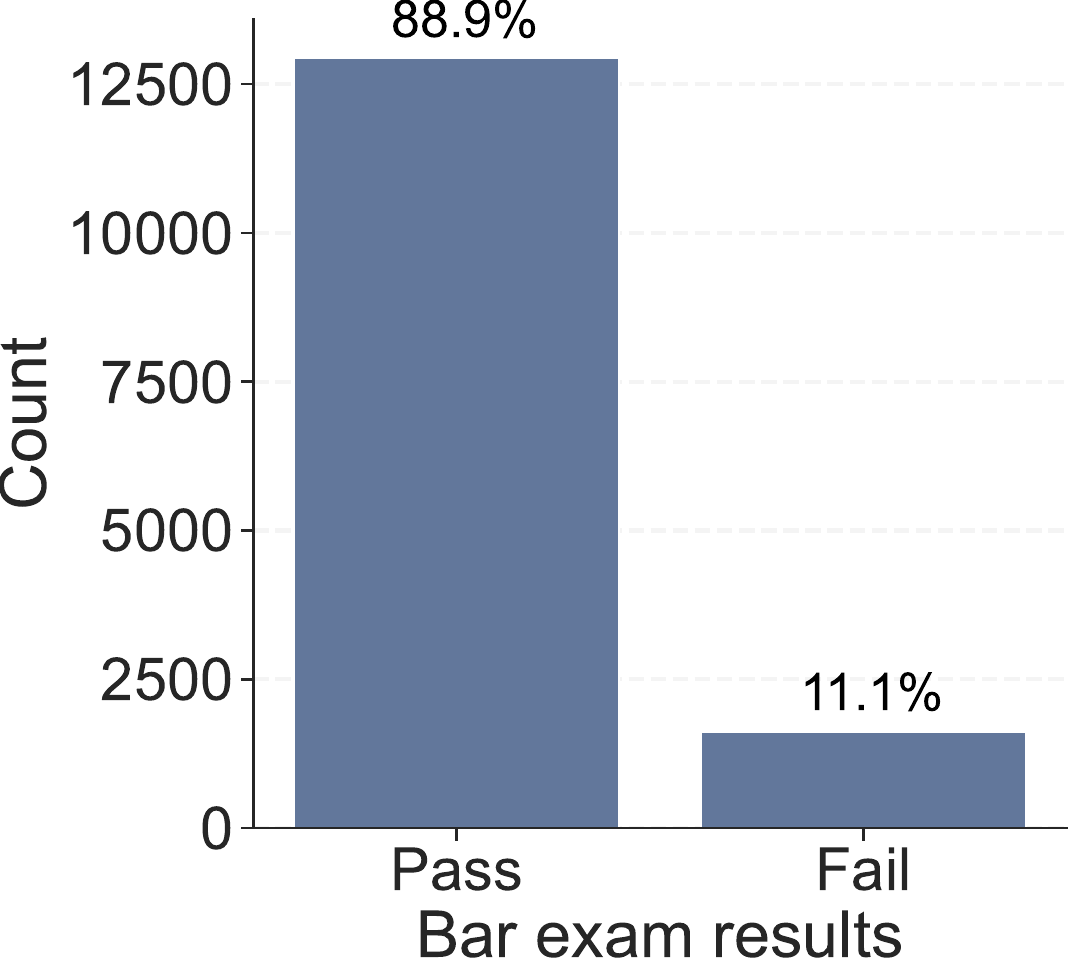}
        \end{minipage}
    }
    \hfill
    \subfloat[Synthetic\label{fig:syn2_y_hist}]{
        \begin{minipage}{0.22\linewidth}
            \centering
            \includegraphics[width=\linewidth]{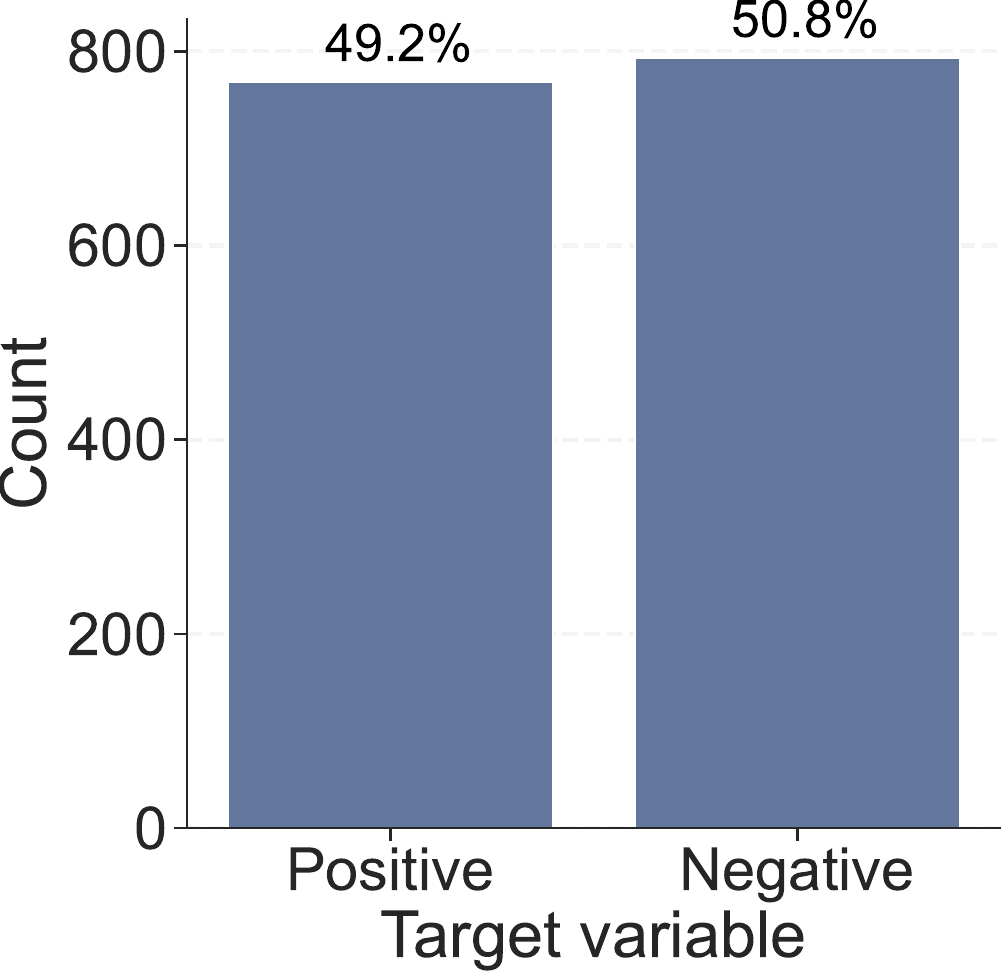}
        \end{minipage}
    }%
    \caption{Target variable distributions of synthetic and real-world train datasets used in the experiments for the (\subref{fig:adult_y_hist}) Adult, (\subref{fig:oulad_y_hist}) OULAD,
    (\subref{fig:law_y_hist}) Law School, and (\subref{fig:syn2_y_hist}) Synthetic datasets.}
    \label{fig:orig_y_hist}
\end{figure}
\begin{figure}[htbp]
    \centering
    \includegraphics[width=0.9\linewidth]{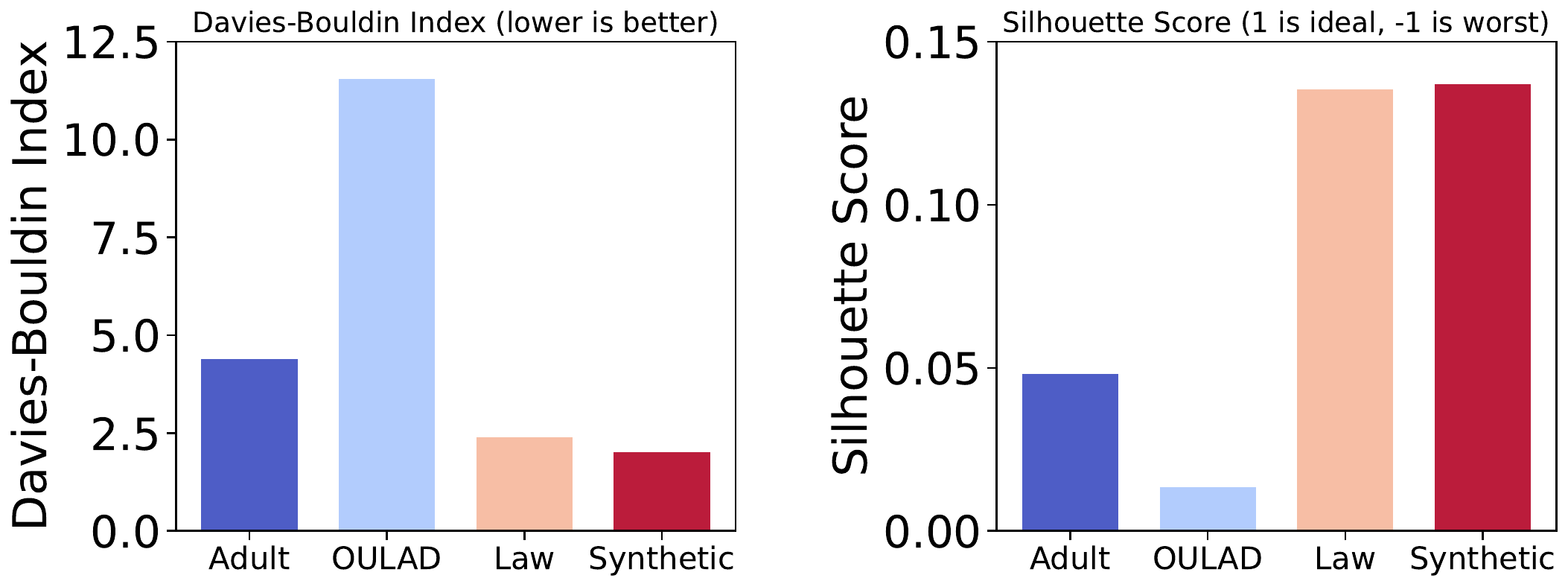}
    \caption{Inspection of clusteredness and class separation using Davies–Bouldin index \cite{davies_bouldin} and  Silhouettes scores \cite{ROUSSEEUW198753,silhouette}\label{fig:jumbleness}.}
\end{figure}
\begin{figure}[t!]
    \centering
    \subfloat[Adult\label{fig:adult_knn}]{
        \begin{minipage}{0.45\linewidth}
            \centering
            \includegraphics[width=\linewidth]{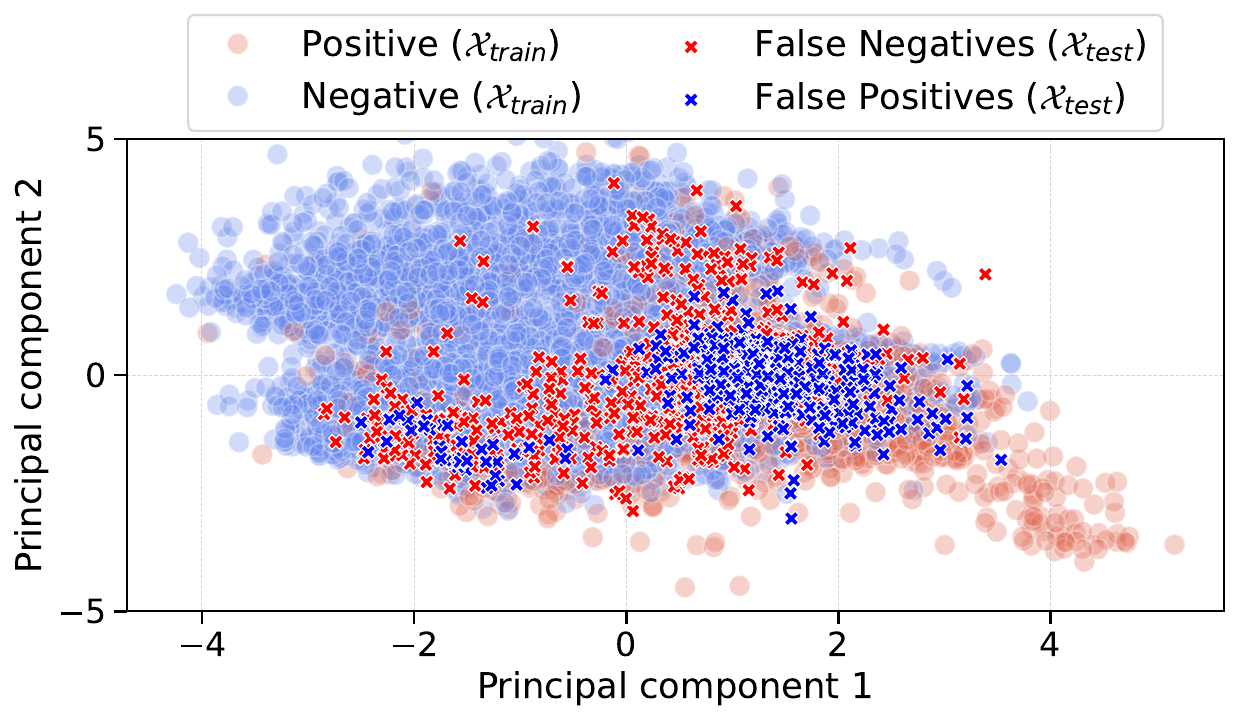}
        \end{minipage}
    }%
    \hfill
    \subfloat[OULAD\label{fig:oulad_knn}]{
        \begin{minipage}{0.45\linewidth}
            \centering
            \includegraphics[width=\linewidth]{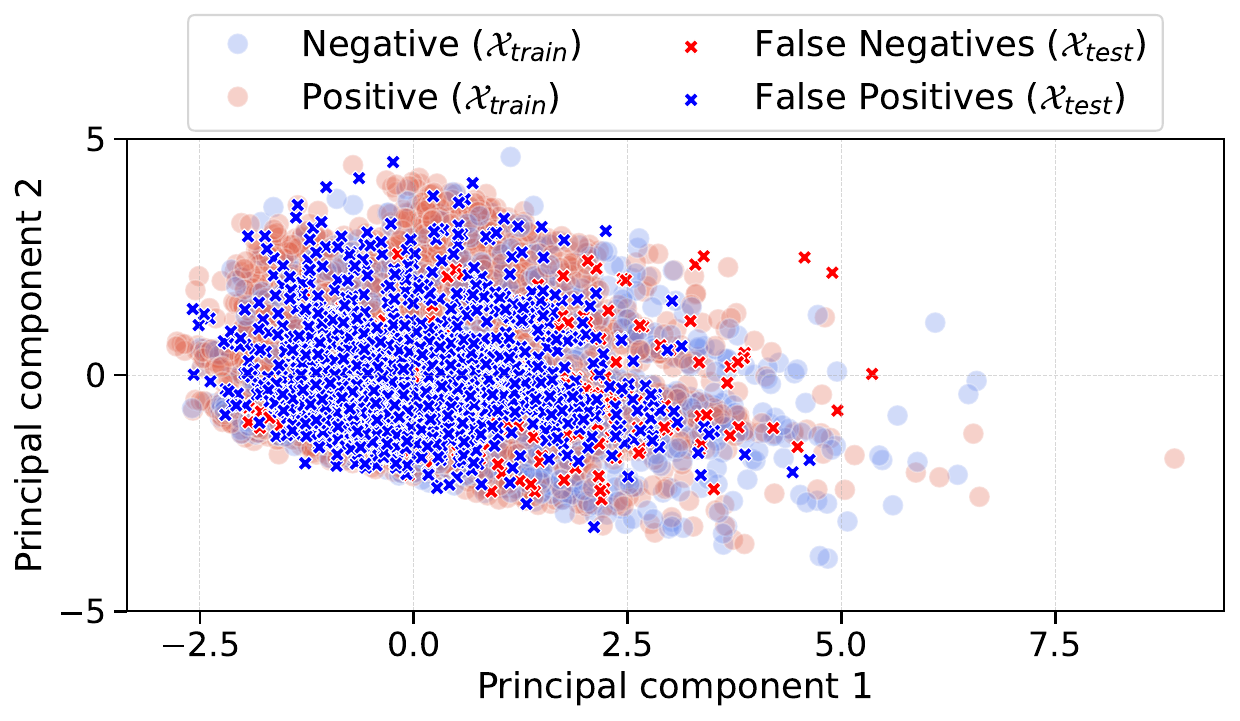}
        \end{minipage}
    }%
    \hfill
    \subfloat[Law school\label{fig:law_knn}]{
        \begin{minipage}{0.45\linewidth}
            \centering
            \includegraphics[width=\linewidth]{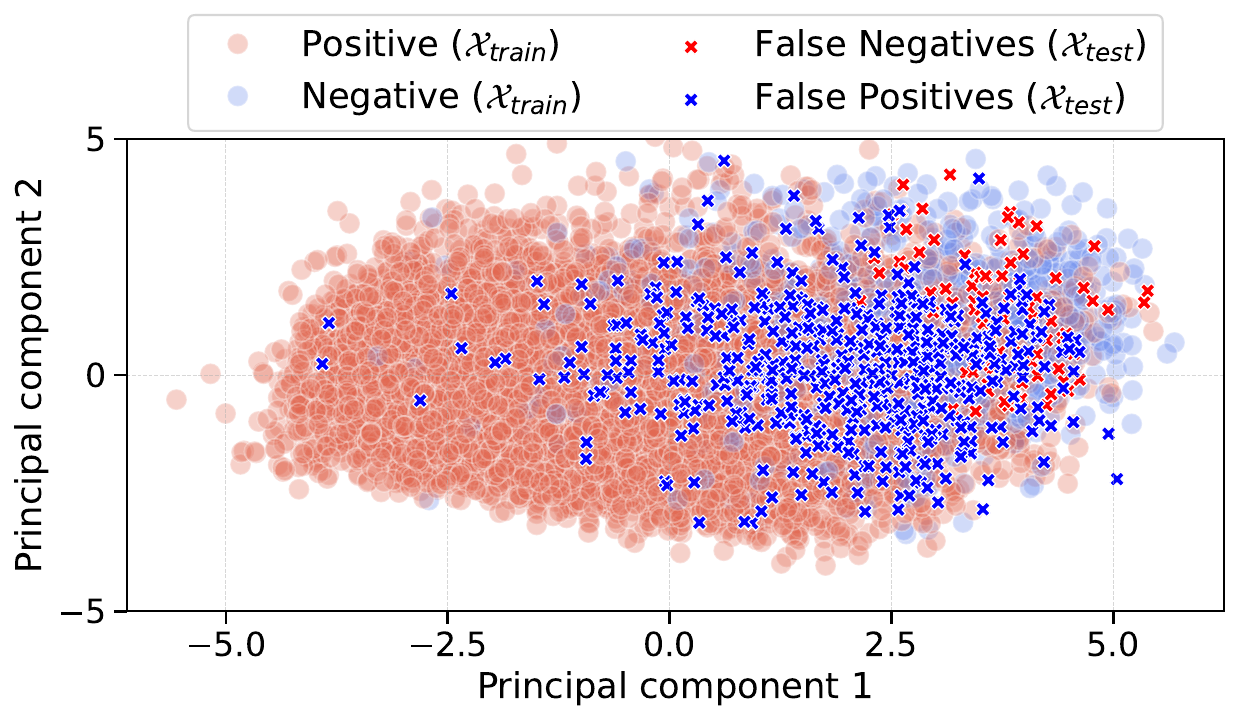}
        \end{minipage}
    }
    \hfill
    \subfloat[Synthetic\label{fig:syn2_knn}]{
        \begin{minipage}{0.45\linewidth}
            \centering
            \includegraphics[width=\linewidth]{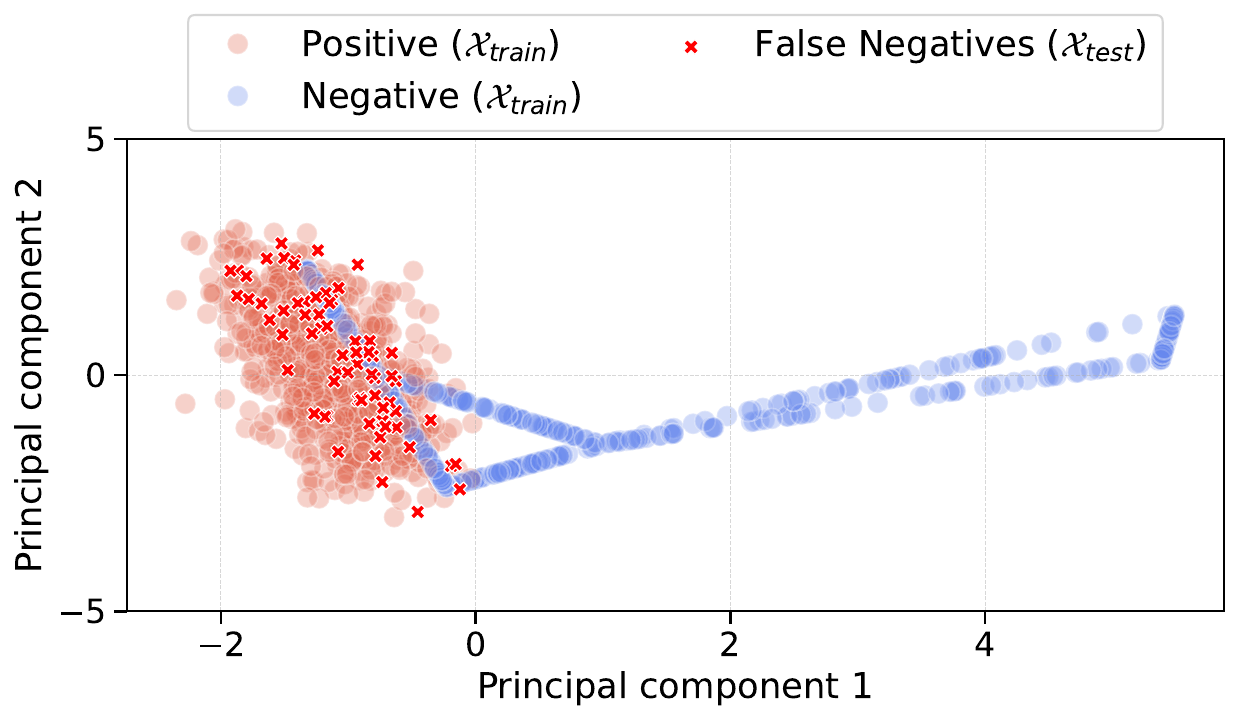}
        \end{minipage}
    } 
    \caption{Scatter plots of the two principal components of the training data and of the \(k\)-NN misclassification on test data for the (\subref{fig:adult_knn}) Adult (\(k\)-NN error: \(0.1670\), FNR: \(0.4610\), FPR: \(0.0678\)), (\subref{fig:oulad_knn}) OULAD (\(k\)-NN error: \(0.3291\), FNR: \(0.1195\), FPR: \(0.7645\)), (\subref{fig:law_knn}) Law School (\(k\)-NN error: \(0.1010\), FNR: \(0.0180\), FPR: \(0.7875\)), and  (\subref{fig:syn2_knn}) Synthetic (\(k\)-NN error: \(0.1016\), FNR: \(0.1960\), FPR: \(0.000\)) datasets.}
    \label{fig:orig_knn}
\end{figure}
\clearpage
\begin{table*}[h]
    \centering
    \footnotesize
    \setlength{\tabcolsep}{3pt} 
    \begin{tabular}{llllll}
        \toprule
        \textbf{Dataset} & \textbf{DTC1} & \textbf{DTC2} & \textbf{RFC1} & \textbf{RFC2} & \textbf{XGB} \\
        \midrule
        Adult      & \(0.999967 \pm 0.000049\) & \(0.999967 \pm 0.000049\) & \(0.999934 \pm 0.000079\) & \(0.999967 \pm 0.000049\) & \(0.999967 \pm 0.000049\) \\
        Law        & \(1.000000 \pm 0.000000\) & \(1.000000 \pm 0.000000\) & \(1.000000 \pm 0.000000\) & \(1.000000 \pm 0.000000\) & \(0.999952 \pm 0.000071\) \\
        OULAD      & \(1.000000 \pm 0.000000\) & \(1.000000 \pm 0.000000\) & \(1.000000 \pm 0.000000\) & \(1.000000 \pm 0.000000\) & \(1.000000 \pm 0.000000\) \\
        Synthetic & \(1.000000 \pm 0.000000\) & \(1.000000 \pm 0.000000\) & \(1.000000 \pm 0.000000\) & \(1.000000 \pm 0.000000\) & \(1.000000 \pm 0.000000\) \\
        \bottomrule
    \end{tabular}
    \caption{Accuracy score of the \(f^\star\) models when trained and tested on \(\mathcal{S}_{T}\) across different datasets.}
    \label{tab:fstarclassifiers_gen_acc}
\end{table*}
\begin{table*}[h]
    \centering
    \footnotesize
    \setlength{\tabcolsep}{3pt} 
    \begin{tabular}{llllll}
        \toprule
        \textbf{Dataset} & \textbf{DTC1 (LOO)} & \textbf{DTC2 (LOO)} & \textbf{RFC1 (LOO)} & \textbf{RFC2 (LOO)} & \textbf{XGB (LOO)} \\
        \midrule
        Adult      & \(0.8084 \pm 0.0044\) & \(0.8053 \pm 0.0045\) & \(0.8541 \pm 0.0040\) & \(0.8545 \pm 0.0040\) & --- \\
        Law        & \(0.8507 \pm 0.0048\) & \(0.8428 \pm 0.0049\) & \(0.8962 \pm 0.0041\) & \(0.8972 \pm 0.0041\) & \(0.8894 \pm 0.0043\) \\
        OULAD      & \(0.5941\pm 0.0066 \) & \(0.5952 \pm 0.0066 \) & \(0.6684 \pm 0.00663\) & \(0.6689 \pm 0.0063\) & --- \\
        Synthetic & \(0.9955 \pm 0.0028\) & \(0.9951 \pm 0.0029\) & \(0.9969 \pm 0.0023\) & \(0.9973 \pm  0.0021\) & \(0.9955 \pm 0.0028\) \\
        \bottomrule
    \end{tabular}
    \caption{Average leave one out (LOO) score of the \(5, f^\star\) models on \(\mathcal{S}_{T}\) across different datasets.}
    \label{tab:fstarclassifiers_loo}
\end{table*}

\subsection{Classifiers}\label{app:sec_classifiers}

In all experiments we set the random seed to \(42\) to ensure reproducibility and consistency across all runs. 
All experiments were conducted on a laptop computer with the following hardware specifications: \(2.6\)-GHz 6-Core Intel Core i7 processor, \(16\) GB of \(2400\)-MHz DDR4 RAM, and an Intel UHD Graphics \(630\) graphics card with \(1536\) MB of memory. Below are supplementary details about the classification models used.
Below are supplementary details about the classification models used.

\subsubsection{The \texorpdfstring{\(f^\star\) model} .} 

The function \( f^\star \) served as the ground truth labeler, assessing whether the agent's modifications led to a successful improvement. We evaluated five standard machine learning binary classification models, each achieving near \(100\%\) accuracy when trained and tested on \(\mathcal{S}_{T}\) (see Table~\ref{tab:fstarclassifiers_gen_acc}). These models include two decision tree classifiers (DTC1 and DTC2), two random forest classifiers (RFC1 and RFC2), and a gradient boosting classifier (XGB). Descriptions of these models are provided below.
\begin{enumerate}
    \item \textbf{Model \(f^{\star}_{1}\)} (DTC1): A decision tree classifier with the following hyperparameters: criterion = ``entropy", min\_samples\_split = \(2\), min\_samples\_leaf = \(1\), and random\_state = \(42\).
    \item \textbf{Model \(f^{\star}_{2}\)} (DTC2): A decision tree classifier with the following hyperparameters: criterion = ``gini", min\_samples\_split = \(2\), min\_samples\_leaf = \(1\), and random\_state = \(42\).
    \item \textbf{Model \(f^{\star}_{3}\)} (RFC1): A random forest classifier with default settings and random\_state = \(42\).
    \item \textbf{Model \(f^{\star}_{4}\)} (RFC2): A random forest classifier with the following hyperparameters: n\_estimators = \(500\), min\_samples\_split = \(2\), min\_samples\_leaf = \(1\), max\_features = ``sqrt", bootstrap = True, oob\_score = True, and random\_state = \(42\).
    \item \textbf{Model \(f^{\star}_{5}\)} (XGB): A gradient boosting classifier with the following hyperparameters: n\_estimators = \(500\), max\_depth = \(50\), learning\_rate = \(0.088\), min\_child\_weight = \(2\), subsample=\(0.9\),  gamma = \(0.088\), and random\_state = \(42\).
\end{enumerate}
We define the ground truth labeler \(f^\star\) either as a singular near-\(100\%\) accuracy model (see Table~\ref{tab:fstarclassifiers_gen_acc}) or as an agreement among multiple near-\(100\%\) accuracy models.
\paragraph{The multi-defined \(f^\star\) model.} Although the five \(f^\star\) models described above achieve nearly \(100\%\) accuracy when trained and tested on \(\mathcal{S}_{T}\), we assessed their generalization using the leave-one-out (LOO) validation score. The observed differences in LOO validation scores (Table~\ref{tab:fstarclassifiers_loo}), despite similar and high accuracy scores (Table~\ref{tab:fstarclassifiers_gen_acc}), highlight potential generalization gaps. To account for this, we employ a multi-defined \(f^\star\) model to validate the experimental results.  \\

For a given data point \(x\), the five models: \(f^{\star}_{1}(x), f^{\star}_{2}(x), f^{\star}_{3}(x), f^{\star}_{4}(x)\) and \(f^{\star}_{5}(x)\) each make a prediction for the label of the data point. Based on these predictions, we define a boolean agreement mask \(M(x)\) that checks whether all four models agree on the prediction:
\[
M(x) = \mathbbm{1}( f^{\star}_{1}(x) = f^{\star}_{2}(x) = f^{\star}_{3}(x) = f^{\star}_{4}(x) = f^{\star}_{5}(x) )
\]
where \(\mathbbm{1}(\cdot)\) is the indicator function that outputs \(1\) if all four models agree, and \(0\) otherwise. 
Using this agreement mask, we define the ground truth labeling function \(f^{\star}(x)\) as follows:
\begin{equation}
    f^{\star}(x) = \begin{cases} 
        f^{\star}_{1}(x), & \text{if } M(x) = 1 \text{ (i.e., full agreement)} \\
        0, & \text{otherwise} 
    \end{cases}
\end{equation}

\paragraph{The singularly-defined \(f^\star\) model.} 
Alternatively, we define the labeling function  \(f^\star(x)\) using a single near-\(100\%\) accuracy model trained and tested on \(\mathcal{S}_{T}\), selected from the set \(\{f^{\star}_{1}(x), f^{\star}_{2}(x), f^{\star}_{3}(x),\) \(f^{\star}_{4}(x), f^{\star}_{5}(x)\}\).
Unless otherwise stated , all experimental results were obtained using the DTC2 model (\(f^{\star}_{2}(x)\)) as the designated singularly-defined \(f^\star(x)\) function.

\subsubsection{The decision-maker's model \texorpdfstring{(\(h\))}.}  

We trained two-layer neural networks, denoted as \(h\) functions, using PyTorch with Adam optimizer with a learning rate of \(0.001\) and a batch size of \(64\). These \(h\) functions generate decisions for the test set agents. In cases where the test agent receives a negative classification, they can, if within budget, improve their feature values to get the desired classification from the \(h\) function. Table~\ref{tab:h_classifiers} summarizes the performance metrics of the \(f^\star\) and \(h\) model functions, demonstrating their varied performance across the datasets.

Since the empirical setup evaluates the impact of improvement on \(h\)'s error drop rates, we vary the loss functions we train the model \(h\) function with. We use the standard binary cross entropy loss (BCE) and the risk-averse weighted-BCE (wBCE) loss functions defined in Equation~\ref{eq:losses}. In particular, because only negatively classified test-set agents improve, improvement (\(x'\)), if successful (that is, \(f^{\star}(x') = h(x') = 1\)) reduces the false negative rate and turns true negatives into true positives. On the other hand, when unsuccessful (that is, \(f^{\star}(x') = 0 \ \text{and} \ h(x') = 1\)), it increases the false positive rate by turning true negatives into false positives and false negatives into false positives. 

The model trained with the weighted-BCE loss corresponds to a more risk-averse classifier that penalizes the false positive (FP) errors more heavily than the false negative (FN) one, creating a more compact positive agreement region that ensures more successful improvements. We prioritize minimizing FPs by ensuring the false positive to false negative weight ratio \(\frac{w_{\textrm{FP}}}{w_{\textrm{FN}}} > 1 \) is high, for example, \(\frac{4.4}{0.001} = 4400\) for the adult dataset. Another form of risk-averse classification we consider is only classifying an agent as positive \textit{iff} the probability of being positive is high. That is to say, we use the standard threshold \(0.5\) for the standard classifier and a higher threshold \(0.9\) for a more risk-averse classifier.
\begin{table}[ht!]
    \centering
    \renewcommand{\arraystretch}{1.1} 
    \setlength{\tabcolsep}{7pt} 
    \begin{tabular}{l l l l l l}
        \toprule
        \textbf{Dataset} & \textbf{Model (kind)} & \textbf{Accuracy} & \textbf{Precision} & \textbf{Recall} & \textbf{F1 Score} \\
        \midrule
        Adult & 2-layer neural network  (\(h(x)\)) & 0.841087 & 0.699811 & 0.647677 & 0.672736 \\
        Law & 2-layer neural network (\(h(x)\)) & 0.901282 & 0.911691 & 0.984731 & 0.946805 \\
        OULAD & 2-layer neural network (\(h(x)\)) & 0.678003 & 0.698609 & 0.919853 & 0.794109 \\
        Synthetic & 2-layer neural network (\(h(x)\)) & 0.994021 & 1.000000 & 0.988473 & 0.994203 \\
        \hline
    \end{tabular}
    \caption{Average performance of the standard  models functions \(h\) across different datasets' test sets \(\mathcal{S}_{\textrm{test}}\) when test-set agent cannot improve that is, \(r=0\).}
    \label{tab:h_classifiers}
\end{table}

\newpage

\subsection{Agents Improvement}\label{app:sec_improve}

Given the feature vector of a negatively classified test-set agent, \(x_{\text{orig}}\) and it's negative label  \(h(x_{\text{orig}})\), the loss function \(\mathcal{L}\) (BCE or wBCE), improvement budget \(r\), step size \(\alpha\), number of iterations \(T\) and set of indices of improvement features \(S\), compute the agent's improvement features. For each dataset, we predefine the improvable features that the agents can change in order to get a desirable (positive) model outcome (see Table~\ref{tab:datasets_info}). We vary the improvement budget in the empirical setup to so as to assess the impact of improvement on the the error drop rates. 

Below are the steps of the improvement algorithm we used to compute each agent's improvement features.
\paragraph{Initialization:}
\[
x'_{(0)} = x_{\textrm{orig}}
\]

\paragraph{Iterative updates:}
For $t = 0, 1, \dots, T-1$:
\begin{enumerate}
    \item Compute the gradient of the loss $\mathcal{L}$ with respect to the agent's updates $x'_{(t)}$:
    \[
    \mathbf{g}_{(t)} = \nabla_{x'_{(t)}} \mathcal{L}\Big(h\big(x'_{(t)}\big), h\big(x_{\textrm{orig}}\big)\Big)
    \]
    \item Update the improvement features by taking a step in the direction of the sign of the gradient:
    \[
    \rho_{(t)}[i] =
    \begin{cases} 
        \alpha \cdot \text{sign}(\mathbf{g}_{(t)}[i]), & \text{if} \ i \in S \\
        0, & \text{otherwise}
    \end{cases}, \quad  \forall i \in [d]
    \]
    \[
    x'_{(t+1)} = x'_{(t)} + \rho_{(t)}
    \]
    \item Project the updated improvement features back onto the $r$-ball around the original features  \(x_{\text{orig}}\):
    \[
    x'_{(t+1)} = \big(x_{\textrm{orig}} + \textrm{clip}_{[-r, r]}(x'_{(t+1)} - x_{\textrm{orig}})\big)
    \]
\end{enumerate}
\paragraph{Improvement vector:} After $T$ iterations, the final agent's improvement is given by:
\[
x' = x'_{(T)}
\]

\subsection{Evaluation Results: Supplementary Details}\label{app:sec_results}

Following the key insights mentioned in the main paper in Section~\ref{sec:experiments}, below are the detailed observations from the experimental evaluation and the supplementary figures of the experimental results. 

\paragraph{Effect of dataset class separability.} For all datasets we consider, Figure~\ref{fig:main_thresh0.5} shows that as the improvement budget increases, error rates drop significantly, particularly when agents improve in response to a risk-averse model trained using the \(\mathcal{L}_{\textrm{wBCE}}\) loss function. Dataset characteristics notably influence performance. For instance, as shown in Figure~\ref{fig:jumbleness}, the Law school and Synthetic datasets exhibit the highest separability and require relatively less risk aversion to achieve substantial and close to zero error reductions. Among all datasets, the Synthetic dataset shows a sharper error decline, reaching zero as the improvement budget increases (refer to Figure~\ref{fig:main_thresh0.5}). Furthermore, as depicted in Figure~\ref{fig:orig_knn}, the Adult and Synthetic datasets demonstrate the lowest \(k\)-NN test-set false positive rates (FPRs) of \(0.0678\) and \(0.000\), respectively, compared to the OULAD and Law School datasets. As seen in Figure~\ref{fig:main_thresh0.5}, for both datasets, the error drops close to \(0\) as \(r\) increases.

\paragraph{Effect of risk-aversion and improvement budget.} Figures~\ref{fig:app_adult_0.5}, \ref{fig:app_oulad_0.5}, \ref{fig:app_law_0.5} and \ref{fig:app_synthetic_0.5}
show that when agents improve to a risk-averse model trained using \(\mathcal{L}_{\textrm{wBCE}}\) with \(\frac{w_{\textrm{FP}}}{w_{\textrm{FN}}} > 1\), the error rate decreases rapidly, particularly as the improvement budget increases. Notably, the higher the false positive to false negative weight ratio \(\frac{w_{\textrm{FP}}}{w_{\textrm{FN}}}\), the faster the reduction in error (see  Figure~\ref{fig:bce_wbce_threshvar}). 
In contrast, models trained with the standard \(\mathcal{L}_{\textrm{BCE}}\) loss function exhibit a slower error reduction rate, almost looking like a line, under the same conditions as the effects of improvement are minimal and cancel each other out (see Figure~\ref{fig:oulad_synthetic_move_erroreval_0.5}). 

Furthermore, the false negative rate (FNR) decreases as the improvement budget \(r\) grows when the agents respond (improve) to an \(\mathcal{L}_{\textrm{wBCE}}\)--trained model (see Figures~\ref{fig:app_adult_move_fpr_fnr_0.5},  \ref{fig:app_oulad_move_fpr_fnr_0.5}, and \ref{fig:app_synthetic_move_fpr_fnr_0.5}). This is because almost \(100\%\) of the false negatively classified agents improve to become true positives as shown in Figures~\ref{fig:app_adult_move_0.5}, \ref{fig:app_oulad_move_0.5}, and \ref{fig:app_synthetic_move_0.5}.
On the other hand, on datasets where the weighted-BCE loss function effectively removed false positives (close to \(0\) false positive rate (FPR)) before agents' improvement, remains very low,  in some cases close to \(0\) (e.g., in Figure~\ref{fig:app_synthetic_move_fpr_fnr_0.5}), since no or few agents become false positives after improvement (see Figure~\ref{fig:app_synthetic_move_0.5}). However, when the weighted-BCE loss function wasn't as effective, the false positive rate increases as the improvement budget \(r<2.0\) grows (see Figure~\ref{fig:app_oulad_move_fpr_fnr_0.5}). 

Although we observe similar trends when the threshold for classifying an agent as positive increases to \(0.9\) (instead of \(0.5\)), the error is slightly higher and the reduction becomes slower (Figure~\ref{fig:app_thresh0.5thresh0.9}). Additionally, while Figures~\ref{fig:main_thresh0.5} and \ref{fig:bce_wbce_threshvar} demonstrate diminishing returns for \(r > 2.0\), a different trend emerges in Figures~\ref{fig:app_adult_0.9}, \ref{fig:app_oulad_0.9}, \ref{fig:app_law_0.9} and \ref{fig:app_synthetic_0.9} where we classify agents positive with high probability (\(0.9\)).

\paragraph{Effect of choice of \(f^\star\) model.}
Although different \(f^\star\) models achieved \(\sim 100\%\) accuracy on a given dataset while exhibiting varied LOO accuracy (cf. Tables~\ref{tab:fstarclassifiers_gen_acc} and \ref{tab:fstarclassifiers_loo}), the error drop rate showed consistent patterns when different \(f^\star\) models are used. As shown in Figure~\ref{fig:app_oulad_multi_thresh0.5}, although some \(f^\star\) models, such as RFC2 (\(f^\star_{4}\)) (cf. Figure~\ref{fig:app_oulad_rfc2_0.5}), had higher performance gains than others, in all cases, the error drops rapidly as improvement budget increases and agents respond (improve) to wBCE-trained models.

Additionally, performance gains observed with evaluation of successful improvement using the multi-defined \(f^\star\) (cf. Figure~\ref{fig:app_multi_thresh0.5}) were quite similar in trend and gains to those when a singularly-defined model function \(f^\star = f^\star_{2}\) was used (cf. Figure~\ref{fig:app_thresh0.5thresh0.9}, column one (\subref{fig:app_adult_0.5}, \subref{fig:app_oulad_0.5}, \subref{fig:app_law_0.5} and \subref{fig:app_synthetic_0.5})).

Our results indicate that while the \(f^\star\) models achieve \(\sim 100\%\) accuracy when trained and tested on the unsplit dataset \(\mathcal{S}_{T}\), they often overfit, as shown by the LOO scores (cf. Table~\ref{tab:fstarclassifiers_loo}). Nevertheless, they yield comparable performance gains when assessing agents' improvement (cf. Figures~\ref{fig:app_oulad_multi_thresh0.5} and ~\ref{fig:app_multi_thresh0.5}). 

\begin{figure}[htb!]
    \centering
    \subfloat[Adult \(\big(\mathcal{L}_{\textrm{wBCE}} (w_{\textrm{FN}}=1.0)\big)\)\label{fig:app_adult_fn1_th0.5}]{
        \begin{minipage}{0.47\linewidth}
            \centering
            \includegraphics[width=\linewidth]{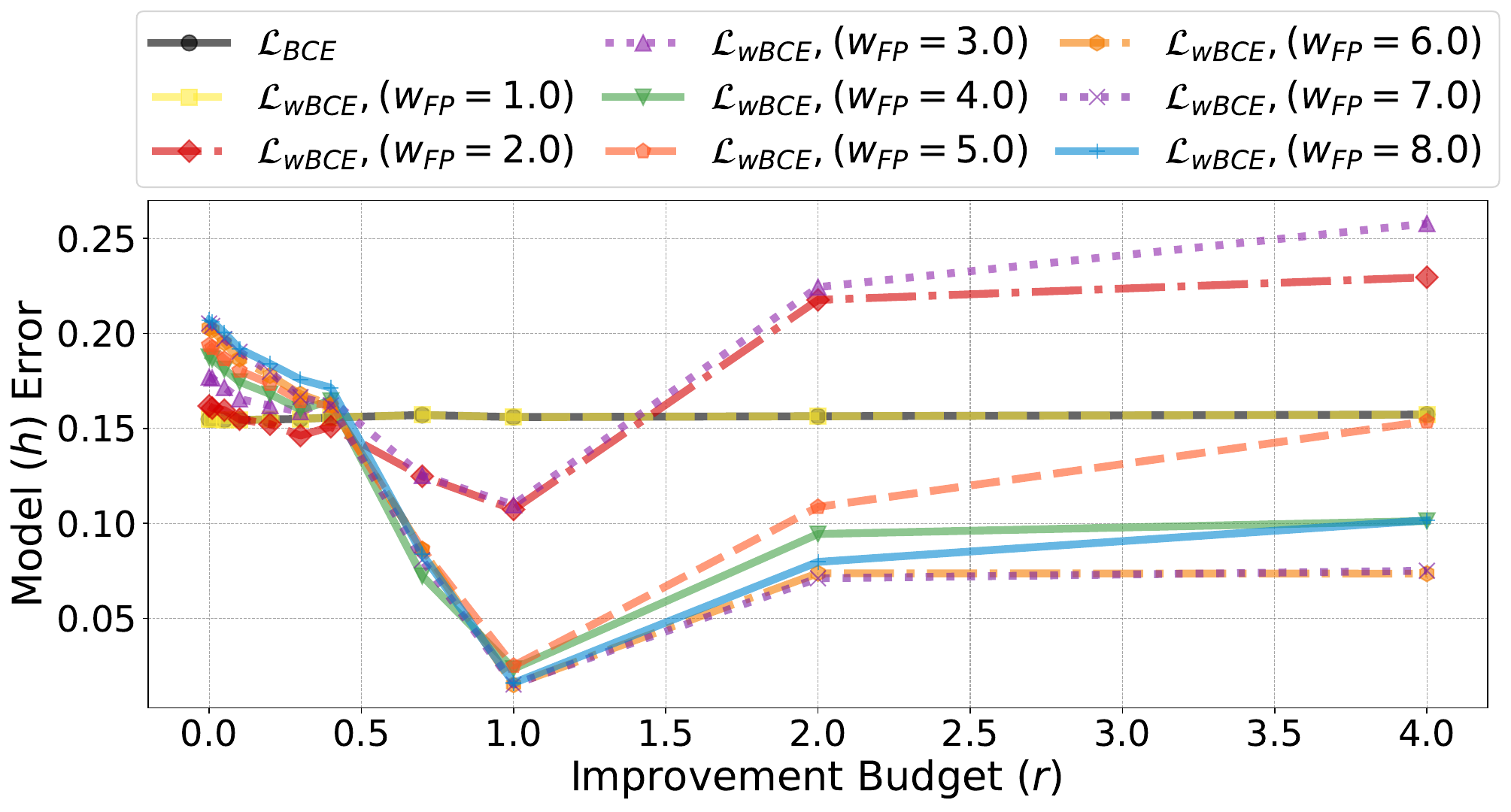}
        \end{minipage}
    } 
    ~
    \subfloat[Adult \(\big(\mathcal{L}_{\textrm{wBCE}} (w_{\textrm{FN}}=0.001)\big)\)\label{fig:app_adult_fnz_th0.5}]{
        \begin{minipage}{0.47\linewidth}
            \centering
            \includegraphics[width=\linewidth]{experimental_figures/lfconlyw_0.5_error_droprate_linfx_adult.pdf}
        \end{minipage}
    }  
    \vskip\baselineskip
    \subfloat[OULAD \(\big(\mathcal{L}_{\textrm{wBCE}} (w_{\textrm{FN}}=1.0)\big)\) \label{fig:app_oulad_fn1_th0.5}]{
        \begin{minipage}{0.47\linewidth}
            \centering
            \includegraphics[width=\linewidth]{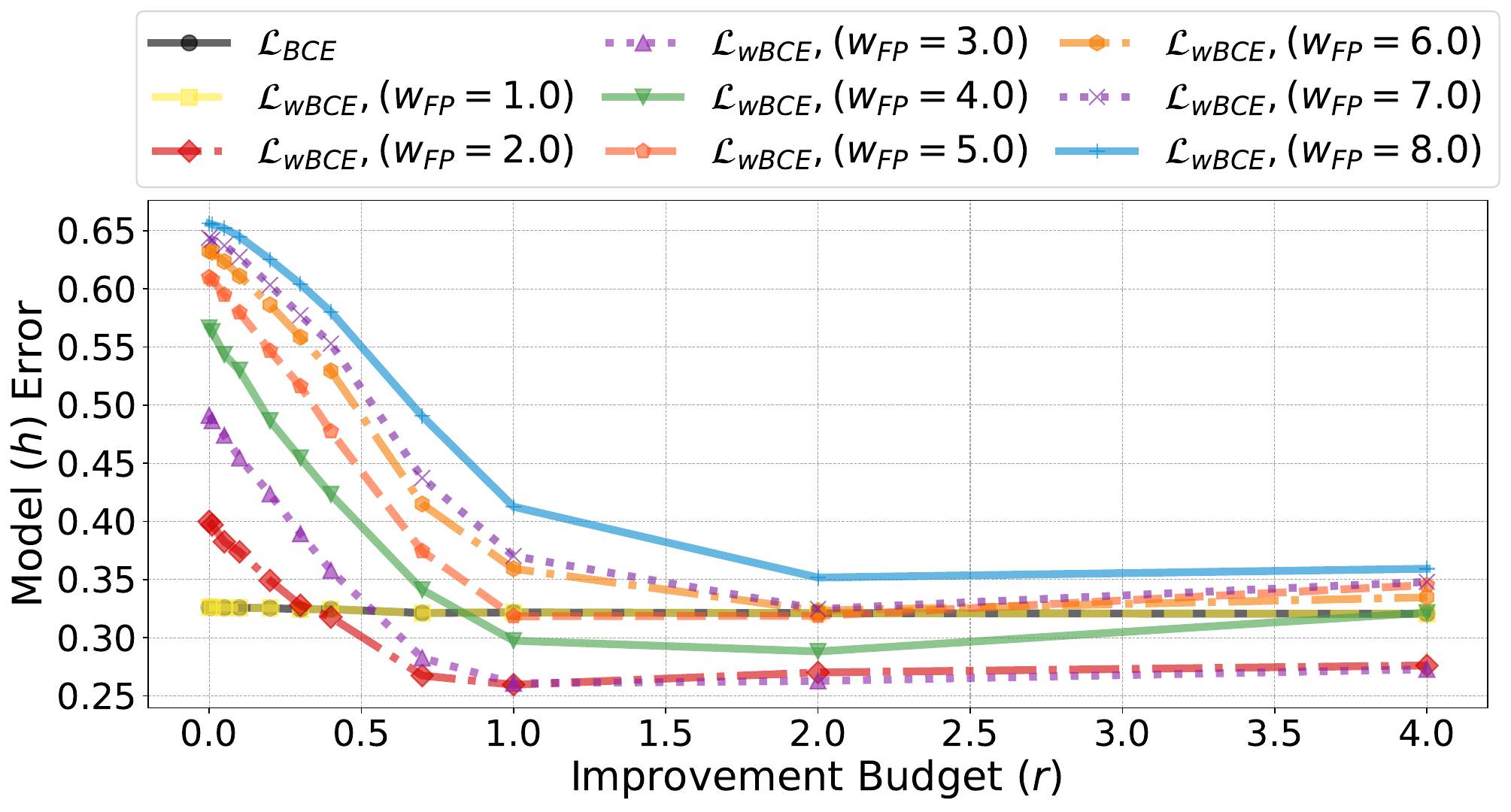}
        \end{minipage}
    } 
    ~
    \subfloat[OULAD \(\big(\mathcal{L}_{\textrm{wBCE}} (w_{\textrm{FN}}=1.33)\big)\) \label{fig:app_oulad_fnz_th0.5}]{
        \begin{minipage}{0.47\linewidth}
            \centering
            \includegraphics[width=\linewidth]{experimental_figures/lfconlyw_0.5_error_droprate_linfx_oulad.pdf}
        \end{minipage}
    }
    \vskip\baselineskip
    \subfloat[Law school \(\big(\mathcal{L}_{\textrm{wBCE}} (w_{\textrm{FN}}=1.0)\big)\) \label{fig:app_law_fn1_th0.5}]{
        \begin{minipage}{0.47\linewidth}
            \centering
            \includegraphics[width=\linewidth]{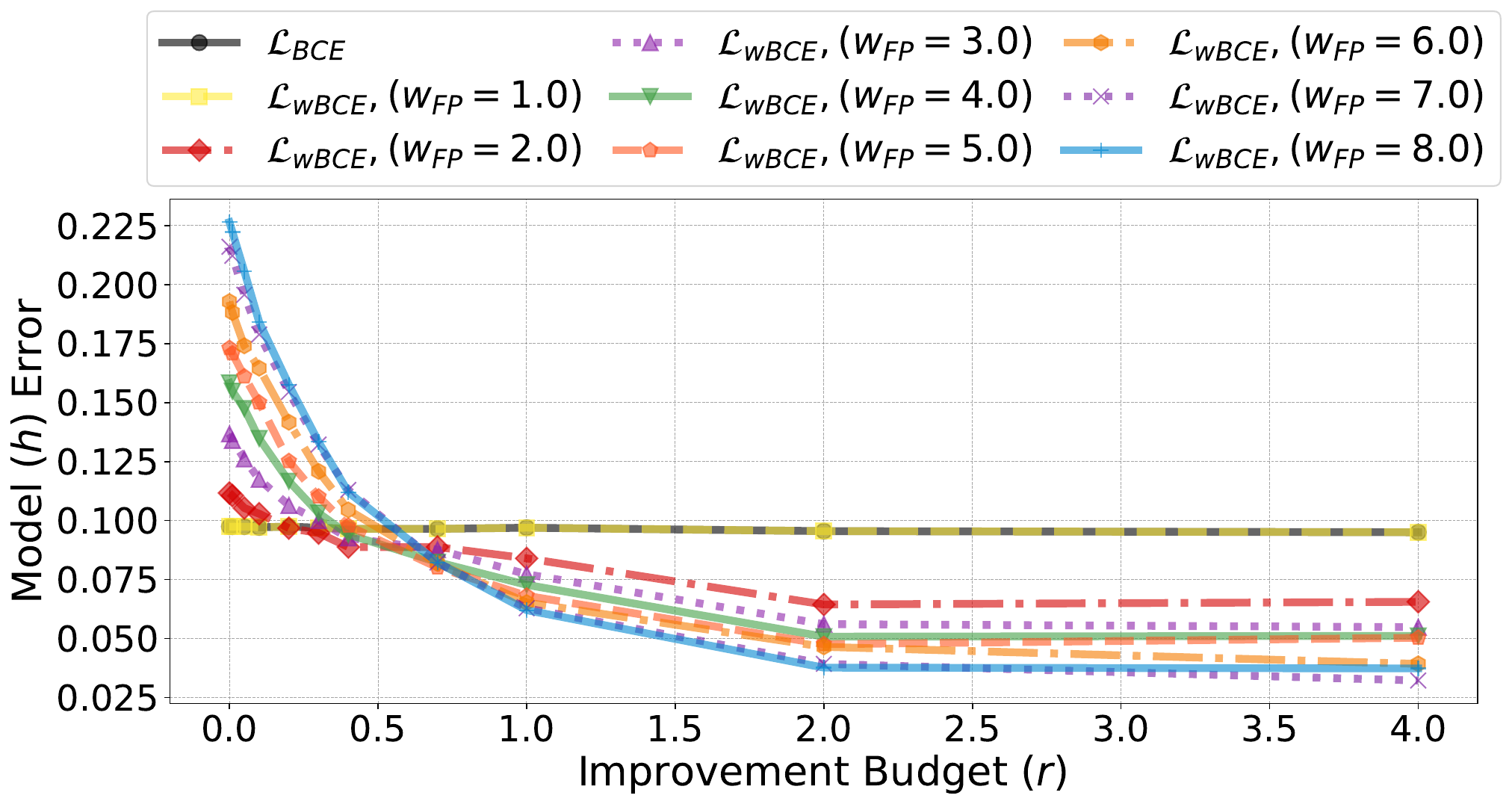}
        \end{minipage}
    } 
    ~
    \subfloat[Law school \(\big(\mathcal{L}_{\textrm{wBCE}} (w_{\textrm{FN}}=0.009)\big)\) \label{fig:app_law_fnz_th0.5}]{
        \begin{minipage}{0.47\linewidth}
            \centering
            \includegraphics[width=\linewidth]{experimental_figures/lfconlyw_0.5_error_droprate_linfx_law.pdf}
        \end{minipage}
    }
    \vskip\baselineskip
    \subfloat[Synthetic \(\big(\mathcal{L}_{\textrm{wBCE}} (w_{\textrm{FN}}=1.0)\big)\) \label{fig:app_syn2_fn1_th0.5}]{
        \begin{minipage}{0.47\linewidth}
            \centering
            \includegraphics[width=\linewidth]{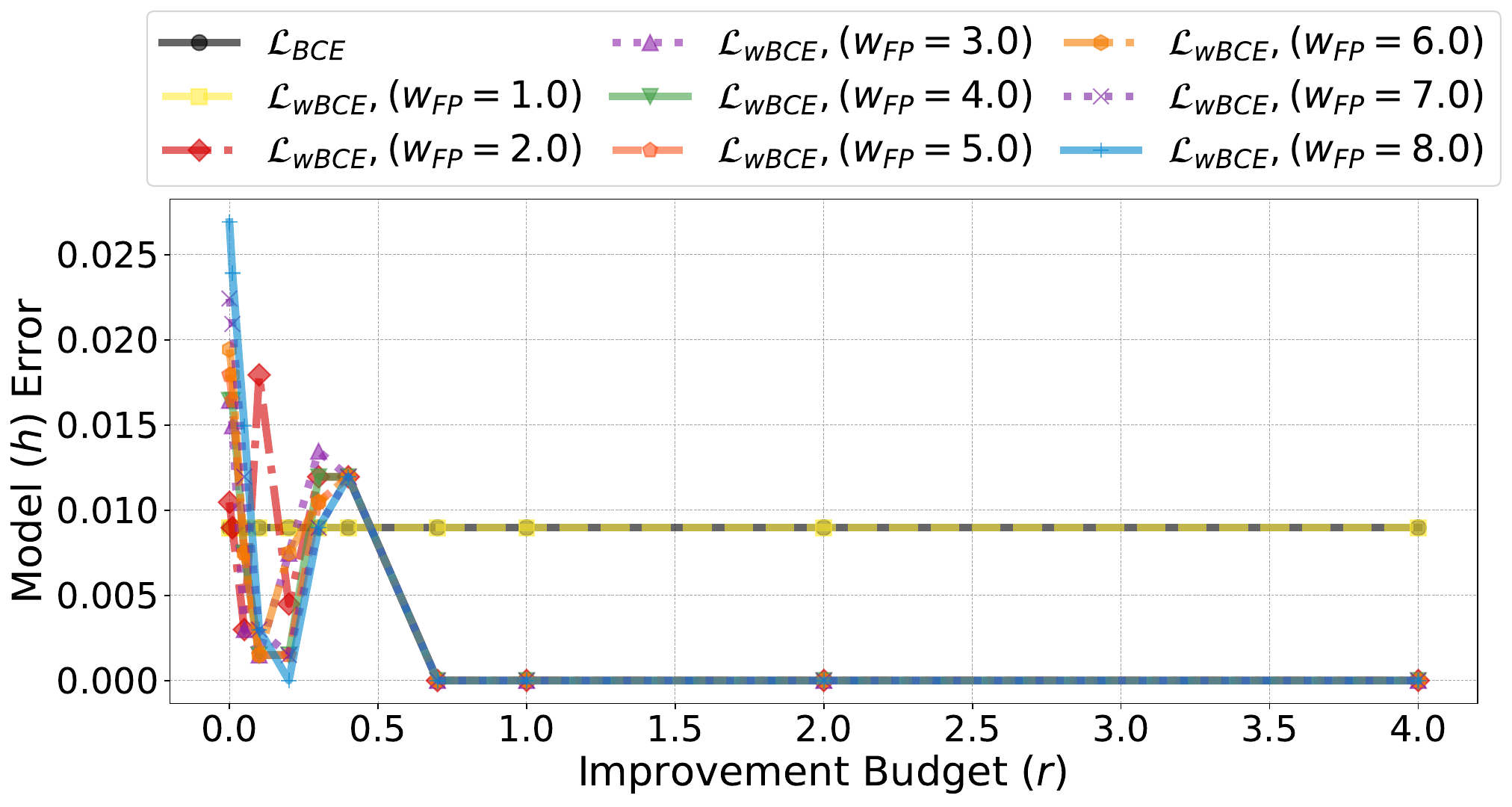}
        \end{minipage}
    } 
    ~
    \subfloat[Synthetic \(\big(\mathcal{L}_{\textrm{wBCE}} (w_{\textrm{FN}}=0.009\big)\) \label{fig:app_syn2_fnz_th0.5}]{
        \begin{minipage}{0.47\linewidth}
            \centering
            \includegraphics[width=\linewidth]{experimental_figures/lfconlyw_0.5_error_droprate_linfx_synthetic2.pdf}
        \end{minipage}
    } 
    \caption{Comparison of the error drop rate when agents improve to the risk-averse models trained with \(\mathcal{L}_{\textrm{wBCE}}\) where \(w_{\textrm{FN}}=1.0\) (\subref{fig:app_adult_fn1_th0.5}, \subref{fig:app_oulad_fn1_th0.5}, \subref{fig:app_law_fn1_th0.5}, and \subref{fig:app_syn2_fn1_th0.5}) and where \(w_{\textrm{FN}}\) (\subref{fig:app_adult_fnz_th0.5}, \subref{fig:app_oulad_fnz_th0.5}, \subref{fig:app_law_fnz_th0.5}, and \subref{fig:app_syn2_fnz_th0.5}) is optimized and false positive weight is varied \(w_{\textrm{FP}} = \{i\}_{i=1}^{8}\) across different datasets (Adult, OULAD, Law school and Synthetic). The standard model (black line) trained with \(\mathcal{L}_{\textrm{BCE}}\) loss function is such that \(w_{\textrm{FP}} = w_{\textrm{FN}}=1\), and in all cases an agent is classified as positive if the probability of being positive is above \(0.5\).}
    \label{fig:bce_wbce_threshvar}

    \begin{picture}(0,0)
        \put(-155,690){{\parbox{4cm}{\centering \(\mathcal{L}_{\textrm{wBCE}} (w_{\textrm{FN}}=1.0)\)}}}
        \put(55,690){{\parbox{4cm}{\centering \(\mathcal{L}_{\textrm{wBCE}} (w_{\textrm{FN}}=z)\)}}}
        \put(-240,596){\rotatebox{90}{Adult}}
        \put(-240,437){\rotatebox{90}{OULAD}}
        \put(-240,277){\rotatebox{90}{Law school}}
        \put(-240,128){\rotatebox{90}{Synthetic}}
        
    \end{picture}
\end{figure}
\begin{figure}[ht!]
    \centering
    \centering
    \subfloat[Adult (Movement of agents from TN/FN to TP/FP) \label{fig:app_adult_move_0.5}]{
        \begin{minipage}{0.46\linewidth}
            \centering
            \includegraphics[width=\linewidth]{experimental_figures/adult_all_movements_percentages_logscale.pdf}
        \end{minipage}
    } 
    \hfill
    \subfloat[Adult (FNR/FPR before and after agents' improvement) \label{fig:app_adult_move_fpr_fnr_0.5}]{
        \begin{minipage}{0.49\linewidth}
            \centering
            \includegraphics[width=\linewidth]{experimental_figures/adult_FNsFPs_percentages.pdf}
        \end{minipage}
    }%
    \vspace{1.6em}
    \subfloat[OULAD (Movement of agents from TN/FN to TP/FP)\label{fig:app_oulad_move_0.5}]{
        \begin{minipage}{0.46\linewidth}
            \centering
            \includegraphics[width=\linewidth]{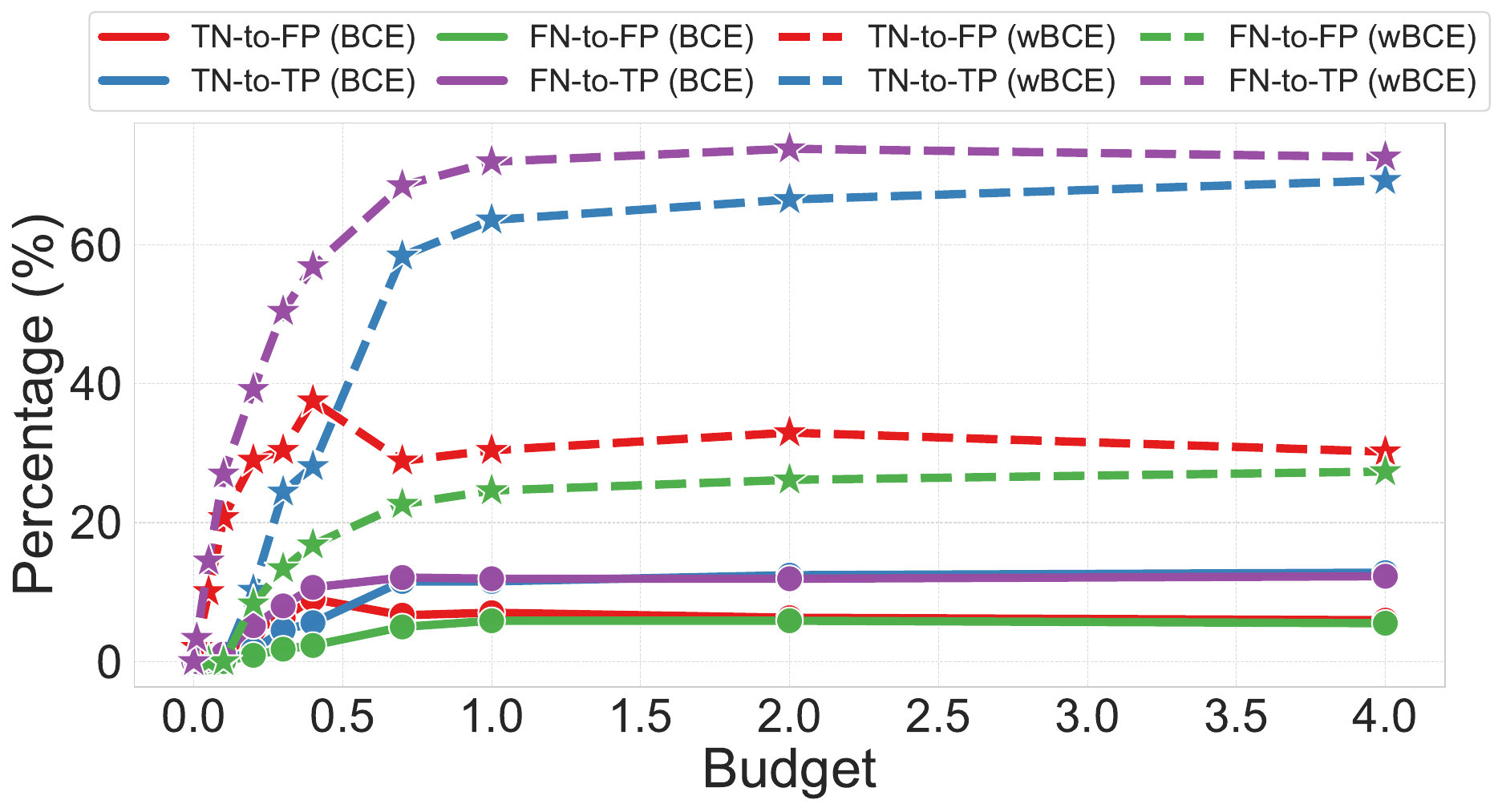}
        \end{minipage}
    }%
    \hfill
    \subfloat[OULAD (FNR/FPR before and after agents' improvement)  \label{fig:app_oulad_move_fpr_fnr_0.5}]{
        \begin{minipage}{0.49\linewidth}
            \centering
            \includegraphics[width=\linewidth]{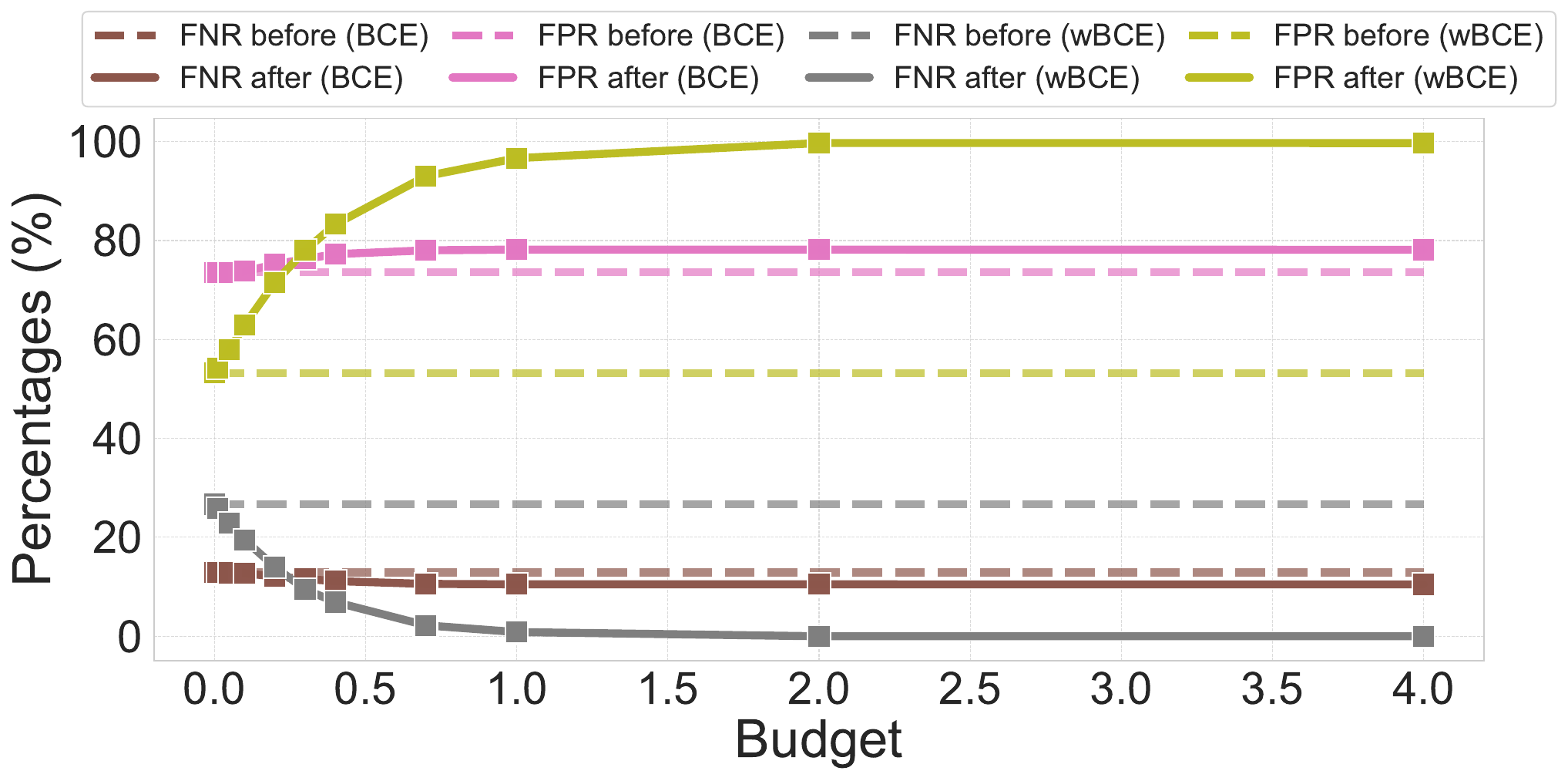}
        \end{minipage}
    }%
    \vspace{1em}
    \subfloat[Synthetic (Movement of agents from TN/FN to TP/FP) \label{fig:app_synthetic_move_0.5}]{
        \begin{minipage}{0.46\linewidth}
            \centering
            \includegraphics[width=\linewidth]{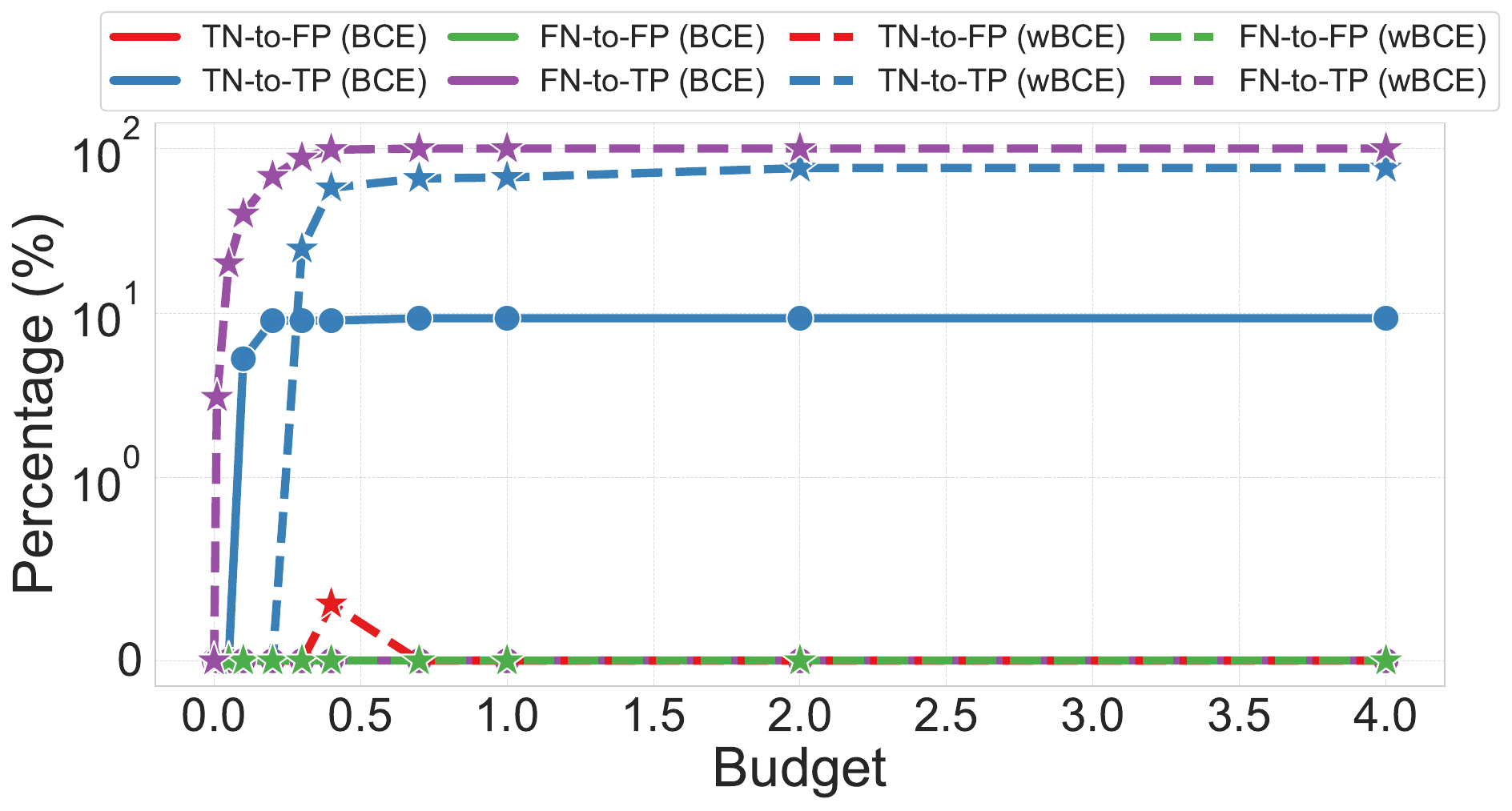}
        \end{minipage}
    }%
    \hfill
    \subfloat[Synthetic (FNR/FPR before and after agents' improvement) \label{fig:app_synthetic_move_fpr_fnr_0.5}]{
        \begin{minipage}{0.49\linewidth}
            \centering
            \includegraphics[width=\linewidth]{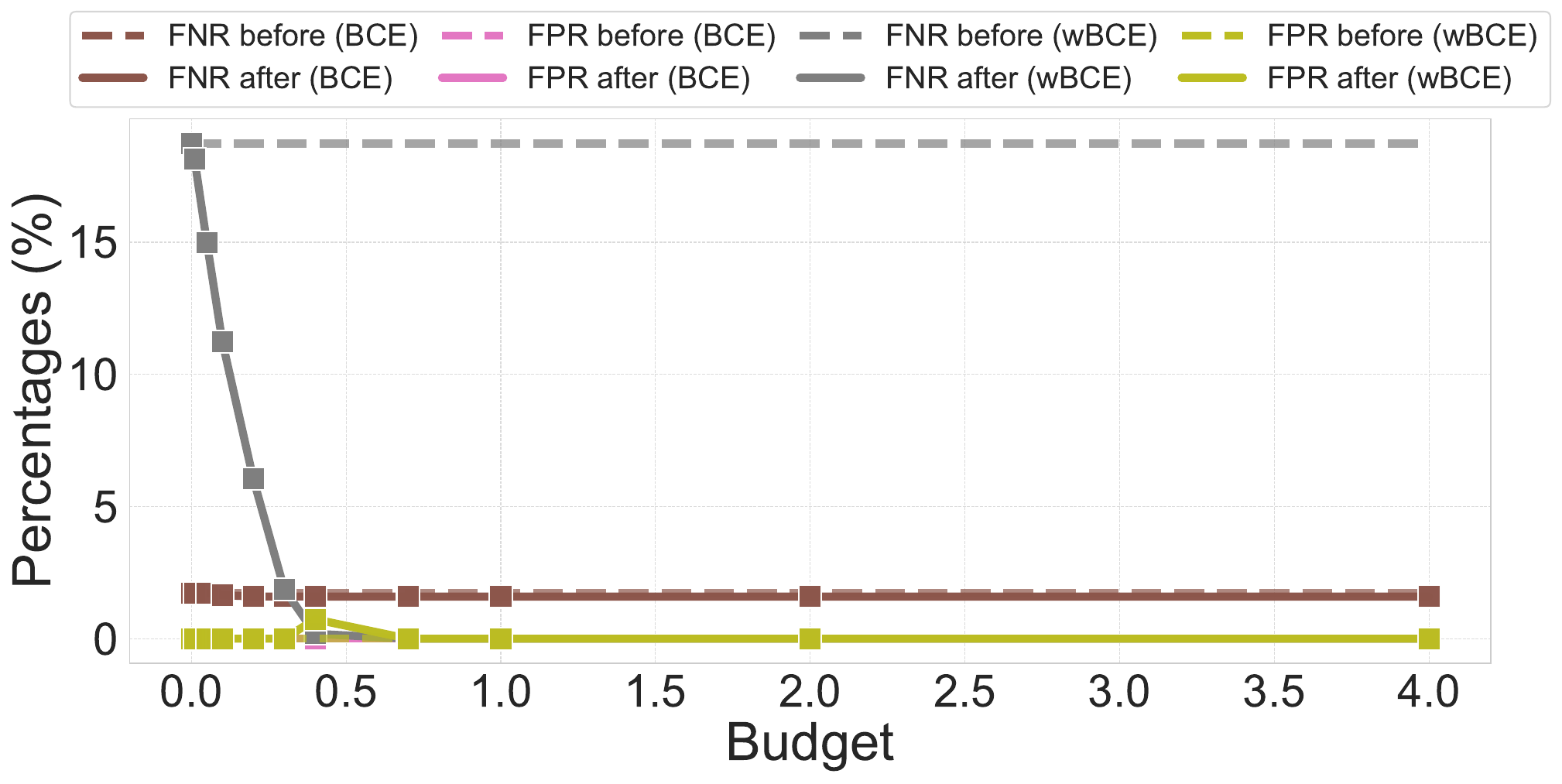}
        \end{minipage}
    }
    \caption{(\subref{fig:app_adult_move_0.5}, \subref{fig:app_oulad_move_0.5}, and \subref{fig:app_synthetic_move_0.5}) The percentage of negatively classified agents (TN and FN) that transition to TP and FP after responding to the classifier (\(h(x)\)). 
    (\subref{fig:app_adult_move_fpr_fnr_0.5}, \subref{fig:app_oulad_move_fpr_fnr_0.5}, and \subref{fig:app_synthetic_move_fpr_fnr_0.5}) The FPR and FNR before and after agents move. On the adult dataset, the wBCE-trained model used \(w_{\textrm{FP}}=0.001\) and \(w_{\textrm{FN}} = 4.4\), while on the OULAD dataset, it used \(w_{\textrm{FP}}=1.33\) and \(w_{\textrm{FN}} = 2\), and on the synthetic dataset, \(w_{\textrm{FP}}=0.009\) and \(w_{\textrm{FN}} = 1.0\). In all cases, BCE-trained models used \(w_{\textrm{FP}} = w_{\textrm{FN}} = 1\), and in all cases an agent is classified as positive if the probability of being positive is above \(0.5\).}
    \label{fig:oulad_synthetic_move_erroreval_0.5}
\end{figure}
\begin{figure}[htb!]
    \centering
    \subfloat[Adult \(\big(\mathcal{L}_{\textrm{wBCE}} (w_{\textrm{FN}}=0.001)\big)\)\label{fig:app_adult_0.5}]{
        \begin{minipage}{0.45\linewidth}
            \centering
            \includegraphics[width=\linewidth]{experimental_figures/lfconlyw_0.5_error_droprate_linfx_adult.pdf}
        \end{minipage}
    } 
    ~
    \subfloat[Adult \(\big(\mathcal{L}_{\textrm{wBCE}} (w_{\textrm{FN}}=0.001)\big)\)\label{fig:app_adult_0.9}]{
        \begin{minipage}{0.45\linewidth}
            \centering
            \includegraphics[width=\linewidth]{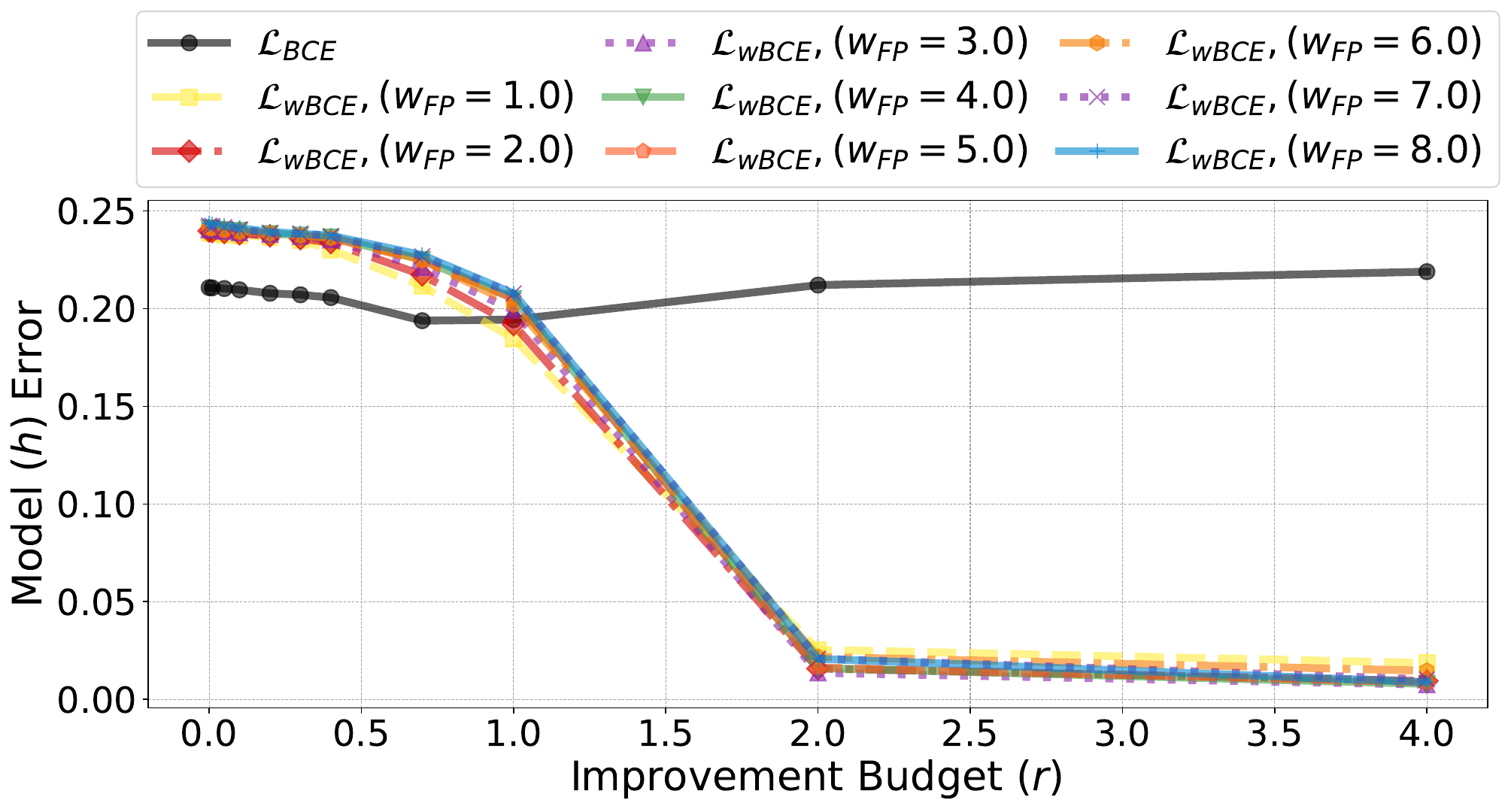}
        \end{minipage}
    } 
    \vskip\baselineskip
    \subfloat[OULAD \(\big(\mathcal{L}_{\textrm{wBCE}} (w_{\textrm{FN}}=1.33)\big)\) \label{fig:app_oulad_0.5}]{
        \begin{minipage}{0.45\linewidth}
            \centering
            \includegraphics[width=\linewidth]{experimental_figures/lfconlyw_0.5_error_droprate_linfx_oulad.pdf}
        \end{minipage}
    }
    ~
    \subfloat[OULAD \(\big(\mathcal{L}_{\textrm{wBCE}} (w_{\textrm{FN}}=1.33)\big)\) \label{fig:app_oulad_0.9}]{
        \begin{minipage}{0.45\linewidth}
            \centering
            \includegraphics[width=\linewidth]{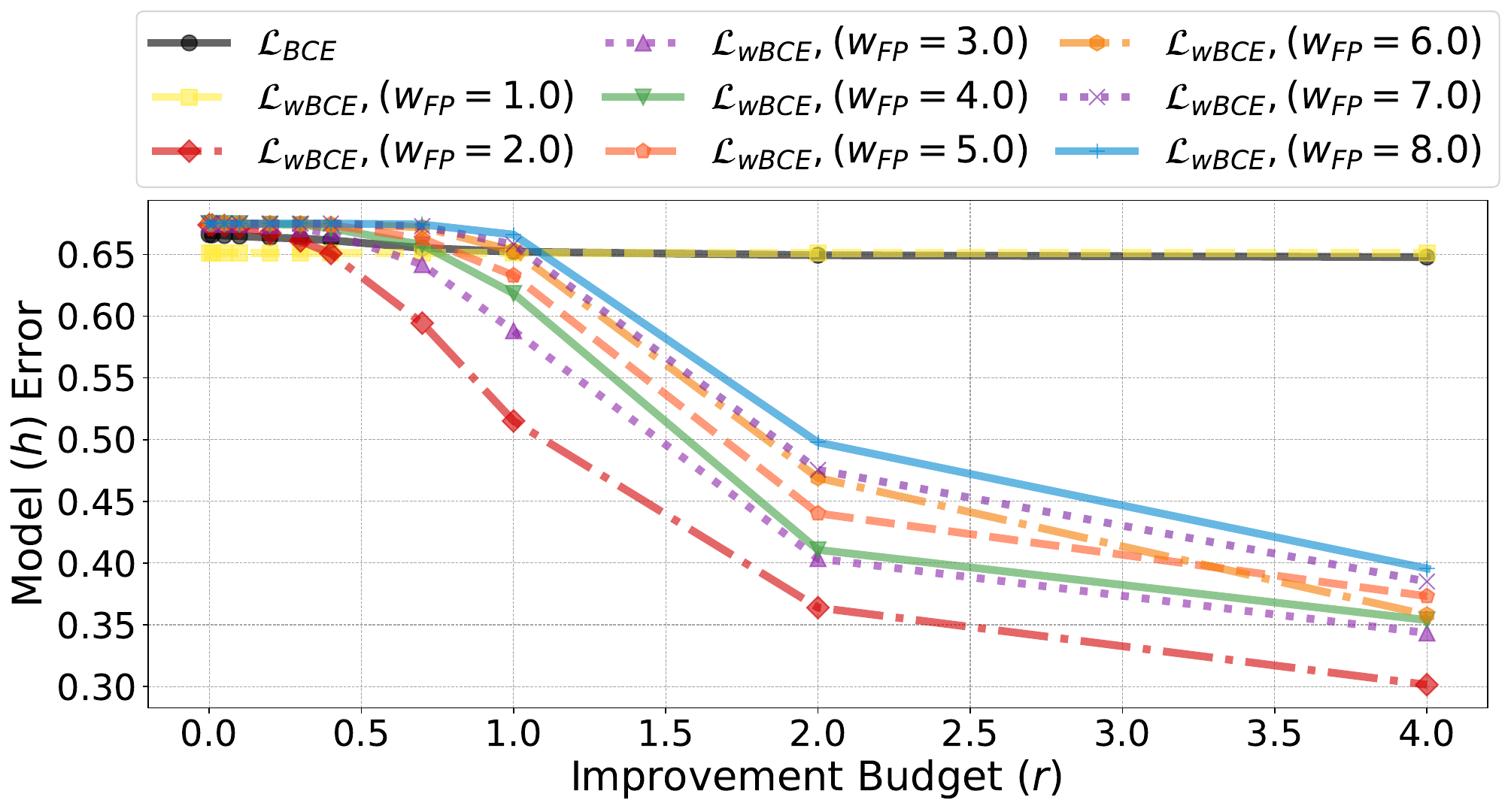}
        \end{minipage}
    }
    \vskip\baselineskip
    \subfloat[Law school \(\big(\mathcal{L}_{\textrm{wBCE}} (w_{\textrm{FN}}=0.009)\big)\)\label{fig:app_law_0.5}]{
        \begin{minipage}{0.45\linewidth}
            \centering
            \includegraphics[width=\linewidth]{experimental_figures/lfconlyw_0.5_error_droprate_linfx_law.pdf}
        \end{minipage}
    }%
    ~
    \subfloat[Law school \(\big(\mathcal{L}_{\textrm{wBCE}} (w_{\textrm{FN}}=0.009)\big)\)\label{fig:app_law_0.9}]{
        \begin{minipage}{0.45\linewidth}
            \centering
            \includegraphics[width=\linewidth]{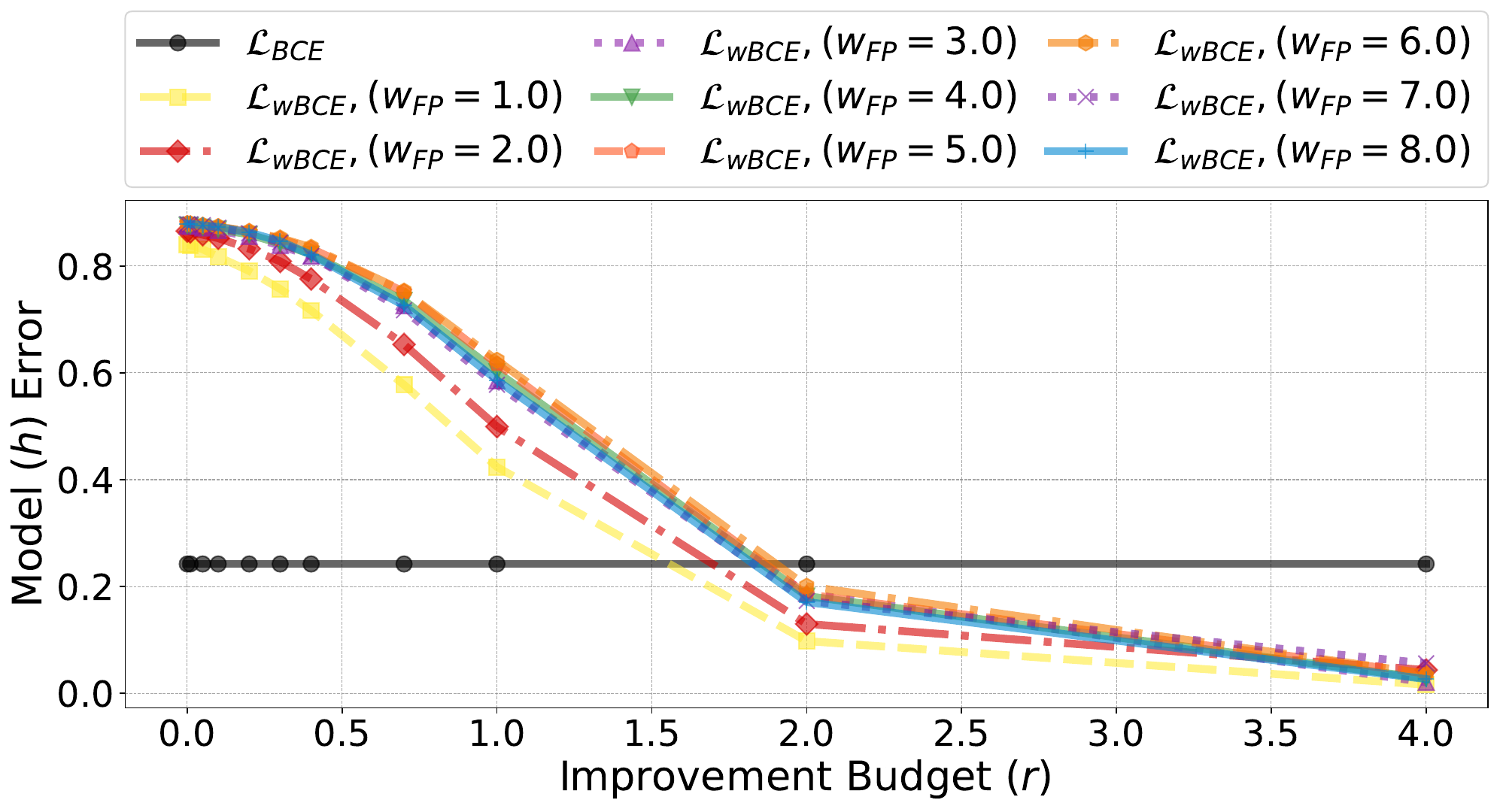}
        \end{minipage}
    }%
    \vskip\baselineskip
    \subfloat[Synthetic \(\big(\mathcal{L}_{\textrm{wBCE}} (w_{\textrm{FN}}=0.009)\big)\)\label{fig:app_synthetic_0.5}]{
        \begin{minipage}{0.45\linewidth}
            \centering
            \includegraphics[width=\linewidth]{experimental_figures/lfconlyw_0.5_error_droprate_linfx_synthetic2.pdf}
        \end{minipage}
    }
    ~
    \subfloat[Synthetic \(\big(\mathcal{L}_{\textrm{wBCE}} (w_{\textrm{FN}}=0.009)\big)\)\label{fig:app_synthetic_0.9}]{
        \begin{minipage}{0.45\linewidth}
            \centering
            \includegraphics[width=\linewidth]{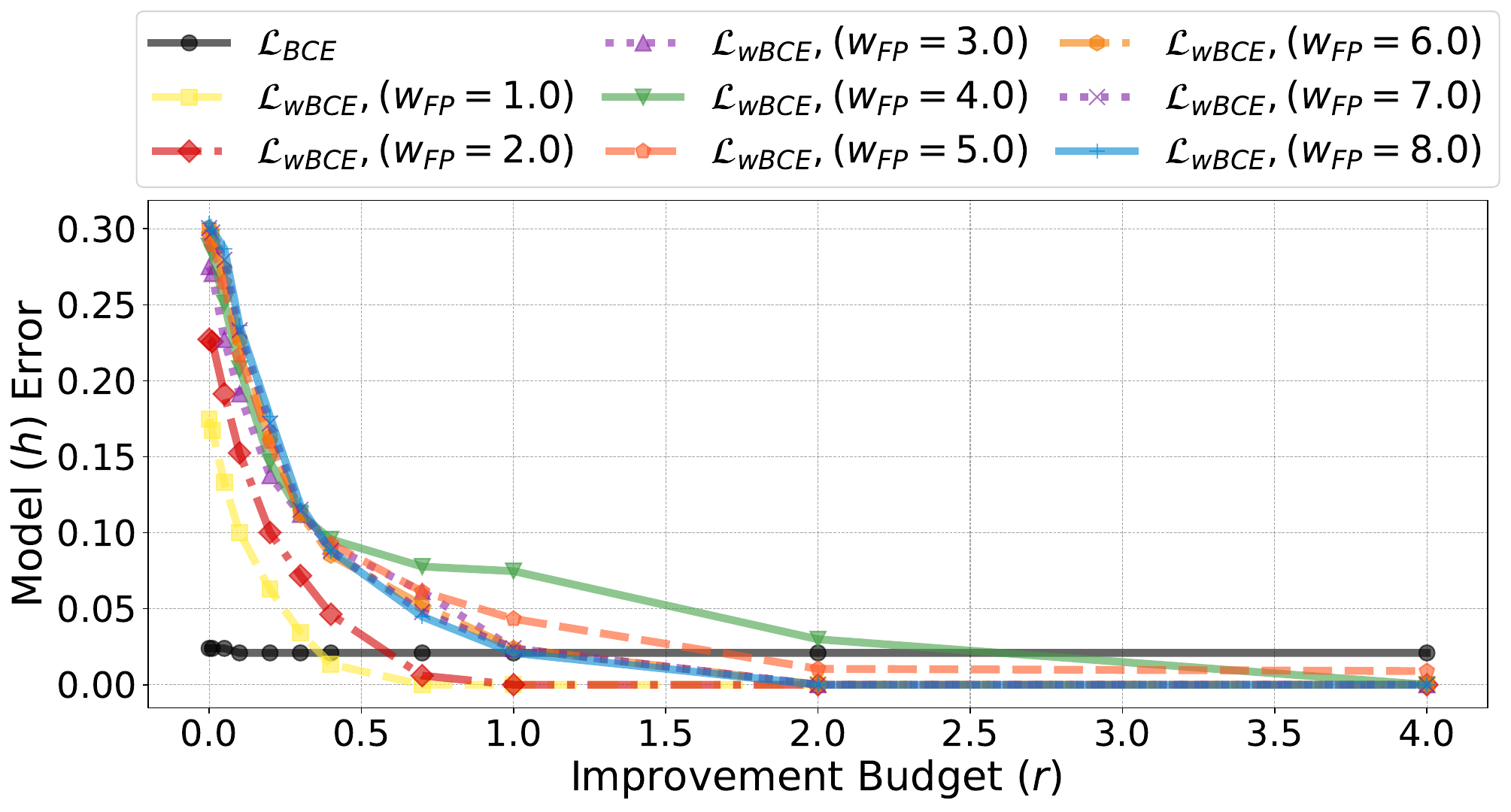}
        \end{minipage}
    }
    \caption{Comparison of the error drop rate when agents improve to the risk-averse \(\big(\mathcal{L}_{\textrm{wBCE}}, \frac{w_\textrm{FP}}{w_\textrm{FN}}>1,  w_{\textrm{FP}}=\{i\}_{i=1}^{8}\big)\) and standard  (\(\mathcal{L}_\textrm{BCE}, w_{\textrm{FP}}= w_{\textrm{FN}}=1\)) models across four datasets (Adult, OULAD, Law school, and Synthetic). \textbf{Column one} (\subref{fig:app_adult_0.5}, \subref{fig:app_oulad_0.5}, \subref{fig:app_law_0.5} and \subref{fig:app_synthetic_0.5}) considers models where an agent is classified as positive if the probability of being positive is above \(0.5\) and \textbf{column two}  (\subref{fig:app_adult_0.9}, \subref{fig:app_oulad_0.9}, \subref{fig:app_law_0.9} and \subref{fig:app_synthetic_0.9}) considers models where a higher threshold is used \(0.9\). Increasing the improvement budget and classifier risk-aversion (high \(\frac{w_{\textrm{FP}}}{w_{\textrm{FN}}}\)) leads to a sharper error drop rate, and loss-based risk aversion is more effective than the threshold-based risk-aversion.}
    \label{fig:app_thresh0.5thresh0.9}

    \begin{picture}(0,0)
        \put(-150,700){{\parbox{4cm}{\centering \(\text{Threshold}=0.5\)}}}
        \put(66,700){{\parbox{4cm}{\centering \(\text{Threshold}=0.9\)}}}
        \put(-230,605){\rotatebox{90}{Adult}}
        \put(-230,450){\rotatebox{90}{OULAD}}
        \put(-230,297){\rotatebox{90}{Law school}}
        \put(-230,152){\rotatebox{90}{Synthetic}}
        
    \end{picture}
\end{figure}
\begin{figure}[t!]
    \centering
    \subfloat[OULAD (\textbf{DTC1})  \label{fig:app_oulad_dtc1_0.5}]{
        \begin{minipage}{0.48\linewidth}
            \centering
            \includegraphics[width=\linewidth]{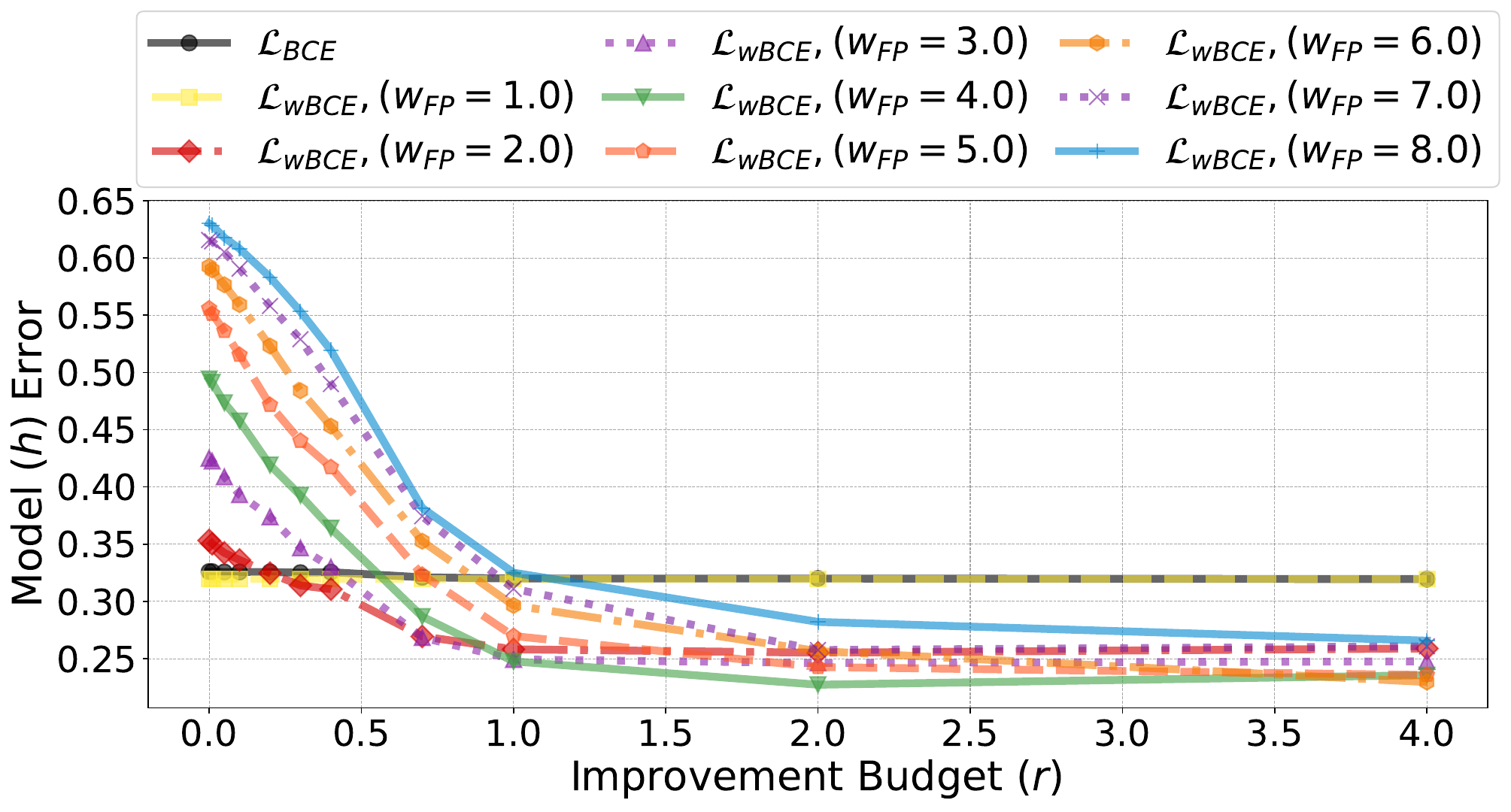}
        \end{minipage}
    } 
    \hfill
    \subfloat[OULAD (\textbf{RFC1}) \label{fig:app_oulad_rfc1_0.5}]{
        \begin{minipage}{0.48\linewidth}
            \centering
            \includegraphics[width=\linewidth]{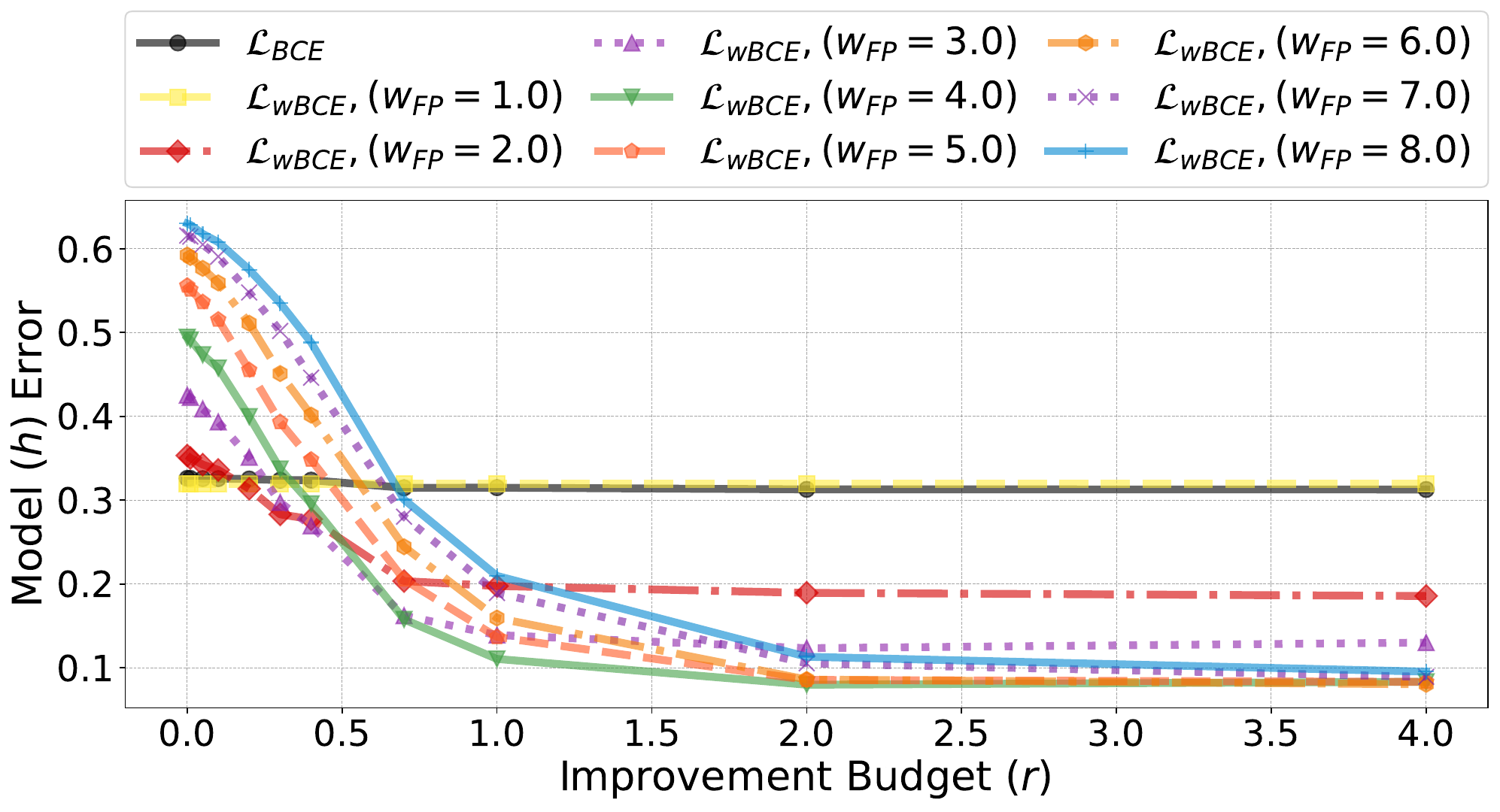}
        \end{minipage}
    }%
    \hfill
    \subfloat[OULAD (\textbf{RFC2}) \label{fig:app_oulad_rfc2_0.5}]{
        \begin{minipage}{0.48\linewidth}
            \centering
            \includegraphics[width=\linewidth]{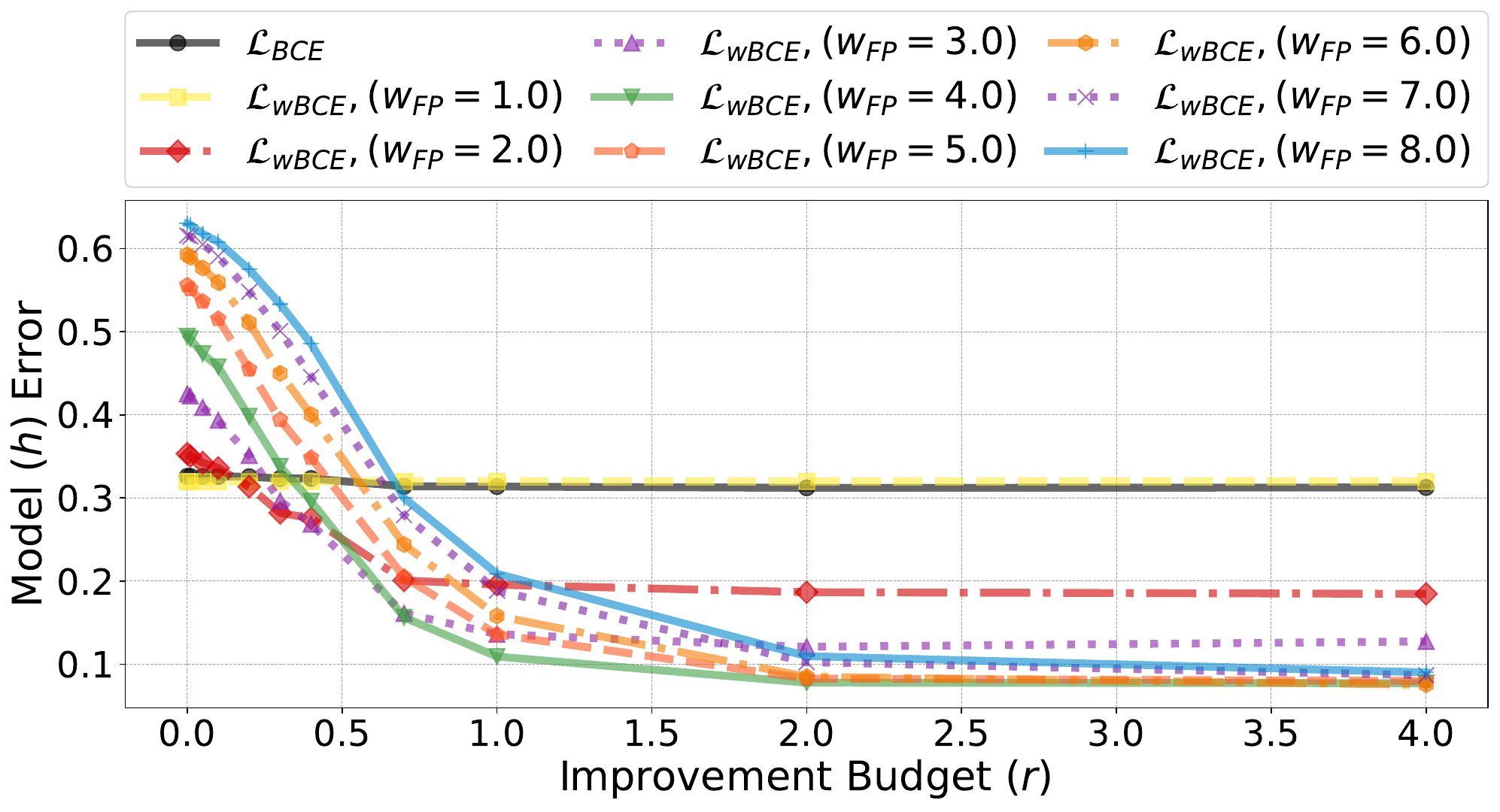}
        \end{minipage}
    }%
    \hfill
    \subfloat[OULAD (\textbf{XGB})\label{fig:app_oulad_xgb_0.5}]{
        \begin{minipage}{0.48\linewidth}
            \centering
            \includegraphics[width=\linewidth]{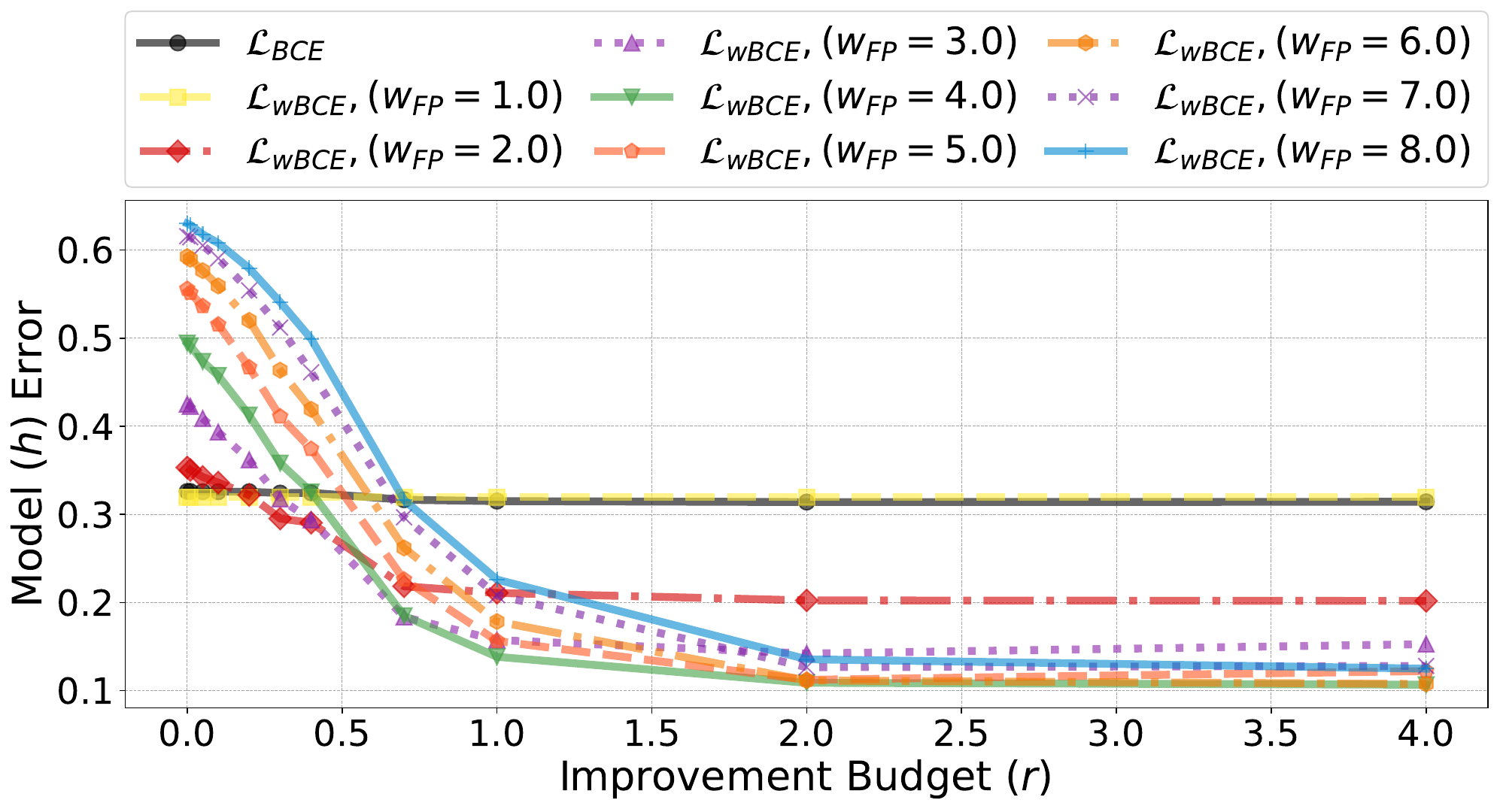}
        \end{minipage}
    }
    \caption{Risk-averse \(\big(\mathcal{L}_{\textrm{wBCE}} \ \text{with} \ w_{\textrm{FP}}=\{i\}_{i=1}^{8}, w_{\textrm{FN}}=1.33\big)\) and standard (\(\mathcal{L}_{\textrm{BCE}}, w_{\textrm{FP}}=w_{\textrm{FN}}=1\)) trained model function (\(h\)) error drop rates versus improvement budget (\(r\)) on the Adult dataset where different \textbf{singularly-defined} \(f^\star\) are used to verify successful-ness of improvement: (\subref{fig:app_oulad_dtc1_0.5}) with \(f^{\star}_{1}\) , (\subref{fig:app_oulad_rfc1_0.5}) with \(f^{\star}_{3}\), (\subref{fig:app_oulad_rfc2_0.5}) with \(f^{\star}_{4}\), and (\subref{fig:app_oulad_xgb_0.5}) with \(f^{\star}_{5}\). For \textbf{DTC2}, \(f^{\star}_{2}\) see Figure~\ref{fig:app_oulad_0.5}. In all cases, the threshold for classifying an agent as positive is \(0.5\).}
    \label{fig:app_oulad_multi_thresh0.5}
\end{figure}
\begin{figure}[ht!]
    \centering
    \subfloat[Adult \(\big(w_{\textrm{FN}}=0.001\big)\) \label{fig:app_adult_multi_0.5}]{
        \begin{minipage}{0.48\linewidth}
            \centering
            \includegraphics[width=\linewidth]{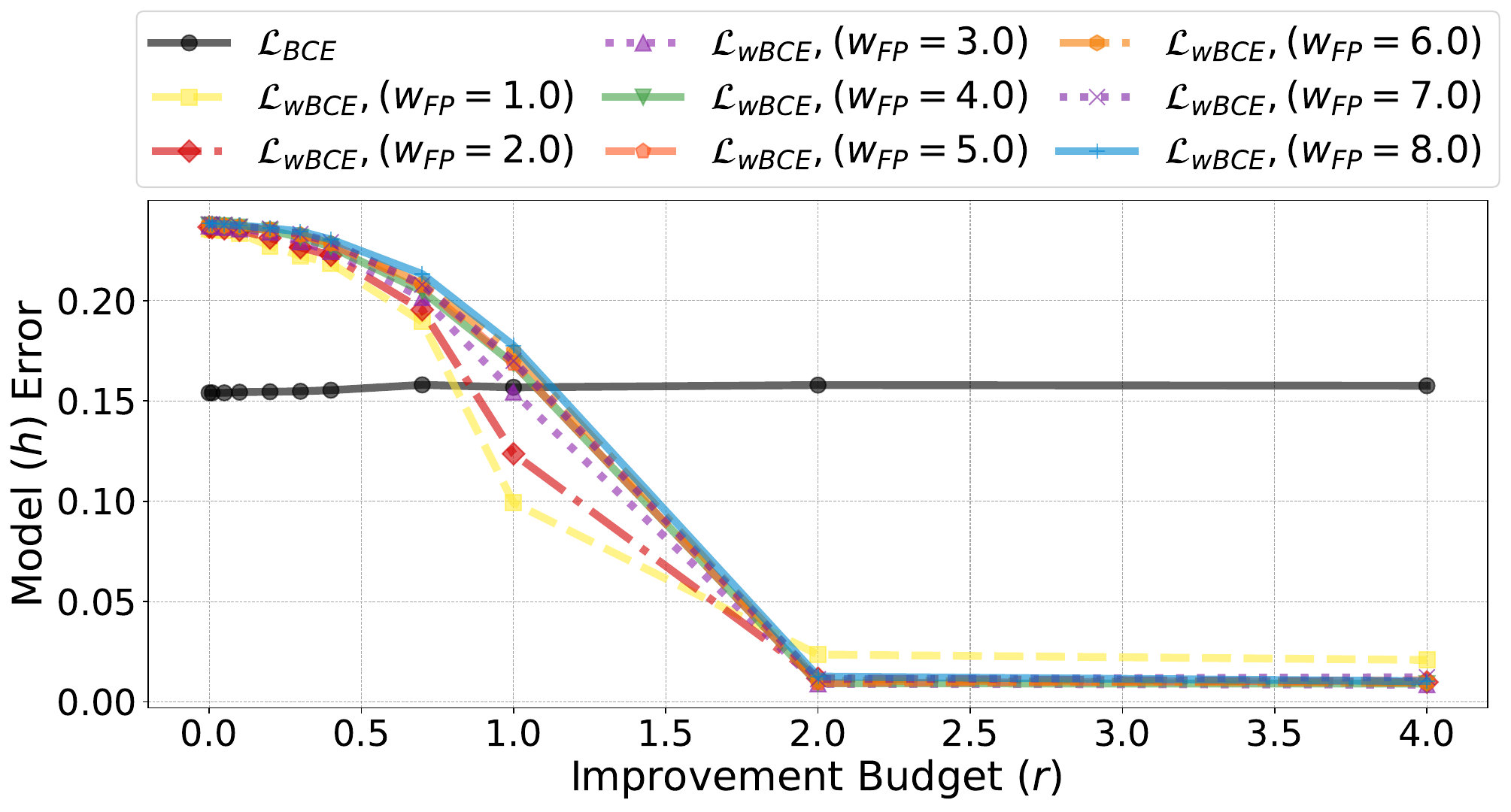}
        \end{minipage}
    } 
    \hfill
    \subfloat[OULAD \(\big(w_{\textrm{FN}}=1.33\big)\) \label{fig:app_oulad_multi_0.5}]{
        \begin{minipage}{0.48\linewidth}
            \centering
            \includegraphics[width=\linewidth]{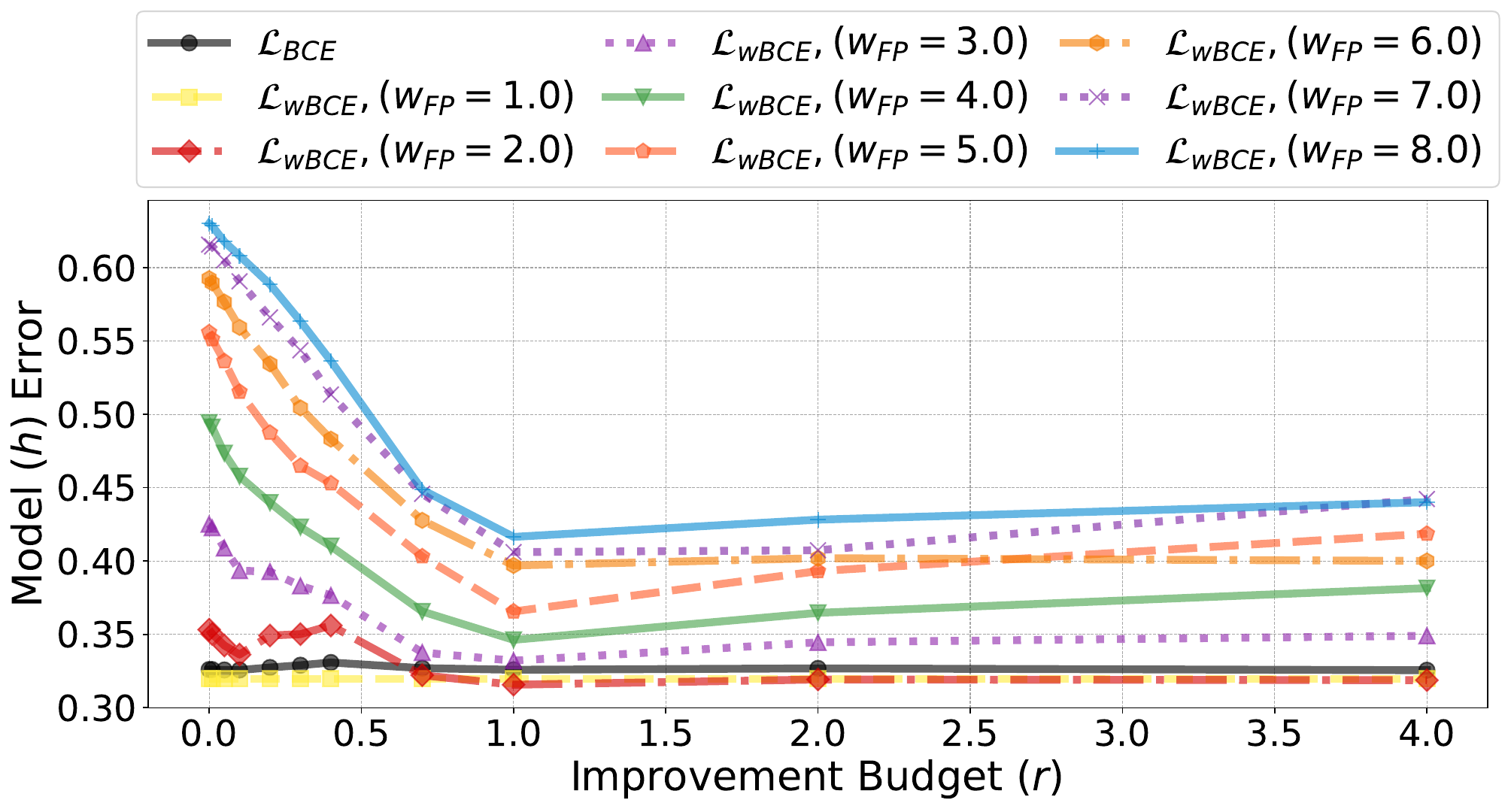}
        \end{minipage}
    }%
    \hfill
    \subfloat[Law school \(\big(w_{\textrm{FN}}=0.009\big)\)\label{fig:app_law_multi_0.5}]{
        \begin{minipage}{0.48\linewidth}
            \centering
            \includegraphics[width=\linewidth]{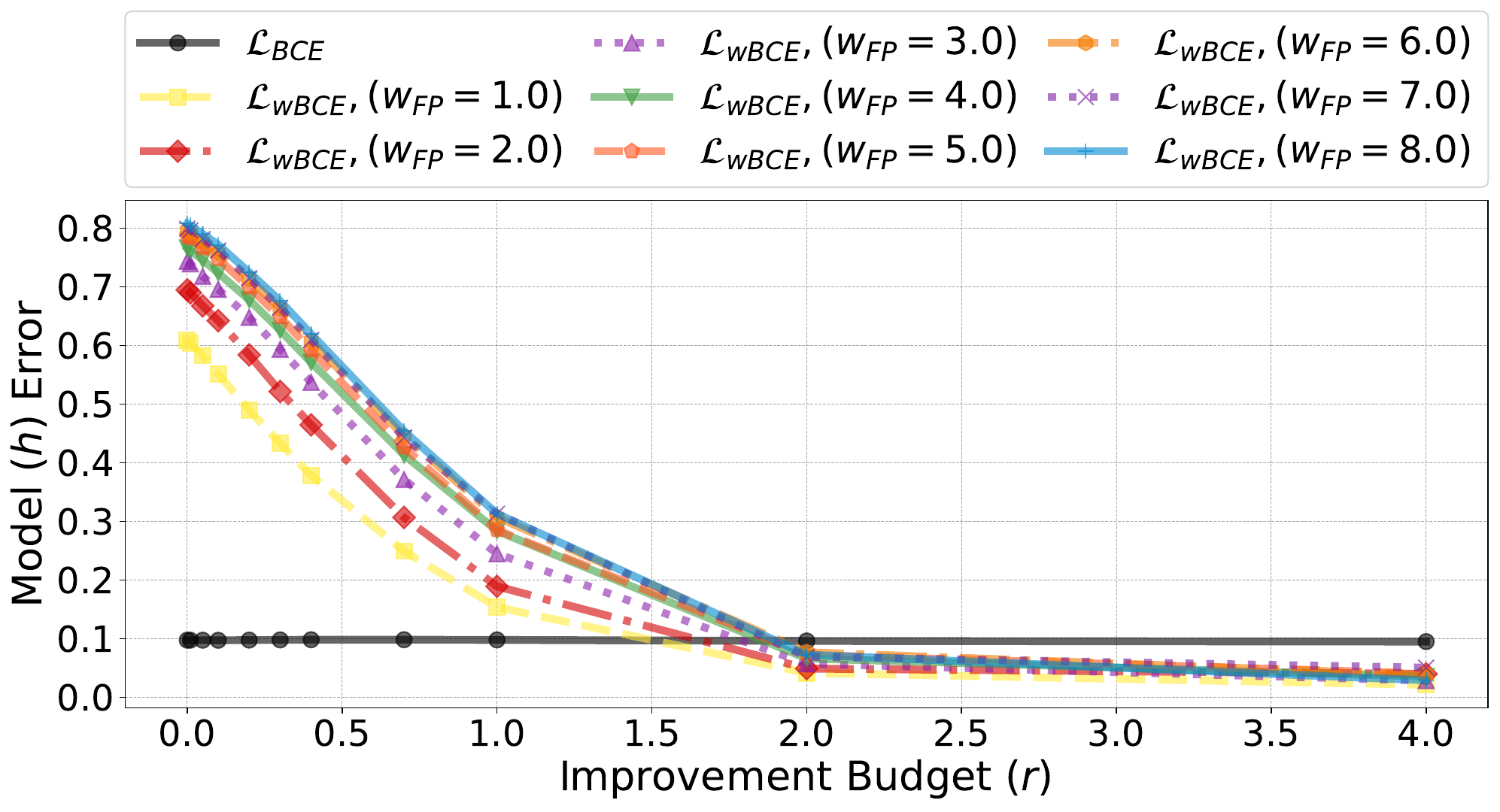}
        \end{minipage}
    }%
    \hfill
    \subfloat[Synthetic \(\big(w_{\textrm{FN}}=0.009\big)\)\label{fig:app_synthetic_multi_0.5}]{
        \begin{minipage}{0.48\linewidth}
            \centering
            \includegraphics[width=\linewidth]{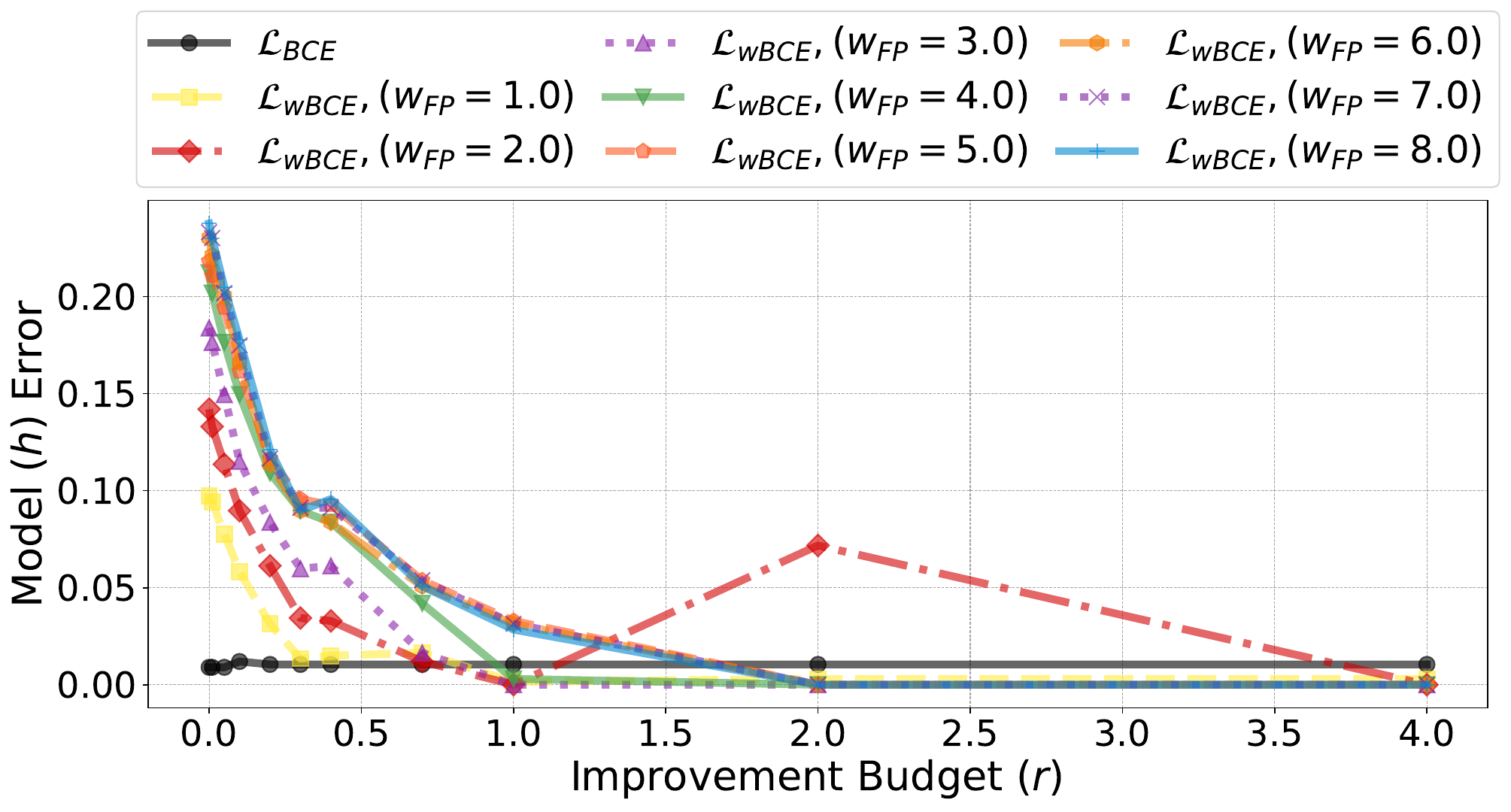}
        \end{minipage}
    }
    \caption{Risk-averse \(\big(\mathcal{L}_{\textrm{wBCE}} \ \text{with} \ w_{\textrm{FP}}=\{i\}_{i=1}^{8}\big)\) and standard (\(\mathcal{L}_{\textrm{BCE}}\)) trained model function (\(h\)) error versus improvement budget (\(r\)) across four datasets (Adult, OULAD, Law school, and Synthetic). For all cases, the threshold for classifying an agent as positive is \(0.5\) and a \textbf{multi-defined} \(f^\star\) model is used to verify successful-ness of improvement. Increasing the improvement budget and classifier risk-aversion (high \(\frac{w_{\textrm{FP}}}{w_{\textrm{FN}}}\)) leads to a faster error drop rate. }
    \label{fig:app_multi_thresh0.5}
\end{figure}